\documentclass{article} 
\IfFileExists{iclr2027_conference.sty}
  {\usepackage{iclr2027_conference,times}}
  {\usepackage{iclr2026_conference,times}}

\usepackage{amsmath,amssymb,amsthm}
\usepackage{booktabs,graphicx}
\usepackage{enumitem}
\usepackage{microtype,xcolor}
\usepackage{colortbl}
\usepackage{hyperref,url}
\usepackage{placeins}
\usepackage{xspace}
\usepackage{cleveref}
\usepackage{float}
\usepackage{tabularx}
\usepackage{algorithm}
\usepackage{algpseudocode}
\usepackage{booktabs}
\usepackage{caption}
\usepackage{algpseudocode}
\makeatletter
\@ifpackageloaded{graphicx}{}{\usepackage{graphicx}}
\@ifpackageloaded{xcolor}{}{\usepackage{xcolor}}
\@ifpackageloaded{tikz}{}{\usepackage{tikz}}
\makeatother
\usetikzlibrary{arrows.meta}

\providecommand{\egofigdir}{tikz/figures}
\providecommand{\egofigwidth}{\linewidth}

\providecolor{egoink}{HTML}{1F2430}
\providecolor{egoblue}{HTML}{2D6CDF}
\providecolor{egoteal}{HTML}{00A08A}
\providecolor{egoplum}{HTML}{8A4FBE}
\providecolor{egoamber}{HTML}{E8A317}
\providecolor{egored}{HTML}{D64545}

\makeatletter
\AtBeginDocument{%
  \@ifpackageloaded{pifont}
    {\providecommand{\egocheck}{\ding{51}}%
     \providecommand{\egocross}{\ding{55}}%
     \providecommand{\egostar}{\ding{72}}}
    {\providecommand{\egocheck}{\ensuremath{\checkmark}}%
     \providecommand{\egocross}{\ensuremath{\times}}%
     \providecommand{\egostar}{\ensuremath{\star}}}%
}
\makeatother

\newtheorem{proposition}{Proposition}
\newtheorem{lemma}{Lemma}
\newtheorem{corollary}{Corollary}
\title{Evolutionary One-Step Generators: Fast and Diverse Sampling for Discrete Design}
\author{
  \textbf{Marcus Vukojevic\textsuperscript{1}} \quad
  \textbf{Erik Nielsen\textsuperscript{1}} \quad
  \textbf{Veronica Lachi\textsuperscript{2}} \\
  \textbf{Andrea Passerini\textsuperscript{1}} \quad
  \textbf{Giovanni Iacca\textsuperscript{1}} \\[0.5em]
  \textnormal{\textsuperscript{1}University of Trento, Italy} \\
  \textnormal{\textsuperscript{2}UiT The Arctic University of Norway,
  Troms\o{}, Norway} \\[0.5em]
  \textnormal{\textsuperscript{1}\texttt{\{firstname.lastname\}@unitn.it}} \\
 \textnormal{\textsuperscript{2}\texttt{veronica.lachi@uit.no}}
}

\iclrfinalcopy

\newcommand{\methodname}{\textsc{EGO}\xspace}
\newcommand{\methodnamez}{\textsc{EGO}$_0$\xspace}

\begin{document}

\maketitle
\lhead{Preprint}

\begin{abstract}
Several discrete design tasks, such as molecular discovery, require diverse collections of useful candidates at low computational cost. High validity alone does not guarantee a useful candidate library: repeatedly generating the same valid structures leaves few distinct alternatives. Training for both feasibility and diversity is challenging because many relevant criteria can only be evaluated after hard decoding. To address this challenge, we propose \methodname (Evolutionary Generators with One-step inference), a framework for training compact generators directly on discrete outputs. The method combines distribution matching with structural constraints and optional diversity or history-dependent rewards, using antithetic low-rank evolution strategies without requiring criterion-specific differentiable surrogates. Once trained, the generator produces the entire graph in a single neural-network evaluation.
On molecular generation benchmarks, our compact generator achieves over $50\times$ the valid-and-unique yield per estimated dense operation compared to recent one-step flow-map baselines while retaining high chemical validity. 
In scaffold completion, \methodname achieves an observed $44.3\times$
speedup over MoLeR in generation to SMILES and produces approximately
$10\times$ as many filter-passing proposals within matched time budgets
for generation and screening.
Beyond chemistry, \methodname produces $1.54\times$ as many distinct held-out elite architectures as relaxed gradient training on NAS-Bench-101. 
The low generation cost may enable real-time candidate generation across discrete design tasks, supporting interactive exploration of constrained design spaces and rapid construction of candidate sets for downstream evaluation.
\end{abstract}

\section{Introduction}

Many discrete generative applications ultimately require not an unlimited stream of plausible samples, but a \emph{finite collection of useful candidates} \citep{DBLP:journals/jcisd/RenzLK24}. 
Molecular discovery constructs libraries of compounds for downstream screening, while neural architecture search considers a limited set of computational graphs for further evaluation \citep{brown2019guacamol}. 
In such settings, distributional fidelity and validity remain important, but they are not sufficient: invalid outputs waste evaluation budget, repeated outputs add little to the candidate pool, and expensive sampling limits how many proposals can be considered. A useful generator must therefore produce many \emph{distinct qualifying
candidates} \citep{DBLP:conf/iclr/XieX0M23} under a finite proposal and compute budget.

This is challenging for discrete generative models because many relevant criteria are naturally defined only on the final discrete object. 
Connectivity, graph validity, exact identity, duplicate detection, and history-dependent rewards are straightforward to evaluate on generated structures, while optimizing them through categorical sampling typically requires task-specific relaxations or surrogate gradients.
Continuous relaxations and straight-through estimators instead optimize differentiable surrogates \citep{bengio2013estimating,jang2017gumbel}, which can be highly effective when the surrogate faithfully represents the desired hard-output criterion. 
Score-function estimators provide direct access to hard rewards, but introduce different optimization trade-offs.

One-step generators are particularly appealing in this setting because,
once trained, they map latent noise and optional conditions directly
to a complete proposal in a single learned-network evaluation
\citep{decao2018molgan,roos2026categoricalflowmaps,deng2026drifting}.
The central training question is:
\emph{How can this one-step proposal distribution be shaped directly
by criteria evaluated on its final hard outputs?}

We address this problem by proposing \methodname, an objective-direct framework for training one-step discrete generators. 
Complete hard samples are scored using distribution-matching terms, structural constraints, and optional set- or history-dependent rewards, without requiring these criteria to be differentiable. 
We optimize the resulting objective with antithetic evolution strategies (ES) using shared sampling noise and low-rank parameter perturbations.

Building on established ES methods
\citep{salimans2017es,sarkar2025eggroll},
we contribute a training interface for reusable one-step discrete
generators, a finite-budget formulation, and an empirical study of
when the resulting training--inference trade-off is favorable.
Population evaluation increases training cost, but after training each proposal requires only one evaluation of the central generator. The \methodname workflow is shown in \Cref{fig:ego-overview}.

Our experiments show that this trade-off can be favorable in structured candidate-generation tasks.
On molecular benchmarks, \methodname achieves high distinct valid yield at low estimated inference cost: relative to CFM-ECLD \citep{roos2026categoricalflowmaps}, its valid-and-unique yield per estimated dense operation is approximately $50\times$ higher on QM9 \citep{ramakrishnan2014qm9} and $200\times$ higher on ZINC250K \citep{irwin2012zinc}, with a measurable distributional-fidelity trade-off. 
Under a matched 30-minute active-training budget on ZINC250K,
\methodname reaches a mean strict valid--unique--novel yield of $77.9\%$, compared with $5.8\%$ for relaxed training, despite the gradient methods completing approximately six times as many optimizer updates.
We further explore the potential of \methodname as a reusable candidate generator through pilot experiments.
On 100 held-out scaffold queries, \methodname produces an average of
$547.3$ filter-passing proposals per query, compared with $54.5$ for
MoLeR \citep{maziarz2022moler}, under matched generation and screening
budgets averaging $6.08$ seconds per query,  giving a $10.05\times$
gain in the number of filter-passing proposals.
In a complementary evaluation at $100000$ proposals,
\methodname returns $20245$ distinct filter-passing completions
versus $16684$ for MoLeR, a $1.21\times$ increase with an observed speedup of $44.3\times$.
On NAS-Bench-101 \citep{ying2019nasbench101}, it produces $1.54\times$ as many distinct held-out elite architectures as relaxed training at the same proposal budget, although relaxed training achieves higher mean accuracy and closer feature-distribution matching. 

Together, these results suggest that objective-direct evolutionary training is most useful when hard feasibility and candidate diversity are poorly represented by a differentiable surrogate, and when low-cost repeated sampling is important. To summarize, our main \textbf{contributions} are:
\begin{enumerate}[leftmargin=*]
\item We formulate discrete generation as a finite-budget candidate-generation problem, showing that distinct successful yield depends both on success probability and on how probability mass is
distributed across successful identities.
\item We introduce an objective-direct training framework for one-step discrete generators that supports distributional, structural, and history-dependent criteria evaluated
on complete hard samples.
\item We instantiate this interface with antithetic low-rank evolution
strategies and evaluate it against diverse graph-generation baselines,
demonstrating its effectiveness for distinct candidate generation with
one-step inference while documenting the associated trade-offs.
\end{enumerate}

The appendix further examines this trade-off through variations in
optimization recipe and model scale, as well as controlled
non-molecular experiments. It also discusses related work and
provides full implementation and evaluation details for reproducibility.

\begin{figure}[t]
  \centering
  \input{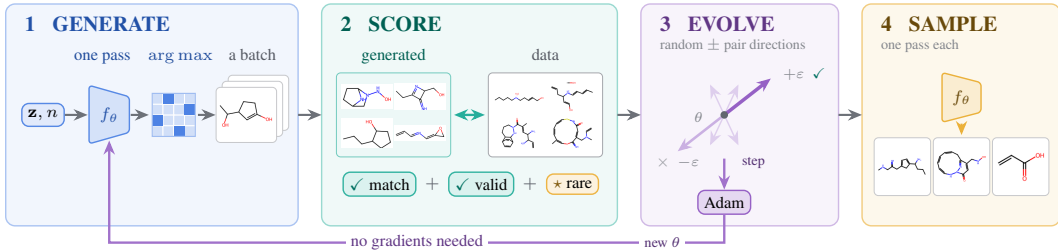}
  \caption{\methodname. One evaluation of $f_\theta$ turns $(\mathbf{z},n)$ into a batch of complete graphs (1). The criterion is read off those hard graphs: distribution match, feasibility, and rarity (2). The update compares opposite probes $\theta\pm\varepsilon$ scored under identical noise (3). Adam carries the step back to the generator along the loop below. After training, each proposal costs one network evaluation (4).}
  \label{fig:ego-overview}
\end{figure}

\section{Finite-Budget Discrete Generation}
\label{sec:problem_formulation}

We consider generation over a discrete space $\mathcal{X}$, with attributed
graphs as our primary setting. When generation is conditional, let
$c\in\mathcal{C}$ denote the condition. A stochastic generator with parameters $\theta$ induces the distribution
\begin{equation}
q_\theta(x \mid c) = \Pr_\theta(X=x \mid c),
\label{eq:generator_distribution}
\end{equation}
for $x \in \mathcal{X}$.
For unconditional generation, we omit $c$. The following finite-budget
analysis applies to independent complete samples regardless of the
generator's sampling procedure. \Cref{sec:method} introduces
\methodname, which generates each proposal in one network evaluation.

Many applications evaluate a generator from a fixed number of samples rather than a single output.
Let $X_1,\ldots,X_N$ be $N$ independent samples for a given condition $c$. We denote by
$a(x,c)\in\{0,1\}$ whether $x$ satisfies the evaluation criterion. To identify repeated outputs, let
$\kappa_c:\mathcal{X}\rightarrow\mathcal{K}_c$ map each output to the identity used for duplicate
detection. The number of distinct successful outputs is
$U_N=\left|\{\kappa_c(X_i):a(X_i,c)=1,\;i\in[N]\}\right|$, with corresponding yield
$Y_N=U_N/N$. What counts as a successful output depends on the task. For example, it may be a
valid, connected, training-novel molecule, an architecture in a target performance set, or a scaffold
completion passing a set of chemical filters. Distinct yield complements, rather than replaces, measures
of distributional fidelity, coverage, and average quality.

At finite $N$, not only the probability of success but also how this probability is distributed across
successful identities affects the resulting candidate set. Let
$\mathcal{A}_c=\{\kappa_c(x):x\in\mathcal{X},\,a(x,c)=1\}$ and define
$p_\theta(u\mid c)=\Pr_\theta(a(X,c)=1,\kappa_c(X)=u\mid c)$ for
$u\in\mathcal{A}_c$.

\begin{proposition}[Finite-budget distinct yield]
\label{prop:finite_budget_yield}
For $N$ independent samples from $q_\theta(\cdot\mid c)$,
\begin{equation}
    \mathbb{E}[U_N\mid c]
    =
    \sum_{u\in\mathcal{A}_c}
    \left[1-(1-p_\theta(u\mid c))^N\right].
    \label{eq:expected_distinct_count}
\end{equation}
Let
$\alpha_\theta(c)=\sum_{u\in\mathcal{A}_c}p_\theta(u\mid c)$
denote the probability of producing a successful output, and let
$\Gamma_\theta(c)=\sum_{u\in\mathcal{A}_c}p_\theta(u\mid c)^2$.
Then
\[
    \max\!\left\{
        0,\,
        N\alpha_\theta(c)
        -
        \binom{N}{2}\Gamma_\theta(c)
    \right\}
    \leq
    \mathbb{E}[U_N\mid c]
    \leq
    N\alpha_\theta(c).
\]
\end{proposition}

The proposition separates two effects of the learned proposal distribution.
The quantity $\alpha_\theta(c)$ measures how much probability mass is assigned to successful outputs,
whereas $\Gamma_\theta(c)$ is the probability that two independent proposals produce the same
successful identity. Thus, a high probability of success alone does not imply a large distinct candidate
set at finite budget. The proof is given in Appendix~\ref{app:proofs}.
This can also be seen directly from Eq.~\eqref{eq:expected_distinct_count}. Let
$f_N(p)=1-(1-p)^N$, which is concave for $N\geq2$. Hence,
$f_N(p_1)+f_N(p_2)\leq2f_N((p_1+p_2)/2)$. Replacing the probabilities of two successful
identities by their average, while keeping their total probability fixed, therefore cannot decrease the
expected number of distinct outputs.

\textbf{Training on discrete samples.}
The criterion used for training need not coincide with the metric used for final evaluation. At iteration
$t$, let $S_B=\{(X_i,c_i)\}_{i=1}^{B}$, with $X_i\sim q_\theta(\cdot\mid c_i)$, be a batch of
$B$ generated discrete objects. Let $R$ be an optional reference batch and let $H_t$ denote an optional
state accumulated before iteration $t$. We consider criteria of the form
\begin{equation}
    J_t(\theta\mid H_t)
    =
    \mathbb{E}\!\left[
        \mathcal{L}(S_B,R,H_t)
        \mid H_t
    \right],
    \label{eq:hard_objective}
\end{equation}
where the expectation is over the sampled conditions, generated objects, and reference data used at
iteration $t$. A training criterion may contain some or all of the following components:
\begin{equation}
    \mathcal{L}(S_B,R,H_t)
    =
    D(S_B,R)
    +
    \sum_j \lambda_j C_j(S_B)
    -
    \sum_r \beta_r \mathcal{R}_r(S_B,H_t).
    \label{eq:objective_decomposition}
\end{equation}
Here, $D$ measures the discrepancy between generated and reference data, $C_j$ are structural or
feasibility penalties, and $\mathcal{R}_r$ are optional task-, set-, or history-dependent rewards.
The state $H_t$ is held fixed while the samples used for the current update are evaluated. All terms
in Eq.~\eqref{eq:objective_decomposition} are evaluated on discrete generated objects and need not
be differentiable.

\textbf{Hard objectives and gradient approximations.}
Discrete sampling prevents ordinary backpropagation through
the sampled choices, but still permits gradient-based
optimization of the expected hard objective.
A continuous relaxation replaces discrete variables with
soft counterparts and optimizes a differentiable proxy. We refer to this relaxation-based approach as Soft.
The straight-through estimator (STE) instead retains the hard
forward computation while substituting derivatives during
backpropagation~\citep{bengio2013estimating,jang2017gumbel}. Soft generally changes the optimized objective, whereas STE's substituted derivatives do not generally correspond to the gradient of Eq.~\eqref{eq:hard_objective}.
Score-function estimators such as
REINFORCE~\citep{williams1992reinforce} provide another route:
they combine hard-sample objective values with gradients
of sampling log probabilities, without differentiating
through the sampled objects.
Our experiments (\ref{sec:questions}) compare this approach with \methodname and 
Appendix~\ref{app:qm9-reinforce} provides the full training
and evaluation protocol. Evolution strategies construct updates from objective
evaluations under parameter perturbations, without requiring
backpropagation through the generator.
We use this approach to incorporate structural checks and
task-, set-, or history-dependent rewards through a common
interface based on hard-objective evaluations.
The additional forward evaluations during training are
part of the computational trade-off assessed in our experiments. \Cref{sec:method} describes the evolutionary,
straight-through, and relaxed training procedures.

\section{Objective-Direct Training of One-Step Generators}
\label{sec:method}

In \Cref{sec:problem_formulation}, we separated the criterion evaluated on discrete samples from the
signal used to optimize the generator. We now connect the three components of our method: a generator
that emits a complete discrete object in one pass, an objective evaluated on that hard object, and an
evolutionary update that uses objective values without differentiating through sampling. The generator
determines inference cost; the training interface determines which criteria shape its proposal distribution.

\textbf{One-step discrete generation.} 
Let $z\sim\mathcal N(0,I)$ be a latent variable, $c$ the condition, and $\xi$ the randomness used
for discrete sampling. We parameterize the distribution in Eq.~\eqref{eq:generator_distribution} as
\begin{equation}
\hat G = d(f_\theta(z,c), \xi; c).
\end{equation}
where $f_\theta$ predicts all output logits in parallel and $d$ applies hard sampling together with
task-specific structural masks. A proposal therefore requires one evaluation of $f_\theta$; discrete
sampling, graph construction, and subsequent validity checks do not require additional generator
evaluations.

\textbf{Instantiations.}
For unconditional molecular generation, the condition $c=n$ specifies the number of heavy atoms.
The generator concatenates $z$ with an encoding of $n$ and processes it with a multilayer perceptron:
$h_0=[z;\operatorname{onehot}(n)]$ and
$h_\ell=\operatorname{ReLU}(W_\ell h_{\ell-1}+b_\ell)$.
Two output heads predict atom-state logits $L^A$ and bond logits $L^B$. Hard atom and bond states
are sampled as
\begin{equation}
    \widehat a_i
    =
    \arg\max_k(L^A_{ik}+\xi^A_{ik}),
    \qquad
    \widehat b_{ij}
    =
    \arg\max_k(L^B_{ij,k}+\xi^B_{ij,k}),
    \quad i<j,
    \label{eq:hard_categorical_sampling}
\end{equation}
with independent standard Gumbel noise. Only the first $n$ nodes are active, and bonds involving
inactive nodes are masked out. This enforces the requested graph size and an undirected graph without
self-loops, but does not enforce chemical validity or connectedness. During distributional generation,
$n$ is sampled from the training size distribution. For scaffold completion, we instead use a separate one-step decoder that preserves the supplied core atoms and
bonds and predicts the added atoms, their parent assignments, and bond orders. Each added atom selects
a parent from the scaffold or from an earlier added atom, which guarantees attachment to a connected
nonempty core. Additional predicted edges can introduce cycles. The scaffold-only setting supplies only
the scaffold, with the number of added atoms sampled from the training data.
For NAS-Bench-101, an unconditional MLP predicts in parallel logits for the 21 possible
forward edges on a seven-node graph and categorical operations for its five internal nodes.
Restricting edges to the strict upper triangle ensures acyclicity. Input and output nodes
are fixed, and nodes outside input-to-output paths are pruned.
Full architectural details for all three generators are given in Appendix~\ref{app:method}.

\textbf{Objectives on discrete samples.}
The objective in Eq.~\eqref{eq:objective_decomposition} must reward realistic samples without losing
the hard constraints that determine whether they are usable. We therefore combine distribution matching
with penalties evaluated on sampled discrete structures. Generated and reference objects are compared
through task-specific feature maps
$\phi_r$. For a generated batch $\{\widehat G_i\}_{i=1}^{B}$ and reference batch
$\{G_j^{\mathrm{ref}}\}_{j=1}^{B_+}$, let
$x_i^{(r)}=\phi_r(\widehat G_i)/s_r$ and
$y_j^{(r)}=\phi_r(G_j^{\mathrm{ref}})/s_r$.
For each feature family, we use the empirical Energy discrepancy~\citep{szekely2013energy}
\begin{equation}
\begin{aligned}
    \mathcal E_r
    ={}&
    \frac{2}{BB_+}
    \sum_{i,j}
    \rho(x_i^{(r)},y_j^{(r)})
    -
    \frac{1}{B(B-1)}
    \sum_{i\neq i'}
    \rho(x_i^{(r)},x_{i'}^{(r)})
    \\
    &-
    \frac{1}{B_+(B_+-1)}
    \sum_{j\neq j'}
    \rho(y_j^{(r)},y_{j'}^{(r)}),
\end{aligned}
\label{eq:energy_discrepancy}
\end{equation}
where $\rho(a,b)=\sqrt{\lVert a-b\rVert_2^2+10^{-12}}$.
The first term brings generated features close to the reference data, while the second prevents concentration of the generated feature distribution. Feature maps describe task-relevant graph properties, such as node and edge attribute frequencies, degree statistics, local patterns, and spectral or connectivity descriptors. Their choice determines which aspects of the reference distribution are matched. We combine these discrepancies with graph-constraint penalties:
\begin{equation}
    \mathcal L_{\mathrm{core}}
    =
    \sum_r \mathcal E_r
    +
    \sum_j \lambda_j
    \frac{1}{B}\sum_{i=1}^{B}C_j(\widehat G_i).
    \label{eq:core_objective}
\end{equation}
Each $C_j$ measures violations of a task-specific structural or attribute constraint. The feature maps, penalties, and their weights are specified for each application in Appendix~\ref{app:method}. Because \methodname scores hard samples directly, its
distribution-matching objective is modular. For scaffold completion, our experiments favored Drift \citep{deng2026drifting}, a local attraction--repulsion objective, over the global Energy discrepancy in Eq.~\eqref{eq:energy_discrepancy} (see Appendix~\ref{app:scaffold} for details).

Some useful criteria concern a candidate set rather than an isolated sample. Equation~\eqref{eq:objective_decomposition}
therefore also allows rewards that depend on previously generated samples. History-dependent rewards can encourage exploration of graph structures that have been sampled less frequently. Let $h(G)$ be a task-specific graph signature and let $V_t(h(G))$ count its visits before iteration $t$. We define the batch \emph{rarity reward} as 
\begin{equation}
    \mathcal R_t
    =
    \frac{1}{B}
    \sum_{i=1}^{B}
    \frac{
        \mathbf 1_{\mathrm{feasible}}(\widehat G_i)
    }{
        \sqrt{1+V_t(h(\widehat G_i))}
    }.
    \label{eq:rarity_reward}
\end{equation}
The history is held fixed while the current population is evaluated and
updated only afterwards. The feasibility indicator refers to the constraints used during training, which need not coincide with the final evaluation criterion. Likewise, the archive signature $h$ need not coincide with the identity map $\kappa_c$ used for evaluation. Additional downstream filters can therefore assess properties that were not explicitly optimized during training.

\textbf{Evolutionary optimization.}
\label{sec:es_training}
Evolution strategies optimize the generator from objective
evaluations without differentiating through hard sampling,
feature extraction, or structural checks~\citep{salimans2017es}.
At each update, we evaluate antithetic parameter perturbations
using the same latent variables, conditions, reference data,
and Gumbel noise. Sharing the sampling noise couples the
categorical decisions and removes disagreements caused solely
by resampling. Lemma~\ref{lem:shared-gumbel}, stated and proved
in Appendix~\ref{app:shared-gumbel}, quantifies this stability
by bounding the probability of different discrete outputs
in terms of the changes in their logits.

For small logit changes, a shared-noise comparison changes a category only near an argmax boundary;
independent noise introduces disagreement even between identical parameter vectors. The lemma therefore
justifies coupling sampling noise within each antithetic pair. It does not imply that the rank-shaped ES
direction is an unbiased gradient estimator. Appendix~\ref{app:proofs} gives the proof.

To reduce the memory required by the population, we use rank-one perturbations. For a weight matrix
$W_\ell$, direction $m$ is $\epsilon_{m\ell}=u_{m\ell}v_{m\ell}^{\top}$, with independent
Gaussian vectors $u_{m\ell}$ and $v_{m\ell}$. Each direction defines the antithetic
pair
\begin{equation}
    \theta_{m,\ell}^{\pm}
    =
    \theta_\ell
    \pm
    \sigma_\ell\epsilon_{m\ell},
    \qquad m=1,\ldots,M,
    \label{eq:antithetic_parameters}
\end{equation}
with analogous vector perturbations for biases. This follows scalable neural ES with antithetic
sampling~\citep{salimans2017es} and low-rank population evaluation~\citep{sarkar2025eggroll}. Let $s_m^+$ and $s_m^-$ denote the shaped fitness values assigned to the two members of pair $m$.
The search direction for parameter block $\ell$ is
\begin{equation}
    g_\ell
    =
    \frac{1}{2M\sigma_\ell}
    \sum_{m=1}^{M}
    (s_m^+-s_m^-)\epsilon_{m\ell}.
    \label{eq:es_direction}
\end{equation}
We use Adam \citep{kingma2015adam} to update the central parameters from $g$. Fitness shaping is experiment-dependent: the molecular experiments use centered ranks, with the rarity term ranked separately when present,
whereas the controlled NAS uses centered negative loss without an explicit
rarity reward. We use Eq.~\eqref{eq:es_direction} as a search direction and do not interpret it as an
unbiased gradient estimator. The population of networks is required only during training. After training, generation uses the central network
in Eq.~\eqref{eq:generator_distribution}. Exact perturbation scales, population sizes, optimizer settings, feature
normalization, and implementation details are reported in Appendix~\ref{app:optimization}.

\section{Experiments}
\label{sec:questions}
We organize the evaluation into three questions. \textbf{RQ1} asks what validity--efficiency trade-off one-step and multi-step molecular generators offer. \textbf{RQ2} asks whether one-step generation supports useful downstream candidate libraries. \textbf{RQ3} asks whether the approach extends beyond molecular graphs.

\textbf{Baselines and controlled comparisons.}
\label{sec:gradient_comparisons}
We compare \methodname with three gradient-based baselines sharing the
generator architecture and training split within each controlled experiment.
All methods are evaluated on hard discrete outputs.
STE~\citep{bengio2013estimating} uses hard features in the forward pass
and their soft counterparts' derivatives:
$\phi_{\mathrm{STE}}=\phi_{\mathrm{soft}}+
\operatorname{stopgrad}(\phi_{\mathrm{hard}}-\phi_{\mathrm{soft}})$.
Soft~\citep{jang2017gumbel} optimizes the relaxed objective.
REINFORCE with leave-one-out (RLOO)~\citep{williams1992reinforce} estimates
gradients from hard rewards without differentiating through discrete outputs.
The molecular size sweep includes \methodnamez, which omits the rarity
reward in Eq.~\eqref{eq:rarity_reward} from \methodname to assess its
contribution to candidate yield.
A separate QM9 control compares \methodname and RLOO using the same hard
reward and fixed feature scales at matched hard-reward evaluation budgets,
excluding rarity and rank shaping to isolate the update recipes.
See Appendices~\ref{app:gradient-baselines} and~\ref{app:qm9-reinforce}
for experiment-specific objectives, budgets, and checkpoint rules. Molecular comparisons also include MolGAN~\citep{decao2018molgan},
MoFlow~\citep{zang2020moflow},
CFM-CSD/ECLD~\citep{roos2026categoricalflowmaps},
Kernel-gradient drifting~\citep{esteban2026kernelgradient},
CatFlow~\citep{eijkelboom2024catflow}, DeFoG~\citep{qin2025defog},
MELD~\citep{seo2026meld}, G2PT~\citep{chen2025g2pt}, and
SMILES LSTM~\citep{brown2019guacamol}.
The training-time comparison also includes Cometh~\citep{siraudin2025cometh}.
Additional baselines and protocols are given in
Appendix~\ref{app:external-molecular}.

\textbf{Evaluation protocol.} Full validity requires chemical sanitization of the complete molecular graph. VU measures distinct valid molecules per attempt, while VUN additionally excludes training molecules; strict variants also require connectivity. These yields differ from uniqueness conditioned on validity. Estimated efficiency $E_{\mathrm{VU}}$ normalizes VU yield by reconstructed dense-projection work, including all neural evaluations; it excludes other computation and is not a runtime measurement. Scaffold uniqueness is measured per query. Matched-time budgets cover
generation and screening, excluding final SMILES export for \methodname
(Appendix~\ref{app:scaffold}). NAS measures distinct held-out elite yield alongside validity and average accuracy. Tables report mean $\pm$ sample SD across runs. Appendix~\ref{app:method} details metric definitions, seed counts, checkpoint selection, and computational accounting.

\begin{figure}[!ht]
    \centering
    \includegraphics[width=\linewidth]{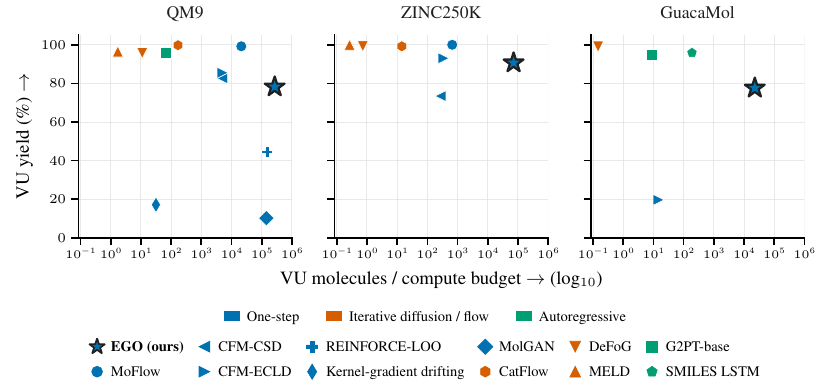}
    \caption{
        \textbf{Distinct valid yield versus estimated inference efficiency.}
        Axes are valid-and-unique yield (VU) and VU per $10^{12}$ estimated dense operations ($E_{\mathrm{VU}}$, log scale). The GuacaMol CFM-ECLD marker is a local 3/84-epoch feasibility checkpoint, not a converged baseline, and is excluded from ratio and ranking claims. Source protocols and sampling budgets differ, and $E_{\mathrm{VU}}$ is not wall-clock throughput. Full accounting and methods references appear in Appendix~\ref{app:rq1}.
    }
    \label{fig:molecular-validity-efficiency}
\end{figure}

\subsection{RQ1: \methodname Reaches High Molecular Yield at Low Training and Estimated Inference Cost}
\Cref{fig:molecular-validity-efficiency} places \methodname in the
high-yield, high-efficiency region of each molecular benchmark. With one
network evaluation per proposal, it attains VU yields of 78.21\%, 90.66\%,
and 77.61\% on QM9, ZINC250K, and GuacaMol. Relative to CFM-ECLD, \methodname achieves approximately
$50\times$ and $200\times$ the valid-and-unique yield per
estimated dense-projection operation on QM9 and ZINC250K,
respectively, while VU is lower by 7.24 and 2.37
percentage points. These ratios compare reconstructed dense projection work
under the source protocols rather than matched-hardware runtime. The efficiency gain trades off distributional fidelity. FCD, Fr\'echet ChemNet Distance \cite{DBLP:journals/jcisd/PreuerRUHK18}, is 3.53 versus
2.16 on QM9, 30.36 versus 20.0 on ZINC250K for \methodname and
CFM-ECLD, and on GuacaMol $34.62\pm1.58$. Thus, \methodname favors a high yield of distinct valid
candidates over closer matching of the reference distribution. For an extensive and detailed comparison, see Appendix~\ref{app:rq1}. 

Figure~\hyperref[fig:molecular-training-evidence]{\ref*{fig:molecular-training-evidence}a} tests whether the yield advantage persists under controlled training budgets. In 30 minutes on ZINC250K, \methodname reaches 77.9\% Strict VUN, compared with 53.7\% without rarity, 5.8\% for Soft, and 29.5\% for STE. This gives $13.47\times$ Soft's yield and a $9.29\times$ advantage without rarity, although the gradient methods complete $\approx6\times$ as many updates (Appendix~\ref{app:zinc-time-budget}). Across the six molecular size settings (two datasets, three size caps each), \methodnamez also outperforms STE, while rarity
improves yield throughout (Appendix~\ref{app:ksweep}).
In a separate five-seed QM9 experiment matching architecture, reward, and fixed feature scales, \methodname reaches 38.4 Strict VUN versus 12.0 for RLOO at 11.52 million hard-reward graph evaluations. Without rarity or rank shaping, this supports a yield advantage for \methodname's update recipe at this budget beyond differences in reward construction (Appendix~\ref{app:qm9-optimizer-control}). 

Figure~\hyperref[fig:molecular-training-evidence]{\ref*{fig:molecular-training-evidence}b}
compares the computational resources required by \methodname with those of baselines other than CFM. \methodname records 0.30, 2.10, and 14.85 optimizer-update hours on QM9, ZINC250K, and GuacaMol, respectively. DeFoG reported times are 6.5, 14, and 141 hours, while Cometh reports 6 on QM9 and 96 hours on GuacaMol. CFM reports approximately 48 and 144 hours on QM9 and ZINC250K, while its configured GuacaMol schedule projects to 2.2--2.8k optimizer-update hours from our local three-epoch pilot (see Appendix~\ref{app:cfm-guacamol}). Hardware, GPU count, and timing scope differ, so these values provide resource context rather than measured speedups. Even with this caveat, \methodname lies at the low-time end on every dataset.

Overall, the controlled experiments support higher candidate yield under limited training time, while the external comparison shows higher yield per estimated inference work; both gains come with a measurable distributional-fidelity trade-off. Appendix~\ref{app:ksweep} reports the full six-setting molecular-size sweep, extended STE and RLOO comparisons, rarity ablations, and the complete training-time audit.

\begin{figure*}[t]
    \centering
    \includegraphics[width=0.90\textwidth]{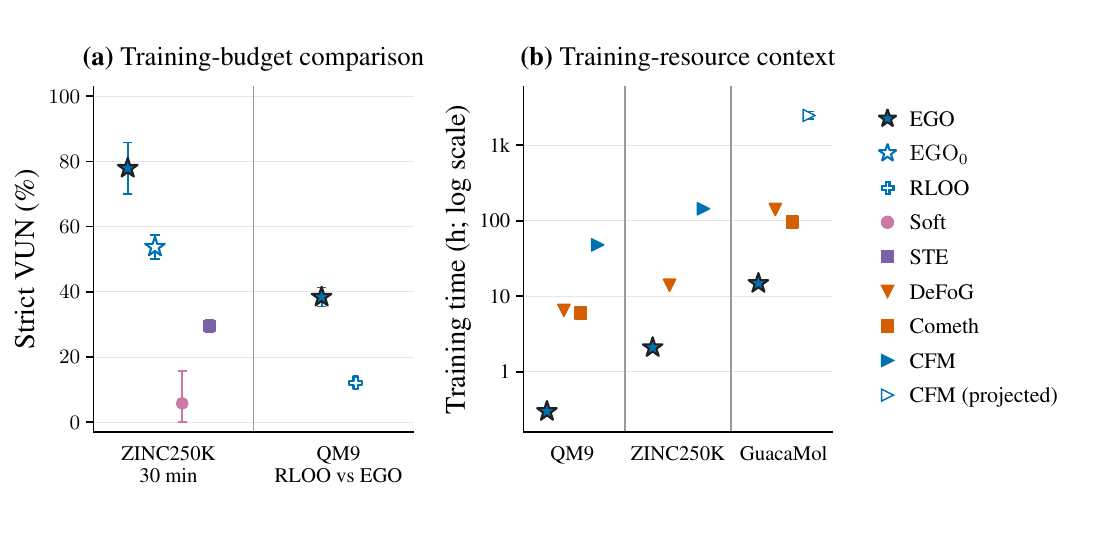}
    \caption{
    \textbf{Training efficiency and resource context.}
    \textbf{(a)} Strict VUN under a shared active 30-minute ZINC250K wall-clock budget; RLOO versus \methodname on QM9.
    \textbf{(b)} Training time (log scale).
    \methodname reports optimizer-update time; DeFoG, Cometh, and CFM report elapsed training time on QM9/ZINC250K. $^\ddagger$Cometh GuacaMol used two A100s; $^\dagger$CFM GuacaMol: 84-epoch constant-throughput projection from a local three-epoch pilot. See Appendix~\ref{app:external-guacamol}.
}
    \label{fig:molecular-training-evidence}
\end{figure*}

\subsection{RQ2: \methodname Produces More Filter-Passing Proposals at Matched Time}
Scaffold completion tests whether low generation cost allows more
candidates to be screened within a common time budget. We compare with
MoLeR trained from scratch on the same ZINC250K scaffold split,
using the same 100 held-out queries and screening rules
(details in Appendix~\ref{app:scaffold}).
Both models run on the same GPU after warmup. For each query, the time
MoLeR needs to generate and screen 256 proposals (a single batch) defines the deadline
for both methods. \methodname screens graph outputs directly and the final
SMILES export is outside this interval. Screening requires a valid,
connected completion containing the scaffold, QED (Quantitative Estimate
of Drug-likeness) $\geq0.7$, SA (Synthetic Accessibility) $\leq4$, and no
additional PAINS (pan-assay interference compounds) alert types relative
to the core.
At a mean deadline of 6.08 seconds per query, \methodname screens
1892.7 proposals on average, compared with 256 for MoLeR.
Of these, 547.3 versus 54.5 pass the filters, a $10.05\times$
increase at matched time (\Cref{tab:scaffold}). 

Averaging passing counts across the three replicates for
each query, \methodname returns more passing proposals on 73 of 100
scaffolds, with 13 ties and 14 MoLeR wins. MoLeR covers more queries
and more often returns the higher-QED passing candidate
(Appendix~\ref{app:scaffold-matched-time}).
\Cref{fig:scaffold_example} shows the selected completions for one query.
In a complementary evaluation with $100{,}000$ proposals,
\methodname produces $20245$ distinct filter-passing completions
versus $16684$ for MoLeR, a $1.21\times$ increase, while achieving
a measured $44.3\times$ speedup in generation including SMILES
conversion (Appendix~\ref{app:scaffold-results}).

\begin{table}[t]
\centering
\begin{minipage}[c]{0.67\linewidth}
\caption{\textbf{Scaffold screening at matched time.}
Mean $\pm$ sample SD over three runs on 100 queries: three \methodname
training seeds and three sampling seeds from one MoLeR checkpoint. Fractions divide passing by screened counts
within each run. Query wins compare mean passing counts across runs
(13 ties). Both methods share each query's deadline.
Protocol: Appendix~\ref{app:scaffold-matched-time}.}
\label{tab:scaffold}
\setlength{\tabcolsep}{2.5pt}
\renewcommand{\arraystretch}{1.16}
\begin{tabularx}{\linewidth}{@{}>{\raggedright\arraybackslash}X
  r@{\,${}\pm{}$\,}l@{\hspace{10pt}}r@{\,${}\pm{}$\,}l@{}}
\toprule
Metric & \multicolumn{2}{c}{\methodname} & \multicolumn{2}{c}{MoLeR} \\
\midrule
Mean deadline/query (s)
    & $6.08$ & $0.10$ & $6.08$ & $0.10$ \\
\midrule
Screened proposals/query
    & $\mathbf{1892.7}$ & $\mathbf{63.7}$ & $256.0$ & $0.0$ \\
Filter-passing proposals/query
    & $\mathbf{547.3}$ & $\mathbf{15.9}$ & $54.5$ & $0.4$ \\
Passing fraction (\%)
    & $\mathbf{28.95}$ & $\mathbf{1.58}$ & $21.27$ & $0.14$ \\
\midrule
Scaffold coverage (/100)
    & \multicolumn{2}{c}{$79$} & \multicolumn{2}{c}{$\mathbf{85}$} \\
Query wins (/100)
    & \multicolumn{2}{c}{$\mathbf{73}$} & \multicolumn{2}{c}{$14$} \\
\bottomrule
\end{tabularx}
\end{minipage}\hfill
\begin{minipage}[c]{0.31\linewidth}
\centering
\includegraphics[width=\linewidth]{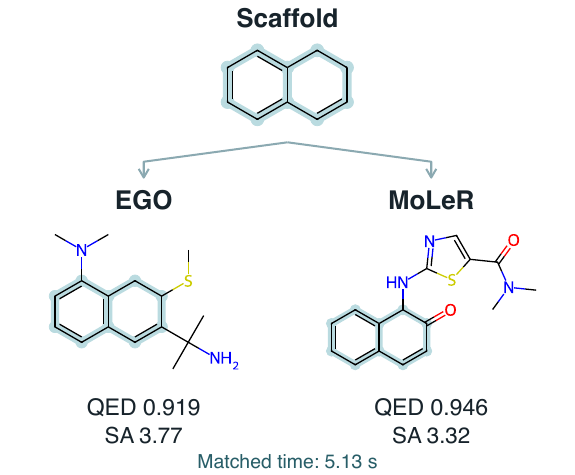}
\captionof{figure}{Highest-QED filter-passing completions for the
prespecified scaffold 23 under a shared 5.13-second deadline.
The supplied core is highlighted.}
\label{fig:scaffold_example}
\end{minipage}
\end{table}
\subsection{RQ3: \methodname Produces More Distinct Elite Neural Architectures}
\label{sec:nas}
NAS-Bench-101~\citep{ying2019nasbench101} tests whether
the candidate-library perspective extends to neural architectures.
The task is to learn a one-step proposal distribution from an
offline corpus of high-performing architectures and generate
distinct candidates of similarly high quality beyond those used
for training. \methodname, STE, and Soft share the generator architecture,
initialization within each seed, data split, sampling noise,
and checkpoint selection rule. \methodname evaluates Energy and a fixed
structural-violation penalty on hard graphs. STE preserves the
same numerical hard forward objective while using surrogate
derivatives, and Soft optimizes their relaxed
objective. This instantiation of \methodname uses neither a rarity reward nor an explicit duplicate penalty, allowing us to assess candidate yield without these additional incentives. The Energy objective itself still encourages diversity among generated samples. All methods receive 15,000 training updates, matching update count rather than training work, with cumulative optimizer-update times of approximately 46 seconds for \methodname and 34 seconds for STE and Soft (Appendix~\ref{app:nas}).

Here, \emph{elite} denotes an architecture whose benchmark
validation accuracy is at or above the 90th percentile
in the architecture training partition.
The evaluation asks how many different elite architectures,
excluded from both training and checkpoint selection, each model
returns within 10000 proposals. 
The outcome of interest is the number of distinct
qualifying alternatives available for further consideration.
\Cref{tab:nas} shows that \methodname produces the largest
collection of distinct held-out elite architectures. From 10,000
proposals, it returns $221.7\pm7.5$ such candidates, compared
with $138.0\pm28.9$ for STE and $143.7\pm18.6$ for Soft
across the three seeds, providing $1.61\times$ and $1.54\times$
as many alternatives on average, respectively. The advantage
also appears in the broader collection of valid architectures:
over half of the proposal budget yields distinct valid candidates
with \methodname, compared with one third for STE
and one quarter for Soft. Thus, the candidate-yield advantage observed in molecular
generation also appears in this architecture-search setting. 

This broader library comes with a trade-off. Soft achieves
higher validity and mean test accuracy, and lower feature MMD (Maximum Mean Discrepancy),
but repeats architectures more often. The additional distinct
candidates from \methodname offer more options for downstream
comparison and screening beyond the accuracy cutoff.
The advantage is task-dependent: on binarized MNIST,
Soft achieves better distribution matching and broader
predicted digit coverage than \methodname
(Appendix~\ref{app:mnist}).

\begin{table}[t]
\centering
\caption{\textbf{NAS-Bench-101 results with SD from 10,000 proposals per selected
model}. The elite cutoff is the 90th percentile of
benchmark validation accuracy in the architecture training
partition. Mean test accuracy is averaged over valid proposals,
including repetitions. Feature MMD
compares the feature distributions of generated architectures
and the elite target; lower values indicate a closer match.
Protocol details appear in Appendix~\ref{app:nas}.}
\label{tab:nas}
\setlength{\tabcolsep}{4pt}
\begin{tabular}{@{}l*{3}{r@{${}\pm{}$}l}@{}}
\toprule
Metric
& \multicolumn{2}{c}{\methodname (ES)}
& \multicolumn{2}{c}{STE}
& \multicolumn{2}{c}{Soft} \\
\midrule
Validity (\%) $\uparrow$
& $72.35$ & $1.76$
& $77.45$ & $6.27$
& $\mathbf{89.56}$ & $\mathbf{0.82}$ \\

Valid-and-unique yield (\%) $\uparrow$
& $\mathbf{55.51}$ & $\mathbf{0.29}$
& $33.25$ & $6.02$
& $25.12$ & $4.00$ \\

Held-out elite yield (\%) $\uparrow$
& $\mathbf{2.22}$ & $\mathbf{0.08}$
& $1.38$ & $0.29$
& $1.44$ & $0.19$ \\

Mean test accuracy (\%) $\uparrow$
& $91.85$ & $0.06$
& $92.38$ & $0.29$
& $\mathbf{92.70}$ & $\mathbf{0.10}$ \\

Feature MMD ($\times 10^{-3}$) $\downarrow$
& $27.70$ & $2.70$
& $10.90$ & $5.20$
& $\mathbf{3.10}$ & $\mathbf{0.70}$ \\
\bottomrule
\end{tabular}
\end{table}
\section{Conclusion}
\label{sec:conclusion}

We introduced \methodname, a framework for training one-step discrete
generators using objectives evaluated directly on hard outputs. Across
molecular generation, scaffold completion, and NAS-Bench-101, the
approach produces large sets of distinct qualifying candidates while
retaining one-network-evaluation inference. These results support
finite-budget candidate yield as a useful complement to validity,
average quality, and distributional fidelity. The results also highlight the central trade-off of the approach:
additional objective evaluations are spent during training to shape a cheap, reusable proposal distribution at deployment. This trade-off is
most useful when hard feasibility, diversity, or downstream utility
are not well captured by differentiable surrogates. Future work could improve training efficiency, scale the approach to
larger generators and more complex discrete structures, and incorporate
richer downstream feedback into the training objective.

\textbf{Limitations.}
The benefits of \methodname are task-dependent. Higher candidate yield
can come at the expense of distributional fidelity, average proposal
quality, or coverage across conditions, and relaxed training remains
preferable in some settings. Population evaluation increases training
cost, and scalability to substantially larger generators remains to be
established. The learned distribution also depends on the chosen
features, penalties, and rewards. Finally, the molecular screening
results rely on heuristic filters and do not establish experimental
usefulness.

\subsection*{AI use statement}
Generative AI tools were used to assist with drafting portions of the
manuscript and revising the text for clarity, organization, and readability.
The authors take full responsibility for the final manuscript, including
its originality, technical claims, reported results, and references.

\subsection*{Ethics statement}
This work investigates discrete generative modeling using existing benchmark datasets and does not involve
human participant studies or clinical experiments. The molecular experiments
evaluate computationally generated candidates using structural criteria and
heuristic property filters; these measures do not establish safety,
biological efficacy, or experimental synthesizability. As with other
molecular design methods, faster candidate generation could support
beneficial discovery but could also be misused to propose harmful compounds.
Any downstream use therefore requires application-specific assessment,
appropriate safety screening, and expert oversight before experimental
testing. The learned distributions depend on the benchmark data and chosen
objectives, and the reported improvements should not be interpreted as
guarantees of usefulness outside the evaluated settings.

\subsection*{Reproducibility statement}
Section~\ref{sec:method} describes the proposed framework, and
Appendix~\ref{app:proofs} provides the assumptions, derivations, and proofs
of the theoretical results. Appendix~\ref{app:method} describes the
generator architectures, objective functions, training procedures, and
evaluation conventions. Appendices~\ref{app:rq1}, \ref{app:scaffold},
and~\ref{app:rq3} document task-specific data preparation, experimental
protocols, checkpoint selection, sampling budgets, and additional results.
Appendix~\ref{app:external-accounting} explains the computational accounting
and timing conventions, including the distinction between estimated dense
operation counts and measured runtime. The experimental appendices also
state limitations of the available run records and qualify comparisons
with published reference results.
A documented code repository, including instructions for reproducing all
experiments presented in this paper, will be provided as part of the
supplementary materials. The repository will also be made publicly
available when a public version of the paper is released.

\bibliographystyle{plainnat}
\bibliography{iclr2027_conference}

\clearpage

\appendix


\section*{Appendix contents}
\begingroup
\small
\setlength{\parindent}{0pt}
\setlength{\parskip}{0pt}

\newcommand{\egoappsection}[2]{%
\par\addvspace{5pt}%
\noindent\hyperref[#1]{\textbf{\ref*{#1}\quad #2}}%
\nobreak\hfill\hyperref[#1]{\pageref*{#1}}\par
}
\newcommand{\egoappsubsection}[2]{%
\noindent\hspace*{1em}%
\hyperref[#1]{\ref*{#1}\quad #2}%
\nobreak\dotfill\hyperref[#1]{\pageref*{#1}}\par
}

\egoappsection{app:proofs}{Theoretical Results}
\egoappsubsection{app:finite-budget-proof}{Finite-Budget Distinct Yield}
\egoappsubsection{app:shared-gumbel}{Stability under Shared Gumbel Noise}

\egoappsection{app:related-work}{Background and Related Work}
\egoappsubsection{app:related-work-training}{One-Step Generation and Discrete Training}
\egoappsubsection{app:related-work-applications}{Structured Generation and Applications}

\egoappsection{app:method}{Model, Training, and Evaluation Details}
\egoappsubsection{app:method-molecular}{Molecular Generation}
\egoappsubsection{app:method-scaffold}{Scaffold Completion}
\egoappsubsection{app:method-nas}{Neural Architecture Generation}
\egoappsubsection{app:method-mnist}{Binarized MNIST}

\egoappsection{app:rq1}{RQ1: Molecular Generation}
\egoappsubsection{app:external-molecular}{Shared Evaluation Protocol}
\egoappsubsection{app:external-qm9}{QM9}
\egoappsubsection{app:zinc}{ZINC250K}
\egoappsubsection{app:external-guacamol}{GuacaMol}
\egoappsubsection{app:external-accounting}{Inference Cost and Training Time}

\egoappsection{app:scaffold}{RQ2: Scaffold Completion}
\egoappsubsection{app:scaffold-protocol}{Training and Scaffold Only Inference}
\egoappsubsection{app:scaffold-results}{Screening and Candidate Yield}
\egoappsubsection{app:scaffold-timing}{Generation to SMILES Time}
\egoappsubsection{app:scaffold-matched-time}{Generation and Screening at Matched Time}

\egoappsection{app:rq3}{RQ3: Beyond Molecular Generation}
\egoappsubsection{app:nas}{NAS-Bench-101}
\egoappsubsection{app:mnist}{Binarized MNIST Control}

\endgroup

\section{Theoretical Results}
\label{app:proofs}

\subsection{Finite-Budget Distinct Yield}
\label{app:finite-budget-proof}

\paragraph{Proof of Proposition~\ref{prop:finite_budget_yield}} Fix a condition $c$ and write
\[
p_u = p_\theta(u\mid c),
\qquad u\in\mathcal A_c.
\]
For each successful identity $u$, define
\[
I_u
=
\mathbf 1\!\left\{
u\in
\{\kappa_c(X_i):
a(X_i,c)=1,\ i\in[N]\}
\right\}.
\]
Then
\[
U_N=\sum_{u\in\mathcal A_c} I_u.
\]
Since the summands are nonnegative, linearity of expectation (or Tonelli's theorem if $\mathcal A_c$ is countably infinite) gives
\[
\mathbb E[U_N\mid c]
=
\sum_{u\in\mathcal A_c}
\Pr(I_u=1\mid c).
\]

For a fixed $u$, each proposal produces the correct identity $u$ with probability $p_u$.
Because the $N$ proposals are independent, the probability that $u$ does not occur in any of them is $(1-p_u)^N$.
Therefore
\[
\Pr(I_u=1\mid c)
=
1-(1-p_u)^N,
\]
and hence
\[
\mathbb E[U_N\mid c]
=
\sum_{u\in\mathcal A_c}
\left[1-(1-p_u)^N\right].
\]

For the upper bound, the union bound gives
\[
1-(1-p_u)^N
\leq Np_u.
\]
Summing over successful identities yields
\[
\mathbb E[U_N\mid c]
\leq
N\sum_{u\in\mathcal A_c}p_u
=
N\alpha_\theta(c).
\]

For the lower bound, let
\[
E_{i,u}
=
\{a(X_i,c)=1,\ \kappa_c(X_i)=u\}.
\]
Then $\Pr(E_{i,u}\mid c)=p_u$, and independence gives $\Pr(E_{i,u}\cap E_{j,u}\mid c)=p_u^2$ for $i\neq j$.
The second-order Bonferroni inequality therefore gives
\[
1-(1-p_u)^N
=
\Pr\!\left(\bigcup_{i=1}^N E_{i,u}\,\middle|\,c\right)
\geq
Np_u-\binom{N}{2}p_u^2.
\]
Summing over $u$ gives
\[
\mathbb E[U_N\mid c]
\geq
N\sum_{u\in\mathcal A_c}p_u
-
\binom{N}{2}
\sum_{u\in\mathcal A_c}p_u^2
=
N\alpha_\theta(c)
-
\binom{N}{2}\Gamma_\theta(c).
\]
Since $U_N\geq 0$ almost surely, its expectation is also nonnegative.
Combining the two lower bounds gives
\[
\mathbb E[U_N\mid c]
\geq
\max\left\{
0,\,
N\alpha_\theta(c)
-
\binom{N}{2}\Gamma_\theta(c)
\right\}.
\]
Together with the upper bound, this proves the result.
\hfill$\square$

\subsection{Stability under Shared Gumbel Noise}
\label{app:shared-gumbel}

Evolutionary updates compare objective values obtained from antithetic parameter perturbations.
When categorical samples are drawn independently for the two evaluations, differences between their outputs can arise from both the parameter perturbation and the sampling randomness.
Sharing the Gumbel noise instead couples the categorical decisions while preserving the marginal sampling distribution of each perturbed generator.
The following result quantifies the resulting stability.
Throughout this analysis, the latent input and conditioning information are held fixed.

\begin{lemma}[Stability under shared Gumbel noise]
  \label{lem:shared-gumbel}
  Let $\ell(\theta)\in\mathbb{R}^K$ denote the logits at one categorical site.
  For a fixed perturbation $\Delta$, define
  \[
  \theta^\pm = \theta \pm \Delta,
  \qquad
  \delta(\Delta)
  =
  \max_{\eta\in\{-1,+1\}}
  \left\|
  \ell(\theta+\eta\Delta)-\ell(\theta)
  \right\|_\infty .
  \]
  Let $\xi\in\mathbb{R}^K$ have independent standard Gumbel entries, shared between the two evaluations, and define
  \[
  X^\pm
  =
  \arg\max_{k\in\{1,\ldots,K\}}
  \left\{
  \ell_k(\theta^\pm)+\xi_k
  \right\}.
  \]
  Then
  \[
  \Pr(X^+\neq X^- \mid \Delta)
  \leq
  \min\left\{
  1,\,
  e^{2\delta(\Delta)}-1
  \right\}.
  \]
\end{lemma}

\begin{proof}
  Write
  \[
  \ell^0=\ell(\theta),
  \qquad
  \ell^\pm=\ell(\theta\pm\Delta),
  \qquad
  \delta=\delta(\Delta).
  \]
  By definition,
  \[
  \|\ell^\pm-\ell^0\|_\infty\leq\delta.
  \]

  Using the same Gumbel vector $\xi$, define the central noisy scores
  \[
  S_k=\ell^0_k+\xi_k,
  \]
  and let
  \[
  X^0=\arg\max_k S_k.
  \]
  Because the Gumbel distribution is continuous, ties occur with probability zero.
  Let
  \[
  M
  =
  S_{X^0}
  -
  \max_{j\neq X^0}S_j
  \]
  be the gap between the largest and second-largest central noisy scores.

  Suppose that $M>2\delta$.
  For every $j\neq X^0$ and $q\in\{-1,+1\}$,
  \[
  \begin{aligned}
    (\ell^q_{X^0}+\xi_{X^0})
    -
    (\ell^q_j+\xi_j)
    &=
    (S_{X^0}-S_j)
    +
    (\ell^q_{X^0}-\ell^0_{X^0})
    -
    (\ell^q_j-\ell^0_j) \\
    &\geq
    M-2\delta
    >0.
  \end{aligned}
  \]
  Hence both perturbed evaluations select the same category $X^0$.
  Therefore
  \[
  \{X^+\neq X^-\}
  \subseteq
  \{M\leq 2\delta\},
  \]
  and thus
  \[
  \Pr(X^+\neq X^-\mid\Delta)
  \leq
  \Pr(M\leq2\delta\mid\Delta).
  \]

  It remains to bound the probability of a small top-two gap under Gumbel-max sampling.
  Let
  \[
  \lambda_k=e^{\ell^0_k},
  \qquad
  \Lambda=\sum_{k=1}^K\lambda_k.
  \]
  Using the exponential-race representation of Gumbel-max, write
  \[
  \xi_k=-\log E_k,
  \qquad
  E_k\stackrel{\mathrm{iid}}{\sim}\operatorname{Exp}(1),
  \]
  and define
  \[
  T_k=\frac{E_k}{\lambda_k}.
  \]
  Then the $T_k$ are independent exponential random variables with rates $\lambda_k$, and maximizing $S_k$ is equivalent to minimizing $T_k$.
  If $T_{(1)}$ and $T_{(2)}$ denote the smallest and second-smallest race times, respectively, then
  \[
  M=\log\frac{T_{(2)}}{T_{(1)}}.
  \]

  For any $t\geq0$, condition on category $i$ winning the race at time $s$.
  The event $M>t$ requires every competing race time to exceed $e^t s$.
  Hence
  \[
  \begin{aligned}
    \Pr(M>t,\ X^0=i)
    &=
    \int_0^\infty
    \lambda_i e^{-\lambda_i s}
    \prod_{j\neq i}
    e^{-\lambda_j e^t s}
    \,ds \\
    &=
    \frac{\lambda_i}
    {\lambda_i+e^t(\Lambda-\lambda_i)}.
  \end{aligned}
  \]
  Since
  \[
  \lambda_i+e^t(\Lambda-\lambda_i)
  \leq e^t\Lambda,
  \]
  we obtain
  \[
  \Pr(M>t,\ X^0=i)
  \geq
  e^{-t}\frac{\lambda_i}{\Lambda}.
  \]
  Summing over $i$ gives
  \[
  \Pr(M>t)\geq e^{-t},
  \]
  and therefore
  \[
  \Pr(M\leq t)
  \leq
  1-e^{-t}
  \leq
  e^t-1.
  \]
  Setting $t=2\delta$ yields
  \[
  \Pr(X^+\neq X^-\mid\Delta)
  \leq
  e^{2\delta}-1.
  \]
  Combining this with the trivial upper bound of one proves the claim.
\end{proof}

\begin{corollary}[Multiple categorical sites]
  \label{cor:shared-gumbel-multisite}
  Suppose an object is represented by $D$ categorical sites, and let $\delta_s(\Delta)$ denote the corresponding logit bound at site $s$.
  If the Gumbel noise at each site is shared between the two perturbed evaluations, then
  \[
  \Pr(X^+\neq X^- \mid \Delta)
  \leq
  \min\left\{
  1,\,
  \sum_{s=1}^D
  \left(
  e^{2\delta_s(\Delta)}-1
  \right)
  \right\}.
  \]
  The same bound applies to final objects obtained by applying a common deterministic decoder to the sampled categorical vectors.
\end{corollary}

\begin{proof}
  If the complete sampled vectors differ, then at least one categorical site differs.
  By the union bound and Lemma~\ref{lem:shared-gumbel},
  \[
  \begin{aligned}
    \Pr(X^+\neq X^- \mid \Delta)
    &\leq
    \sum_{s=1}^D
    \Pr(X_s^+\neq X_s^- \mid \Delta) \\
    &\leq
    \sum_{s=1}^D
    \left(
    e^{2\delta_s(\Delta)}-1
    \right).
  \end{aligned}
  \]
  Combining this with the trivial upper bound of one gives the stated result.

  If a common deterministic decoder is applied to both sampled vectors, different decoded objects can occur only if the underlying categorical vectors differ.
  The same bound therefore applies after decoding.
\end{proof}

For comparison, if the two evaluations instead use independent Gumbel noise and $\Delta=0$, then they are independent draws from the same categorical distribution
\[
\pi=\operatorname{softmax}(\ell(\theta)).
\]
Consequently,
\[
\Pr(X^+\neq X^- \mid \Delta=0)
=
1-\sum_{k=1}^K \pi_k^2.
\]
Thus, independent sampling retains a nonzero source of disagreement even when the two parameter vectors are identical.

By contrast, under shared noise, the bound in Lemma~\ref{lem:shared-gumbel} vanishes with the size of the logit perturbation.
For small $\delta$,
\[
e^{2\delta}-1
=
2\delta+O(\delta^2).
\]
Hence, sufficiently small parameter perturbations can change a sampled category only when the shared noisy logits lie near an argmax boundary.

This result characterizes the stability of the sampled decisions.
It does not establish variance reduction for every possible hard objective, nor does it imply that the resulting rank-shaped ES direction is an unbiased gradient estimator.
Its role is instead to motivate paired evaluations that remove disagreements caused solely by resampling, making the comparison more directly reflect the effect of the parameter perturbation.

\section{Background and Related Work}
\label{app:related-work}

\subsection{One-Step Generation and Discrete Training}
\label{app:related-work-training}

\paragraph{One-step generation.}
Direct generators learn a mapping from latent noise to a complete output. Once trained, they can generate a new proposal with a single network evaluation. MolGAN applies this idea to molecular graphs by jointly predicting node and edge distributions with a feedforward network, followed by categorical sampling~\citep{decao2018molgan}. More recently, Categorical Flow Maps learn continuous maps toward categorical outputs through self-distillation, allowing generation in one or a few steps~\citep{roos2026categoricalflowmaps}. Their CSD and ECLD variants use different consistency objectives within the same framework. We include both as one-step references. Our work studies how to train a direct generator when the objective is evaluated on complete discrete samples. Here, one-step generation refers to the network evaluation. Decoding, validity checks, and conversion into usable outputs also contribute to the total generation cost.

\paragraph{Distribution matching without explicit likelihoods.}
A generator can learn by comparing generated and reference samples without computing the probability of each output. Generative Moment Matching Networks train feedforward generators by minimizing maximum mean discrepancy (MMD), which compares distributions through kernel mean embeddings~\citep{li2015gmmn,gretton2012mmd}. Energy statistics provide a related approach based on distances within and between two sets of samples~\citep{szekely2013energy}. Drifting Models use attraction and repulsion to guide the generated distribution toward the data during training, while retaining one-step inference~\citep{deng2026drifting}. These approaches motivate our batch objectives. We compare generated and reference structures through features that describe their properties, and study Energy and Drift as two distinct objectives. The choice of features matters because different graph distributions can share the same feature distribution. We therefore assess feature matching together with validity, diversity, and the number of useful candidates produced.

\paragraph{Learning through discrete decisions.}
Hard categorical sampling prevents ordinary backpropagation through the sampled decision. Gumbel-Softmax replaces this decision with a differentiable relaxation, while straight-through estimators use hard decisions in the forward pass and approximate their derivatives in the backward pass~\citep{bengio2013estimating,jang2017gumbel}. Both approaches provide practical training signals, but improving a relaxed objective or following an approximate gradient may not improve the final discrete outputs. Score-function estimators offer another option. They use gradients of log probabilities to estimate how parameters affect the expected score, without differentiating through the sampled output itself. Their variance can make training difficult~\citep{bengio2013estimating}. Learned reward models provide a further route. MolGAN combines adversarial training with a reward network that approximates external molecular scores and supplies gradients for property optimization~\citep{decao2018molgan}. Our ES implementation uses scores computed on hard batches under parameter perturbations. This allows new criteria to enter training without requiring a differentiable approximation of each one.

\paragraph{Evolution strategies and low-rank perturbations.}
Evolution strategies update model parameters using evaluations of perturbed models. Salimans et al.\ show how this approach can scale to neural networks and use paired positive and negative perturbations to estimate an update~\citep{salimans2017es}. Since the method only needs objective values, these evaluations can include discrete sampling, structural checks, and other non-differentiable operations. EGGROLL makes neural ES more efficient by representing matrix perturbations with low-rank factors, reducing the memory and computation needed to evaluate a population~\citep{sarkar2025eggroll}. Combining several low-rank perturbations can produce an update with higher rank. Our implementation builds on this approach using rank-one perturbations and scores computed on complete batches of discrete outputs. The population consists of perturbed generators during training. At inference, we sample from a single learned generator. The cost of evaluating the training population must therefore be assessed separately from the cost of generating new samples.

\subsection{Structured Generation and Applications}\label{app:related-work-applications}

\paragraph{Molecular graph generation.}
Molecular generators use several approaches to construct graphs. MolGAN predicts a complete graph directly~\citep{decao2018molgan}, while G2PT represents graphs as token sequences and generates them through autoregressive next-token prediction~\citep{chen2025g2pt}. GDSS models the joint evolution of nodes and edges with stochastic differential equations~\citep{jo2022gdss}. DiGress uses discrete noise and denoising steps to generate categorical graphs~\citep{vignac2023digress}. GruM and DeFoG further develop diffusion and flow-based methods for graph generation~\citep{jo2024grum,qin2025defog}. These methods provide context for the balance between sample quality and generation cost. We examine this balance through distributional fidelity, validity, diversity, and computational cost. Comparisons with external models also depend on their data, decoding procedures, and evaluation protocols, which we distinguish from our controlled comparisons between training methods.

\paragraph{Scaffold completion and candidate libraries.}
Scaffold completion requires a generator to preserve a supplied molecular core while proposing extensions. MoLeR supports this task by adding atoms and motifs through sequential decisions~\citep{maziarz2022moler}. Our scaffold generator predicts the variables needed for a complete proposal in one network evaluation, and its decoder preserves the supplied core. We compare the resulting systems using the same public scaffold queries and proposal budget. This comparison concerns the candidate libraries returned by the complete systems, including their respective training and decoding choices. We measure distinct usable completions, distinct completions that pass the screening filters, scaffold coverage, and generation throughput. These quantities capture different aspects of performance. Producing many candidates for a few scaffolds can increase total yield while leaving other queries unsolved, and repeated valid outputs consume the budget without expanding the library.

\paragraph{Offline architecture generation.}
NAS-Bench-101 provides a finite space of neural architectures with recorded performance measurements, allowing reproducible evaluation without training every proposed architecture from scratch~\citep{ying2019nasbench101}. We use an offline labeled corpus to learn a reusable distribution of architecture proposals. Evaluation measures how many distinct held-out elite architectures the generator produces within a fixed sampling budget. This setting differs from adaptive architecture search, where new performance evaluations guide later proposals. Access to the offline labels is part of the supervision budget and must be considered when comparing these settings. Our experiment asks whether training on discrete graph scores can produce a useful and diverse pool of architecture candidates.

\paragraph{Scope of our contribution.}
Direct generation, distribution matching, and evolution strategies are established ideas. Our contribution is to combine them in a practical training approach for discrete structures and evaluate when this combination is useful. Energy and Drift define distributional objectives, while ES, straight-through estimation, and continuous relaxation define different ways to train the generator. Keeping these choices separate helps identify what each comparison measures. Binarized MNIST provides a further setting in which to compare evolutionary and gradient-based training. Across the experiments, we assess sample quality, useful candidate yield, and training and inference costs to understand when evolutionary training offers a useful trade-off.

\section{Model, Training, and Evaluation Details}
\label{app:method}
\label{app:architectures}
\label{app:objectives}
\label{app:rarity}
\label{app:optimization}
\label{app:gradient-baselines}
\label{app:evaluation}

This section covers the experiments discussed in the main paper, including
the binarized-MNIST control cited in RQ3. For each task, we describe what
the generator produces, how its fitness is built, and how we evaluate it.
Across tasks, one-step generation means one evaluation of the neural
generator per proposal. Sampling, graph assembly, chemical checks, and
benchmark lookup can still add computation after that evaluation.

\paragraph{Shared evolutionary update.}
Let $\theta$ denote the generator parameters and let $\ell$ index a
parameter block, such as a weight matrix or bias vector. Each update
evaluates a population of $P=192$ candidates in $J=P/2=96$ pairs.
Pair $j\in\{1,\ldots,J\}$ uses blocks
$\theta_{j\ell}^\pm=\theta_\ell\pm\sigma_\ell\epsilon_{j\ell}$,
where $\sigma_\ell>0$ is the perturbation scale.
For a weight matrix of shape $d_{\rm out}\times d_{\rm in}$,
$\epsilon_{j\ell}=ab^\top$, with independent standard Gaussian vectors
$a\in\mathbb R^{d_{\rm out}}$ and $b\in\mathbb R^{d_{\rm in}}$;
$d_{\rm in}$ and $d_{\rm out}$ are its input and output widths.
Bias perturbations are standard Gaussian vectors of the same length as
the bias. Perturbations are independent across pairs and parameter blocks.
All candidates share latent noise, sampling noise,
conditions, and reference examples. This makes their scores comparable
without adding differences caused by resampling.
For parameter block $\ell$, the update direction is
\begin{equation}
  \widehat g_\ell=
  \frac{1}{P\sigma_\ell}
  \sum_{j=1}^{J}(u_j^+-u_j^-)\epsilon_{j\ell},
  \label{eq:app-es-direction}
\end{equation}
where $u_j^\pm$ are the scalar utilities of the two candidates and
$\widehat g_\ell$ is the ascent direction for block $\ell$.
Adam applies this direction as an ascent step. We use $t$ to index
optimizer updates and $q\in\{1,\ldots,P\}$ to index individual candidates.
For any array $W$, $\operatorname{RMS}(W)$ denotes the square root of
the mean of its squared entries. The task descriptions specify whether utilities use
ranks or raw losses. Equal update counts do not imply equal training work:
an evolutionary update evaluates a population, whereas a gradient update
optimizes one generated batch.

\paragraph{\methodname training algorithm.}
Algorithm~\ref{alg:app-ego-training} summarizes the shared training procedure. The central
network is the unperturbed generator with parameters $\theta$.
Each task supplies its own fitness components and utility rule, defined in
Sections~\ref{app:method-molecular}--\ref{app:method-mnist}.

\begin{algorithm}[tbp]
\caption{\methodname training}
\label{alg:app-ego-training}
\small
\begin{algorithmic}[1]
\Require Data splits, task settings, and training budget.
\State Initialize $\theta$, Adam state, fixed feature statistics, and any rarity archive.
\While{the training budget remains}
  \State Draw shared latent and sampling noise, references, and any conditions.
  \State Generate an unperturbed batch if the task requires one.
  \State Prepare dynamic scales and targets; age the archive when scheduled.
  \State Hold all scoring quantities and archive counts fixed across candidates.
  \State Sample $J$ perturbation pairs $\theta_{j\ell}^{\pm}
    =\theta_\ell\pm\sigma_\ell\epsilon_{j\ell}$.
  \For{each candidate $q=1,\ldots,P$}
    \State Generate a batch of hard samples using the shared inputs.
    \State Compute the task's fitness components on the whole batch.
  \EndFor
  \State Convert the population's scores to utilities $u_q$ using the task's rule.
  \State Compute $\widehat g_\ell$ with Eq.~\eqref{eq:app-es-direction}
    and apply Adam ascent.
  \State If rarity is used, add eligible pre-update unperturbed samples to its archive.
  \State Follow the task's perturbation-scale and checkpoint schedules.
\EndWhile
\State \Return The central-network checkpoint prescribed by the experiment.
\end{algorithmic}
\end{algorithm}

Utilities are formed after the whole population is scored, so any ranks
compare candidates under the same conditions. Fixed training-derived
feature statistics remain unchanged throughout training. The QM9
matched-reward and ZINC250K active-time controls return the latest completed
checkpoint within each budget; other experiments use the selection rules
specified below. At deployment, each proposal uses one forward pass through
the returned network.

\subsection{Molecular Generation}
\label{app:method-molecular}

\noindent\textbf{Summary.}
We use QM9, ZINC250K, and GuacaMol to test whether a compact one-step
generator can produce many distinct valid molecules at low inference cost.
We measure validity, distinct yield, distributional fidelity, and estimated
neural work. The size sweep and budget controls examine how these
trade-offs change with molecular size and available training resources;
the QM9 reward control compares update recipes under the same hard reward.

\noindent\textbf{Takeaway.}
High validity alone can hide repeated outputs. We therefore assess
distinct valid yield together with fidelity and cost, and check novelty
separately against training molecules.

\paragraph{Generator and experimental settings.}
A ReLU MLP receives a 96-dimensional standard Gaussian latent vector and a one-hot
encoding of the requested number of heavy atoms, sampled from the relevant
training size distribution. Two output heads predict
atom states and bonds. An atom state specifies element, charge, and
hydrogen information; bonds have four states: absent, single, double, or
triple. Aromatic molecules use a Kekul\'e representation. Independent
Gumbel-max draws sample the discrete states. We predict each unordered
atom pair once, mirror its bond, and mask unused atoms. These operations
enforce size and symmetry; they do not guarantee chemical validity.
Output biases start from training atom and bond frequencies.

Table~\ref{tab:method-architectures} distinguishes the full-dataset models
from the controlled comparisons. Within each controlled experiment, the
training methods use the same generator architecture.

\begin{table}[tbp]
  \centering
  \caption{Molecular generator configurations. Canvas size is the maximum
  number of atom positions in the output. Hidden layers are shown as
  depth $\times$ width. The ZINC250K active-time control uses the full
  38-atom canvas with a smaller hidden width than the benchmark model.}
  \label{tab:method-architectures}
  \begin{tabular*}{\linewidth}{@{\extracolsep{\fill}}lcc@{}}
    \toprule
    Experiment & Canvas size & Hidden layers \\
    \midrule
    QM9 benchmark and matched-reward control & 9 & $3\times768$ \\
    ZINC250K, primary experiment & 38 & $3\times1024$ \\
    Full GuacaMol & 88 & $3\times768$ \\
    Controlled molecular size sweep & 32 & $3\times768$ \\
    ZINC250K active-time control & 38 & $3\times768$ \\
    \bottomrule
  \end{tabular*}
\end{table}

\paragraph{Molecular size and search difficulty.}
The canvas has $N$ available atom positions, of which $n\leq N$ are active
in a generated molecule. For these active positions, the generator makes
\begin{equation}
  D(n)=n+\binom{n}{2}=\frac{n(n+1)}{2}
  \label{eq:app-active-decisions}
\end{equation}
categorical decisions: one atom state per position and one bond state per
unordered pair. This gives 136, 300, and 528 decisions at 16, 24, and
32 atoms. With $A$ available atom states and four bond states, the space
before chemical constraints contains $A^n4^{\binom{n}{2}}$ labeled graph
assignments. These are graph encodings, not distinct valid molecules.

Chemical validity restricts this space sharply. A fixed upper bound on
atom valence permits only a number of present bonds proportional to $n$,
although the number of possible pairs grows quadratically. Connectivity
and local chemistry impose further constraints. Training must therefore
coordinate many choices to place probability on this much smaller valid
subset, while spreading that probability over different molecules.

The size sweep tests whether validity, Strict VUN, and feature matching
are maintained as larger molecules enter the task. It uses ZINC250K and
GuacaMol with cumulative caps of 16, 24, and 32 heavy atoms, keeping the
canvas fixed at $N=32$. The vocabulary, chemistry tables, size-encoding
width, and initialization also stay fixed. Equation~\eqref{eq:app-active-decisions}
counts active decisions; masked positions do not contribute. Larger caps
also change the molecular mixture and reference distribution, so changes
in performance cannot be attributed to decision count alone.

\paragraph{Distribution matching.}
We compute features directly from every sampled graph, including invalid
ones. Table~\ref{tab:method-molecular-features} lists the seven feature
families and what each contributes. Keeping invalid graphs in the batch
allows them to affect training instead of disappearing from the objective.

\begin{table}[tbp]
  \centering
  \caption{Molecular features and their role in the fitness. Counts and
  histograms are normalized by the active atom count. Local-neighborhood
  histograms use bond-aware Weisfeiler--Lehman (WL) label refinement.}
  \label{tab:method-molecular-features}
  \begin{tabular*}{\linewidth}{@{\extracolsep{\fill}}p{0.22\linewidth}p{0.70\linewidth}@{}}
    \toprule
    Feature family & Information matched to training molecules \\
    \midrule
    Atom states & Composition, including charge and hydrogen states. \\
    Valence & Distribution of bond-order sums at atoms, in bins 0--8. \\
    Bond types & Relative counts of single, double, and triple bonds. \\
    WL1 & Local atom neighborhoods after one refinement round; 64 bins. \\
    WL2 & Larger neighborhoods after two refinement rounds; 64 bins. \\
    Spectrum & Global connectivity, using 10 bins of normalized-Laplacian eigenvalues. \\
    Global counts & Bonds, connected components, and independent cycles per atom. \\
    \bottomrule
  \end{tabular*}
\end{table}

For feature family $r$, let $\phi_r(x)\in\mathbb R^{d_r}$ denote the
raw features of sample $x$, with $d_r$ coordinates. Write
$\widetilde\phi_r$ for the task-specific normalized feature map defined
below. For generated batch $X=(x_1,\ldots,x_B)$ and reference batch
$Y=(y_1,\ldots,y_{B_+})$, set
$f_i^{(r)}=\widetilde\phi_r(x_i)$ and
$h_j^{(r)}=\widetilde\phi_r(y_j)$.
Here $B,B_+\geq2$ are batch sizes; batches retain repeated samples.
We use the empirical Energy discrepancy
\begin{equation}
\begin{aligned}
  \widehat{\mathcal E}_r=
  {}&\frac{2}{BB_+}\sum_{i=1}^{B}\sum_{j=1}^{B_+}
  d\bigl(f_i^{(r)},h_j^{(r)}\bigr)\\
  &-\frac{1}{B(B-1)}\sum_{i\ne i'}d\bigl(f_i^{(r)},f_{i'}^{(r)}\bigr)
  -\frac{1}{B_+(B_+-1)}\sum_{j\ne j'}d\bigl(h_j^{(r)},h_{j'}^{(r)}\bigr),
\end{aligned}
  \label{eq:app-energy}
\end{equation}
where $d(a,b)=\sqrt{\lVert a-b\rVert_2^2+10^{-12}}$ is a smoothed
Euclidean distance between feature vectors. Within-batch sums run over
all ordered distinct pairs: $i,i'\in\{1,\ldots,B\}$ and
$j,j'\in\{1,\ldots,B_+\}$.
The first term brings generated features toward the reference distribution.
The second rewards spread among generated examples. The final term measures
spread within the reference batch and is constant while candidates are
compared. Thus, the objective already includes a diversity signal before
adding rarity.

The molecular benchmark, size sweep, and active-time comparison use
$\widetilde\phi_r(x)=\phi_r(x)/s_r$.

The scale uses an unperturbed batch $X^0=(x_1^0,\ldots,x_{B_0}^0)$,
with $B_0=192$, and $B_+=192$ references:
\begin{equation}
 s_r=\max\left\{
 \frac{\displaystyle\sum_{i=1}^{B_0}\sum_{j=1}^{B_+}
 \sqrt{\max\{\|\phi_r(x_i^0)-\phi_r(y_j)\|_2^2,10^{-12}\}}}
 {B_0B_+\sqrt{d_r}},\;10^{-8}\right\}.
 \label{eq:app-cross-scale}
\end{equation}
This positive scalar and the references remain fixed for all candidates.
Scaling prevents feature units alone from setting their relative influence.
These finite features summarize molecules; matching them does not guarantee
that the full molecular distributions match.

\paragraph{Structural penalties and rarity.}
Distribution matching alone does not ensure valid chemistry. We therefore
maximize the base fitness
\begin{equation}
  F_{\mathrm{mol}}=-\sum_{r=1}^{7}\widehat{\mathcal E}_r
  -2\overline C_{\mathrm{val}}
  -0.5\overline C_{\mathrm{conn}}
  -\overline C_{\mathrm{local}},
  \label{eq:app-molecular-fitness}
\end{equation}
Here $F_{\mathrm{mol}}$ is a scalar batch fitness to maximize, and
$\overline C=B^{-1}\sum_{i=1}^{B}C(x_i)$ averages a per-graph penalty.
For a graph $x$ with $n$ active atoms, let $a_j$ be atom $j$'s state
and $\nu_j$ its sum of incident bond orders. To compute its valence error
$\delta_j$, consider bins $v=0,\ldots,12$: use cost $|\nu_j-v|$ when
$v$ is allowed for $a_j$, and cost 13 otherwise, then take the minimum.
If $K(x)$ is the number of connected components and $L(x)$ the size of
the largest, the first two penalties are
\begin{equation}
 C_{\mathrm{val}}(x)=\frac1n\sum_{j=1}^{n}\delta_j,
 \qquad
 C_{\mathrm{conn}}(x)=\mathbf1\{K(x)>1\}+1-\frac{L(x)}n.
 \label{eq:app-molecular-penalties}
\end{equation}
The indicator $\mathbf1\{\cdot\}$ is one when its condition holds and
zero otherwise. The local-chemistry penalty $C_{\mathrm{local}}(x)$ adds
two quantities divided by $n$: each atom's nearest-pattern $\ell_1$
distance to a training pattern for the same atom state, summed over atoms, and
the number of bonds with unsupported endpoint-state/bond-order combinations.
The distance sums absolute differences in incident single/double/triple-bond counts.
Pattern lookup clips counts to $(6,3,2)$ and adds the clipped-off overflow; an atom state with no observed pattern receives lookup cost 13.
Patterns and supported bond combinations are obtained from training data.
The weights $(2,0.5,1)$ are fixed experimental settings.

When enabled, rarity adds
\begin{equation}
  R_t(X)=\frac{1}{|X|}\sum_{x\in X}
  \frac{g(x)}{\sqrt{1+\widehat v_t(x)}},
  \label{eq:app-rarity}
\end{equation}
Here $X$ is the candidate's generated batch, $|X|=B$ counts occurrences
including repeats, and $g(x)$ is a binary eligibility indicator: it is one
when the graph is connected, every $\delta_j=0$, and
$C_{\mathrm{local}}(x)\leq10^{-6}$, and zero otherwise. The approximate visit count
$\widehat v_t(x)$ is the smaller of two hash-bucket counts recorded before
update $t$. This reward favors
eligible graphs that have been generated less often. It is neither an
exact identity test nor a reward for novelty relative to the training set.
The archive stays fixed during candidate scoring. After the update, it receives eligible occurrences from the pre-update unperturbed batch.
Before every 1,000th update, its integer counts are halved, rounding down.

In these experiments, each candidate is scored on the same $B=96$ inputs.
For a population score vector $\zeta=(\zeta_1,\ldots,\zeta_P)$, define
its centered rank (CR) as
$\operatorname{CR}(\zeta)_q=\operatorname{rank}_{\rm avg}(\zeta_q)/(P-1)-1/2$,
where ranks are ascending, start at zero, and are averaged across ties.
Let $F_q=F_{\mathrm{mol}}(X_q,Y)$ and $R_q=R_t(X_q)$ be candidate
$q$'s base fitness and rarity score; $F$ and $R$ collect all $P$ scores.
Its utility is $u_q=\operatorname{CR}(F)_q+
\lambda_R\operatorname{CR}(R)_q$.
We use $\lambda_R=0.2$, or zero in the rarity
ablation. Ranking the terms separately controls their influence without
depending on their raw numerical ranges. Adam uses learning rate $10^{-2}$;
weight perturbations have scale $\max\{0.1\operatorname{RMS}(W),10^{-3}\}$
and bias perturbations have scale $0.01$.

\paragraph{STE and Soft comparisons.}
The size-sweep and active-time comparisons share the generator and initialization within
each seed. Soft trains on relaxed categorical outputs. STE keeps the
numerical hard forward loss but obtains derivatives through differentiable
feature and constraint approximations. Both use shared Gumbel noise and
relaxation temperature 1. Their proxy features describe composition,
valence, degree, element/bond combinations, and topology; they replace the
hard hashes and spectral features that cannot be differentiated directly.
Proxy coordinates use fixed training-derived scales, whereas \methodname in these
comparisons uses the dynamic family scales above.
Neither baseline includes hashed rarity. Both use Adam with learning rate
$3\times10^{-4}$ and gradient-norm clipping at 10.

\paragraph{Matched-reward comparison on QM9.}
The separate QM9 control uses five paired seeds (42--46), the same
generator, data split, and initialization, and the same hard reward
$R_B=F_{\mathrm{mol}}$ for \methodname and REINFORCE-LOO (RLOO).
Both methods use feature scales computed once from 8,192 training pairs:
the mean pairwise Euclidean distance in each family, divided by
$\sqrt{d_r}$ and floored at $10^{-8}$. These replace the dynamic $s_r$
in Eq.~\eqref{eq:app-cross-scale} and remain fixed.
Neither method uses rarity or rank shaping. This keeps the
reward construction identical while comparing the two update recipes.

RLOO uses 96 generated graphs and 192 training references per update.
It assigns sample $i$ the difference
\begin{equation}
  A_i=R_B(X_{1:B})-R_{B-1}(X_{-i}),\qquad
  \mathcal L_{\mathrm{RF}}=-\sum_{i=1}^{B}\operatorname{stopgrad}(A_i)
  \log p_\theta(X_i\mid z_i,n_i).
  \label{eq:app-reinforce-loo-loss}
\end{equation}
Here $X_{1:B}$ is the complete policy batch and $X_{-i}$ is the same batch
with graph $X_i$ removed. $R_{B-1}$ evaluates the reward using batch size
$B-1$, and $A_i$ is the resulting leave-one-out advantage.
$p_\theta(X_i\mid z_i,n_i)$ is the probability of the sampled atom and
bond decisions given latent vector $z_i$ and requested atom count $n_i$.
Its log sums only active decisions. The operator $\operatorname{stopgrad}$
keeps its argument's value but treats it as constant during differentiation;
$\mathcal L_{\mathrm{RF}}$ is minimized. The sum is intentional because
$A_i$ already includes the batch normalization.
Removing a graph changes its distances to the other generated graphs, so
the complete batch reward is recomputed with the correct smaller-batch
normalization. References and feature scales remain fixed during this
calculation. \methodname uses the same batch sizes per candidate and raw paired
reward differences across its 192 candidates. Both use Adam, with learning
rates $10^{-2}$ for \methodname and $3\times10^{-4}$ for RLOO.

At 11,520,000 hard-reward graph evaluations per method and seed, \methodname
completes 625 updates, and RLOO completes 120,000. We evaluate the current
parameters at the declared budget boundaries, without validation-based
checkpoint selection. This matches the number of generated graphs scored
by the reward, rather than elapsed time or total computation. The update
recipes still differ in optimizer settings and reference reuse.
Appendix~\ref{app:qm9-reinforce} gives the protocol and results.

\paragraph{Budgets and checkpoint selection.}
Table~\ref{tab:method-protocols} gives the budgets used for the molecular
benchmarks and controlled comparisons. The primary ZINC250K benchmark
retains its original checkpoint rule and its vocabulary and
allowed-valence table collected before splitting. These differ from the
controlled protocols; details are in Appendix~\ref{app:rq1}.

\begin{table}[tbp]
  \centering
  \caption{Budgets for the molecular comparisons. Each evaluated model produces
  10,000 raw proposals for the primary yield metrics. Three-seed
  experiments use seeds 42, 43, and 44; the QM9 reward control uses
  seeds 42--46. Its row reports the largest matched reward budget.}
  \label{tab:method-protocols}
  \begin{tabular*}{\linewidth}{@{\extracolsep{\fill}}lcl@{}}
    \toprule
    Experiment & Seeds & Reported training budget \\
    \midrule
    QM9 benchmark & 3 & 120,000 updates \\
    Full GuacaMol & 3 & 240,000 updates \\
    Molecular size sweep & 5 & 30,000 updates \\
    ZINC250K active-time control & 3 & 30 active minutes \\
    QM9 matched-reward control & 5 & 11.52M graph evaluations \\
    \bottomrule
  \end{tabular*}
\end{table}

The QM9 and full-GuacaMol benchmark models select checkpoints using 2,000
fixed-noise validation proposals. The score is strict valid-and-unique
yield divided by $1+8M_w$, where
$M_w=(2\max\{M_{\mathrm{atom}},0\}+\max\{M_{\mathrm{val}},0\}
+\max\{M_{\mathrm{spec}},0\})/4$.
Here $M_{\mathrm{atom}}$, $M_{\mathrm{val}}$, and $M_{\mathrm{spec}}$
are estimates of squared maximum mean discrepancy (MMD$^2$) for atom-state,
valence, and spectral features;
$M_w$ is their weighted, zero-clipped average. The yield in the selection
score is a fraction between zero and one.
This balances usable output yield with distribution matching.
The size sweep instead prioritizes strict validity, then strict
valid-and-unique yield, then lower $M_w$, using 10,000 fixed-noise proposals
and a training reference panel. The ZINC250K active-time control uses
a separate full-ZINC cache with a training-derived vocabulary and size
prior, shared by all four methods. Methods run sequentially on the same
allocated GPU, with their order rotated across seeds. Each reported point
uses the latest completed update within its budget. Active time includes synchronized
batch preparation, transfers, and optimizer/archive updates; it excludes
setup, compilation, evaluation, and checkpoint I/O. These selection rules
do not use training-set novelty or final test results.

\paragraph{Evaluation.}
Full validity requires RDKit sanitization of the complete generated graph;
strict validity also requires connectedness. Full VU counts distinct
full-valid molecules, while Strict VUN counts distinct, connected, valid
molecules absent from the declared training reference. Both yield counts are divided by all attempted proposals, including invalid outputs and duplicates.
Conditional uniqueness instead divides by valid proposals. Novelty uses
full source-training identities for the benchmark models and QM9 reward control, and
prepared training identities for the size-sweep and time controls.
Identities use canonical non-isomeric SMILES. Disconnected molecules are not replaced by
their largest component in these full-graph metrics.

Distributional evaluation uses feature MMD$^2$ and Fr\'echet ChemNet
Distance (FCD), with lower values preferred. MMD$^2$ excludes diagonal
within-sample terms and can be negative at finite sample sizes. Its
reference-derived bandwidth is fixed within each comparison. The size-sweep
metric averages atom, valence, and spectral MMD$^2$ equally, whereas
$M_w$ above gives atom features twice the weight.
FCD retains repeated valid molecules. The QM9 benchmark, QM9 reward
control, and full GuacaMol use a separate
30,000-attempt bank to obtain up to 10,000 valid molecules for FCD; the
historical ZINC evaluation uses valid outputs from its 10,000-attempt bank.
These extra FCD samples do not change yield denominators. External datasets,
splits, and FCD score conversions are documented in Appendix~\ref{app:rq1}.

\paragraph{External comparisons and computational cost.}
External molecular baselines retain their source training and evaluation
protocols. We estimate dense inference work by counting two operations
per dense multiply-accumulate across all neural evaluations needed for
one proposal. With VU yield $Y_{\mathrm{VU}}$ expressed as a fraction and
estimated operation count per proposal $C_{\mathrm{dense}}$, the reported efficiency is
$E_{\mathrm{VU}}=10^{12}Y_{\mathrm{VU}}/C_{\mathrm{dense}}$.
It measures distinct valid output per estimated neural work; it excludes
feature extraction, graph processing, and chemistry, and is not throughput.
The GuacaMol CFM pilot is a three-epoch feasibility checkpoint, excluded
from ratio and ranking claims. Its projected full-training time is
separate from measured training times. Appendix~\ref{app:external-guacamol} documents that pilot and the external training-time comparison.

\subsection{Scaffold Completion}
\label{app:method-scaffold}

\noindent\textbf{Summary.}
We test whether one-step generation can turn a supplied molecular core
into a useful library for downstream screening. At equal proposal budgets,
we compare with ZINC-trained MoLeR on distinct filter-passing completions,
coverage across cores, and generation-to-SMILES time. A separate
matched-time experiment measures how many candidates can be generated and
screened within a short deadline.

\noindent\textbf{Takeaway.}
A generator can return more passing molecules while serving fewer
scaffolds. Candidate count, query coverage, and generation time therefore
measure different aspects of the resulting library.

\paragraph{Conditional generator and decoding.}
The model has a 38-atom canvas, three 768-unit global layers, and two
relational blocks with 128-dimensional node states. It receives the core,
the requested total size, and global and per-node Gaussian noise. Core
features describe atom states, observed bonds, and available valence, giving
the network information about where additions may fit.
The decoder copies all core atoms and core-to-core pair states. Each added
atom chooses an earlier atom as a parent, with a bond of at least order one.
For a connected core, these parent edges connect every addition to the core;
optional bonds allow cycles. All scores come from one network evaluation.
The decoder does not repair invalid valences.

\paragraph{Conditional distribution fitness.}
The scaffold model uses Drift targets. Its nine raw feature families
$\phi_r$, indexed by $r=1,\ldots,9$, include the seven in
Table~\ref{tab:method-molecular-features}, plus an atom-state/valence
histogram with $13A$ coordinates and a 256-bin hashed histogram of
bonded atom-state pairs and bond orders. Here $A$ is the scaffold atom-state
vocabulary size; valences are clipped to bins 0--12. Both added histograms
are normalized by active atom count. They help distinguish chemically
different uses of the same elements.
For completed graph $x$ and core graph $c$, each joint family is
$\psi_r(x,c)=[\phi_r(x),\lambda_s\phi_r(c)]\in\mathbb R^{D_r}$,
where brackets denote concatenation, $D_r=2d_r$, and $\lambda_s$ weights
the core features relative to the completion features.
Including the core makes the comparison sensitive to which input a
completion belongs to.

\paragraph{Drift targets.}
At each update, the unperturbed and reference batches contain
$B_0=B_+=384$ completion/core pairs. We compute $s_r$ using
Eq.~\eqref{eq:app-cross-scale}, replacing $\phi_r$ by $\psi_r$ and
$d_r$ by $D_r$. For one family, let $G$ and $H$ be the normalized
unperturbed and reference feature matrices: their rows are joint features
divided by $s_r$. They have shapes $B_0\times D_r$ and $B_+\times D_r$.
The family index is omitted in the following construction.

For matrices $U,V$ with the same feature width, define the pairwise
distance matrix by
$\mathsf d(U,V)_{ij}=\sqrt{\max\{\|U_i-V_j\|_2^2,10^{-12}\}}$,
where $U_i,V_j$ are rows. For each temperature
$\tau\in\sqrt{D_r}\{0.02,0.05,0.2\}$, form
\begin{equation}
\begin{aligned}
 L_\tau&=\left[-\mathsf d(G,H)/\tau,\;
              -(\mathsf d(G,G)+10^9I_{B_0})/\tau\right],\\
 Q_\tau&=\sqrt{\operatorname{softmax}_{\rm row}(L_\tau)
              \odot\operatorname{softmax}_{\rm col}(L_\tau)}
          =[Q_\tau^+,Q_\tau^-].
\end{aligned}
\label{eq:app-scaffold-affinities}
\end{equation}
Here $I_{B_0}$ is the identity matrix, so the large diagonal term removes
self-interactions. Row softmax normalizes across columns; column softmax
normalizes across rows. The product $\odot$ and square root act entrywise.
The first $B_+$ columns, $Q_\tau^+$, weight reference examples; the
remaining $B_0$ columns, $Q_\tau^-$, weight unperturbed examples.

Let $m_{\tau,i}^\pm=\sum_jQ_{\tau,ij}^\pm$ denote their row sums.
The attraction--repulsion field and target matrix are
\begin{equation}
\begin{aligned}
 V_{\tau,i}&=m_{\tau,i}^-\sum_{j=1}^{B_+}Q_{\tau,ij}^+H_j
           -m_{\tau,i}^+\sum_{j=1}^{B_0}Q_{\tau,ij}^-G_j,\\
 T_r&=G+\sum_\tau
       \frac{V_\tau}{\max\{\operatorname{RMS}(V_\tau),10^{-8}\}}.
\end{aligned}
\label{eq:app-scaffold-target}
\end{equation}
$V_\tau$ has $B_0$ rows of width $D_r$; RMS averages over all its
entries. The three normalized fields are summed. Targets and scales stay
fixed while candidates are scored. Candidate $q$ receives
\begin{equation}
  O_q=-\frac{1}{B}\sum_{i=1}^{B}\sum_{r=1}^{9}\frac{1}{D_r}
  \left\|\frac{\psi_r(x_{qi},c_i)}{s_r}-T_{r,i}\right\|_2^2,
  \qquad B=96,
  \label{eq:app-scaffold-drift-fitness}
\end{equation}
where $x_{qi}$ is its completion of core $c_i$, and $T_{r,i}$ is row $i$
of the target matrix. Only the first $B=96$ shared conditions are used.
$O_q$ is a scalar score to maximize: it rewards movement toward the
fixed targets, with a dimension correction for each family.

\paragraph{Chemical constraints and diversity.}
For each completed graph, use the per-atom valence errors $\delta_j$
and the local penalty $C_{\mathrm{local}}$ defined above. The scaffold
risks are
\begin{equation}
\begin{aligned}
 C_{\mathrm{val}}^{\mathrm{scaf}}
 &=\frac1n\sum_{j=1}^{n}\delta_j
   +\alpha_{\max}\max_j\delta_j
   +\alpha_{\mathrm{fail}}\mathbf1\{\max_j\delta_j>10^{-6}\},\\
 C_{\mathrm{local}}^{\mathrm{scaf}}
 &=C_{\mathrm{local}}+\alpha_{\mathrm{local}}
      \mathbf1\{C_{\mathrm{local}}>10^{-6}\}.
\end{aligned}
\label{eq:app-scaffold-risks}
\end{equation}
The nonnegative coefficients $\alpha_{\max}$, $\alpha_{\mathrm{fail}}$,
and $\alpha_{\mathrm{local}}$ weight worst-atom error, any valence
failure, and any local-chemistry failure, respectively. These terms stop
one failing part of a molecule from being hidden by an average.
With $n_{\mathrm{add}}$ added atoms and $n_{\mathrm{det}}$ of them
having no path to the core, attachment risk is
$C_{\mathrm{att}}=\mathbf1\{n_{\mathrm{det}}>0\}
+n_{\mathrm{det}}/\max\{n_{\mathrm{add}},1\}$.
The candidate penalty is
\begin{equation}
  \Pi_q=\overline C_{\mathrm{val},q}^{\mathrm{scaf}}
  +1.25\overline C_{\mathrm{att},q}
  +0.25\overline C_{\mathrm{local},q}^{\mathrm{scaf}}.
  \label{eq:app-scaffold-penalty}
\end{equation}
Here each bar averages over candidate $q$'s $B=96$ completions.
Attachment risk is zero for connected supported cores under the forest decoder.

Conditional unique yield $U_q$ sums the number of distinct eligible graph
signatures in each repeated core/size group, then divides by the 64
conditional attempts, including ineligible ones. The same signature in
two groups counts once per group. Rarity $R_q=R_t(X_q)$ uses all 96 candidate
proposals and Eq.~\eqref{eq:app-rarity}, with the scaffold gate
$g_{\mathrm{scaf}}(x)=\mathbf1\{\delta_j=0\text{ for all active }j\}
\mathbf1\{K(x)=1\}$ in place of $g$.
Thus eligibility requires allowed valences and connectedness.
Conditional uniqueness discourages repeated answers to
one query, while rarity discourages repeated graph signatures across
updates. Let $\bar g_t^0=B_0^{-1}\sum_{i=1}^{B_0}g_{\mathrm{scaf}}(x_i^0)$
be the eligible fraction in the 384 unperturbed completions. We use
$\gamma_t=\operatorname{clip}((\bar g_t^0-a)/(b-a),0,1)$,
where $0\leq a<b\leq1$ are the eligibility thresholds at which diversity
pressure starts and reaches full weight. The clip operation restricts its
argument to $[0,1]$.

We convert each score into a centered rank (CR) across the $P=192$
candidates. For a score vector $\zeta=(\zeta_1,\ldots,\zeta_P)$,
\begin{equation}
  \operatorname{CR}(\zeta)_q
  =\frac{\operatorname{rank}_{\mathrm{avg}}(\zeta_q)}{P-1}
  -\frac{1}{2},
  \label{eq:app-scaffold-centered-rank}
\end{equation}
where ranks run from zero for the lowest score to $P-1$ for the highest,
and tied scores receive their average rank. Thus, CR maps scores to
$[-0.5,0.5]$ and removes differences in their numerical scales.
We rank $-\Pi$ so that lower penalties receive higher utility.
The utility combines these separately ranked components:
\begin{equation}
\begin{aligned}
  u_q={}&w_O\operatorname{CR}(O)_q
  +w_\Pi\operatorname{CR}(-\Pi)_q
  +w_v\operatorname{CR}(v)_q+w_c\operatorname{CR}(c)_q\\
  &+\gamma_t\left[w_U\operatorname{CR}(U)_q
  +0.2\operatorname{CR}(R)_q\right],
\end{aligned}
  \label{eq:app-scaffold-shaped-utility}
\end{equation}
Here $v_q$ and $c_q$ are the fractions of candidate $q$'s 96 completions
with allowed valences and connected graphs, respectively.
The vectors $O,\Pi,v,c,U,R$ collect the corresponding scores over all
$P$ candidates. Nonnegative weights $w_O,w_\Pi,w_v,w_c,w_U$ multiply
the ranks of Drift fit, negative penalty, valence success, connectivity
success, and conditional unique yield, in that order.
Separate ranks retain a role for each criterion despite their different
numerical ranges. The logged value $O_q-\Pi_q$ is only a diagnostic.
Table~\ref{tab:method-scaffold-settings} defines the reference values of
these coefficients.

\begin{table}[tbp]
  \centering
  \caption{Scaffold coefficients}
  \label{tab:method-scaffold-settings}
  \begin{tabular*}{\linewidth}{@{\extracolsep{\fill}}ll@{}}
    \toprule
    Coefficient & Value used in the experiments \\
    \midrule
    Core feature weight $\lambda_s$ & $1$ \\
    Risk weights $(\alpha_{\max},\alpha_{\mathrm{fail}},\alpha_{\mathrm{local}})$
      & $(1,0.75,0.5)$ \\
    Diversity thresholds $(a,b)$ & $(0.35,0.75)$ \\
    Rank weights $(w_O,w_\Pi,w_v,w_c,w_U)$ & $(1,0.5,0.25,0.75,0.2)$ \\
    \bottomrule
  \end{tabular*}
\end{table}

\paragraph{Training and evaluation.}
The unperturbed batch has 384 examples. Each candidate uses the same first
96 conditions, including 64 examples in repeated core/size groups.
Core identities are separated across training and held-out partitions.
At inference, we sample the number of added atoms from training molecules
with the same scaffold size, falling back to pooled training counts when
that size is unseen. Only positive additions that fit the supported sizes
and 38-atom limit are retained, and the distribution is renormalized.
The public input is therefore only the scaffold, without the held-out
target molecule or its total size.

Evaluation uses 100 held-out queries and 1,000 attempts per query.
A usable completion must be chemically valid, connected, and contain the
supplied core as a substructure. We deduplicate canonical non-isomeric SMILES within each query;
failed attempts remain in the yield denominator. Screening then requires
QED (quantitative estimate of drug-likeness) $\geq0.7$, SA
(synthetic accessibility score) $\leq4$, and no new PAINS
(pan-assay interference compounds) alert types relative to the core.
These criteria measure the resulting candidate library and never
enter the training fitness.
Appendix~\ref{app:scaffold-protocol} gives the size prior, checkpoint
selection, and available training budgets; Appendix~\ref{app:scaffold}
reports the screening protocol and results. The comparison uses three
\methodname training seeds (42--44) and three sampling seeds (41--43) of one
ZINC-trained MoLeR checkpoint, selected from training seed 42.
These represent different sources of variation. The library-generation
timer includes conversion to SMILES and excludes subsequent screening.

\paragraph{Matched-time example.}
The example in Figure~\ref{fig:scaffold_example} shows the highest-QED
filter-passing candidate from each method. Its deadline is the time MoLeR
needs to generate and screen one batch of 256 completions, with both
models already loaded and warmed up. \methodname generates batches of 256 graphs
and scores them directly; only candidates screened before the deadline
count. The selected \methodname molecule is converted to SMILES afterwards. This timer includes screening,
so it measures a different operation from generation to SMILES.

\subsection{Neural Architecture Generation}
\label{app:method-nas}

\noindent\textbf{Summary.}
NAS-Bench-101 tests whether the approach is useful beyond chemistry.
We train on high-performing architectures and compare \methodname, STE, and Soft
by the number of distinct held-out elite architectures obtained within
10,000 proposals. Validity, mean benchmark accuracy, and feature MMD
measure the trade-offs between a larger candidate library, feasibility,
and architecture quality.

\noindent\textbf{Takeaway.}
Producing more distinct elite architectures can be useful even when mean
accuracy is lower. We report both to distinguish the number of qualifying
alternatives from their average quality.

\paragraph{Data and generator.}
Canonical benchmark architectures are split 80/10/10 with seed 2026.
The elite threshold is the 90th percentile of validation accuracy among
training architectures and is then fixed for the held-out partitions.
A three-layer, 512-unit SiLU MLP maps 128-dimensional Gaussian noise to
21 edge logits and three operation logits for each of five internal
vertices. Input and output vertices are fixed. Logits are clipped to
$[-10,10]$; edges use Bernoulli sampling, and operations use Gumbel-max.
Output biases use training-elite frequencies.

Vertices outside input-to-output paths are pruned. A proposal is invalid
if it has no such path or has more than nine retained edges. The benchmark's
canonical hash identifies architectures after pruning, avoiding duplicate
counts caused by equivalent graph encodings.

\paragraph{Fitness and its purpose.}
Three feature families describe each pruned graph: its edges and operations
(36 coordinates), in/out degrees (14), and size, reachability, and path
counts (9). Together they capture local choices and how those choices
connect input to output. For coordinate $k\in\{1,\ldots,d_r\}$ of
raw family $\phi_r$, the normalized feature is
\begin{equation}
 \widetilde\phi_{r,k}(x)=
 \frac{\phi_{r,k}(x)-\mu_{r,k}^{\mathrm{train}}}
 {\max\{s_{r,k}^{\mathrm{train}},0.05\}},
 \label{eq:app-fixed-normalization}
\end{equation}
where $\mu_{r,k}^{\mathrm{train}}$ and $s_{r,k}^{\mathrm{train}}$ are
the mean and population standard deviation of that coordinate over
training elites. The statistics are fixed and shared by all methods and
by generated and reference examples.
Using the Energy discrepancy in Eq.~\eqref{eq:app-energy}, the fitness is
\begin{equation}
  F_{\mathrm{NAS}}=-\frac13\sum_{r=1}^{3}\widehat{\mathcal E}_r
  -\frac{0.5}{B}\sum_{i=1}^{B}
  \left[1-h_i+\frac{\max\{0,e_i-9\}}9\right],
  \label{eq:app-nas-fitness}
\end{equation}
Here $B=64$, $B_+=192$ in each Energy term, $h_i\in\{0,1\}$ is one
exactly when proposal $i$ has an input-to-output path, and $e_i$ counts its
retained edges. The fitness is maximized.
The Energy term matches elite graph structure; the penalty addresses
the benchmark's feasibility rules directly. Invalid proposals remain in
the batch. We use neither a rarity archive nor an explicit duplicate
penalty, so diversity pressure comes from Energy's generated-to-generated
term.

\paragraph{Training, selection, and evaluation.}
All methods use 64 generated examples and 192 training-elite references
per objective evaluation. If $F_q$ is candidate $q$'s NAS batch fitness,
\methodname uses
\begin{equation}
 u_q=F_q-\frac1P\sum_{p=1}^{P}F_p,\qquad P=192.
 \label{eq:app-centered-fitness}
\end{equation}
Centering cancels in each paired difference, leaving raw fitness differences
without ranking. Adam's learning rate is $3\times10^{-3}$.
For a weight block $W$, its perturbation scale is
$\omega_t\max\{0.08\operatorname{RMS}(W),0.001\}$; bias scales are
$0.01\omega_t$. Both output-layer scales receive an additional factor
$0.5$. The global multiplier starts at $\omega_0=1$.
At update 1 and every 100th update, we measure the fractions of edge and
operation decisions changed by perturbations relative to the unperturbed
network under shared noise. If the larger fraction exceeds $0.14$, the
multiplier is multiplied by $0.9$, with a lower bound of $0.2$; below
$0.04$, it is multiplied by $1.1$, with an upper bound of $2$.
Otherwise, it is unchanged.
STE and Soft use temperature $0.75$ and an Adam learning rate of
$3\times10^{-4}$. STE preserves the hard forward loss while differentiating
through relaxed features and pruning; Soft evaluates the relaxed loss.
The architecture, initialization, sampling streams, and hard evaluation
sampler are shared within each seed.

We train for 15,000 updates with seeds 42, 43, and 44. Every 500 updates,
2,000 validation proposals are scored by
$S_{\mathrm{NAS}}=0.4V+0.2U+0.4/(1+\max\{\mathcal M_{\mathrm{NAS}},0\})$.
The score $S_{\mathrm{NAS}}$ is maximized. $V$ is the valid fraction of
all proposals; $U$ is the number of distinct valid architectures divided
by the valid count, or zero if none are valid.
$\mathcal M_{\mathrm{NAS}}$ is feature MMD$^2$ between up to 1,500
valid proposals and 1,500 validation elites, using raw concatenated
features. If MMD cannot be computed, including with fewer than two valid
proposals, its entire fidelity term contributes zero.
This selects for feasibility, variety, and similarity
to held-out elite structure.

Final evaluation uses 10,000 attempts. Held-out elite yield counts distinct
generated test-partition elites divided by all attempts. Mean benchmark
test accuracy instead includes repeated valid architectures.
Evaluation MMD$^2$ uses the raw concatenated graph features and a bandwidth
fixed from training elites; these differ from the standardized training
features. Appendix~\ref{app:nas} gives the remaining benchmark diagnostics.

\subsection{Binarized MNIST}
\label{app:method-mnist}

\noindent\textbf{Summary.}
Binarized MNIST tests the training recipes in a setting without graph
feasibility constraints: every sampled output is already a binary array.
We compare \methodname, STE, and Soft using image-distribution distances,
classifier-predicted digit coverage, and exact uniqueness. The aim is to
assess image quality and diversity when training relies only on
distribution matching, without labels or classifier feedback.

\noindent\textbf{Takeaway.}
Valid binary arrays need not be recognizable or varied digits. Exact
pixel uniqueness must therefore be read alongside image quality and
digit coverage.

\paragraph{Data and generator.}
Images are binarized at a pixel-intensity threshold of 128 on the 0--255
scale. We hold out 500 examples per digit from the official training set using
seed 2026, leaving 55,000 training and 5,000 validation images. The official
10,000-image test set is used after checkpoint selection.
A three-layer, 768-unit ReLU MLP maps 96-dimensional Gaussian noise to
784 logits. Independent Bernoulli draws produce all pixels of a
$28\times28$ image in parallel. Output biases start from training pixel
frequencies, and logits are not clipped.

\paragraph{Fitness and its purpose.}
Table~\ref{tab:method-mnist-features} lists the six feature families.
They describe image structure at several resolutions, so the objective
can compare both individual pixels and larger shapes. We apply
Eq.~\eqref{eq:app-fixed-normalization} to image features, now estimating
$\mu_{r,k}^{\mathrm{train}}$ and $s_{r,k}^{\mathrm{train}}$ over all
55,000 model-training images. These statistics remain fixed.

\begin{table}[tbp]
  \centering
  \caption{MNIST features used to build the fitness. All are computed
  directly from generated images, without digit labels.}
  \label{tab:method-mnist-features}
  \begin{tabular*}{\linewidth}{@{\extracolsep{\fill}}lrp{0.49\linewidth}@{}}
    \toprule
    Feature family & Dimension & Purpose \\
    \midrule
    Pixels & 784 & Fine image detail. \\
    $2\times2$ pooled means & 196 & Local stroke patterns. \\
    $4\times4$ pooled means & 49 & Coarse digit shape. \\
    Row/column means & 56 & Horizontal and vertical ink placement. \\
    Neighbor differences & 56 & Stroke transitions along rows and columns. \\
    Global statistics & 8 & Ink amount, position, spread, and mirror asymmetry. \\
    \bottomrule
  \end{tabular*}
\end{table}

The fitness is
\begin{equation}
  F_{\mathrm{MNIST}}=-\frac16\sum_{r=1}^{6}
  \frac{\widehat{\mathcal E}_r}{\sqrt{d_r}},
  \qquad(d_1,\ldots,d_6)=(784,196,49,56,56,8).
  \label{eq:app-mnist-fitness}
\end{equation}
Here $d_r$ is the number of coordinates in family $r$, and each Energy
term uses $B=96$ generated images and $B_+=192$ references. The fitness
is maximized. The dimension factor limits the influence of larger feature families
solely due to their coordinate count. There are no additional constraints,
rarity rewards, or duplicate penalties. As in NAS, the Energy term itself
rewards spread among generated examples.

\paragraph{Training and checkpoint selection.}
Each objective evaluation uses 96 generated images and 192 training
references. \methodname uses Eq.~\eqref{eq:app-centered-fitness} with
$F_q=F_{\mathrm{MNIST},q}$ and Adam learning rate $10^{-2}$.
Its weight perturbation scale is
$\max\{0.1\operatorname{RMS}(W),0.001\}$ and its bias scale is $0.01$.
STE and Soft use temperature $0.75$ and an Adam learning rate of
$3\times10^{-4}$. They share the generator, initialization, noise streams,
and fixed feature normalizers with \methodname. STE preserves the hard forward
loss; Soft trains on relaxed pixels. All methods sample binary images
at evaluation.

We train for 30,000 updates with seeds 42, 43, and 44. Every 500 updates,
2,000 validation proposals are scored by
$S_{\mathrm{MNIST}}=U_{\mathrm{exact}}e^{-\Delta/2}$, which is maximized.
Here $U_{\mathrm{exact}}$ is the number of distinct complete binary images
divided by 2,000. For each raw feature coordinate $(r,k)$, let
$\mu^g_{r,k},v^g_{r,k}$ be its mean and population variance over these
generated images, and $\mu^v_{r,k},v^v_{r,k}$ the corresponding statistics
over all 5,000 validation images. Superscripts $g$ and $v$ denote generated
and validation data. We compute
\begin{equation}
 \Delta=\frac16\sum_{r=1}^{6}\frac1{d_r}\sum_{k=1}^{d_r}
 \frac{(\mu^g_{r,k}-\mu^v_{r,k})^2+
 \tfrac14\left(\sqrt{v^g_{r,k}+10^{-6}}-
               \sqrt{v^v_{r,k}+10^{-6}}\right)^2}
 {v^v_{r,k}+10^{-3}}.
 \label{eq:app-mnist-selection-distance}
\end{equation}
The $1/4$ coefficient reduces the weight of standard-deviation mismatch;
the small constants stabilize square roots and denominators. This selector
uses raw feature statistics and its own normalization, distinct from the
training normalization. It combines variety with feature agreement and
uses neither digit labels nor a classifier.

\paragraph{Evaluation.}
A frozen LeNet-style classifier, trained only on the 55,000 training
images and their labels, supplies an independent evaluation of digit
structure. We report Fr\'echet distance in its 128-dimensional embedding,
kernel inception distance (KID), here also computed on LeNet embeddings,
and MMD$^2$ with a radial basis function (RBF) kernel, together with predicted digit frequencies and
confidence. Distances use 2,000 generated and 2,000 test images;
confidence, class counts, and exact identities use all 10,000 generated
images. A digit is covered if at least 50 images receive that argmax
prediction, without confidence filtering. Confidence is reported
separately using a maximum class probability threshold of $0.8$.
Exact novelty compares complete binary images with the 55,000 training
images. Appendix~\ref{app:mnist} gives the evaluator and distance settings.

\paragraph{Reporting across experiments.}
Unless stated otherwise, $\pm$ denotes the sample standard deviation
across runs. Rates are computed separately for each run before averaging.
Timing reports identify which stages are included, since neural sampling,
graph processing, chemical conversion, and screening have different costs.
Estimated dense-operation counts measure only the stated neural work; their ratios do not measure runtime speedups.

\section{RQ1: Molecular Generation}
\label{app:rq1}

This section gives the evidence behind the molecular results in the main text.
We first define the shared evaluation protocol, then present QM9, ZINC250K, and GuacaMol separately.
Each dataset section includes the relevant external comparison and controlled experiments.
The network size and training budget ablations appear in Sections~\ref{app:qm9-sweep} and~\ref{app:zinc-sweep}.
They vary the depth and width of \methodname, while the molecular size sweep varies the data mixture with a fixed network architecture.
We conclude with the inference cost calculations and training time audit used in the main figures.
Throughout this section, \methodname denotes our method, and \methodnamez denotes the version without the history-dependent rarity reward.

\FloatBarrier
\subsection{Shared Evaluation Protocol}
\label{app:external-molecular}

\paragraph{Metrics.}
Full validity requires chemical sanitization of the complete generated molecule.
Full valid-and-unique yield (VU) is the number of distinct valid molecular identities divided by all raw generation attempts.
VUN additionally excludes molecules present in the training reference.
Strict validity requires both chemical validity and connectivity, and Strict VUN applies both requirements before counting distinct molecules absent from training.
These yields differ from uniqueness conditioned on validity.
Invalid and repeated proposals remain in the denominator of every yield metric.
\methodname uses complete molecular identities, including disconnected molecules when reporting full validity and full VU.

\paragraph{External comparisons.}
Tables~\ref{tab:qm9-generators}, \ref{tab:zinc-generators}, and~\ref{tab:guacamol-generators} provide the complete external comparisons supporting the molecular results in the main text.
They include the available sampling configurations and the source metrics needed to interpret these comparisons.
Published values retain the original preprocessing, validity checks, identity rules, correction procedures, and sampling budgets.
An approximate VU value reconstructed from published validity and conditional uniqueness is a source-specific estimate, not a measurement from a common proposal bank.
We report FCD as a distance, with lower values preferred.
For GuacaMol references that report the benchmark FCD score, we give its approximate distance equivalent and state the conversion.
Unless stated otherwise, $\pm$ denotes the sample standard deviation across training seeds.

\paragraph{Inference efficiency.}
$E_{\mathrm{VU}}$ measures VU per $10^{12}$ estimated dense operations at the evaluated proposal budget.
The estimate includes every network evaluation required for generation, but excludes other operations and molecular processing.
It is distinct from measured VU per second.
Section~\ref{app:external-accounting} gives the formula, counting convention, architectural assumptions, and limits of this comparison.

\FloatBarrier
\subsubsection{Design of the Controlled Molecular Size Sweep}
\label{app:ksweep}

\paragraph{Why this experiment.}
Larger molecular graphs require more categorical choices and can make valid, diverse generation harder.
We therefore compare \methodname, \methodnamez, STE, and Soft across cumulative heavy-atom caps $k\in\{16,24,32\}$ on ZINC250K and GuacaMol.
The experiment asks whether the yield advantage persists across molecular size mixtures and whether it depends on the rarity reward.
The full results are reported in Sections~\ref{app:zinc-size-sweep} and~\ref{app:guacamol-size-sweep}.

\paragraph{Shared setup.}
Each task contains molecules with $n\leq k$, and generation is conditioned on a size sampled from the corresponding training mixture.
All settings use a fixed 32-node canvas, a shared atom-state vocabulary derived from training, 96 latent dimensions, and three hidden layers of width 768.
The conditioning width and output heads remain fixed across caps.
Each method receives 30,000 updates and produces 10,000 final proposals, with five training seeds per setting.
The checkpoint-selection bank and final evaluation bank are separate.
Selection uses training-reference features; final feature evaluation uses test references.
Checkpoints are selected first by strict validity, then by strict VU without a novelty requirement, and finally by a weighted feature MMD criterion.
The last criterion weights the nonnegative atom, valence, and spectral MMD terms by $(2,1,1)/4$.

For $n$ active nodes, the generator samples $n$ atom states and $\binom{n}{2}$ bond states.
The maximum number of active categorical choices is therefore
\begin{equation}
 D(k)=k+\binom{k}{2}=\frac{k(k+1)}{2},\qquad
 D(16)=136,\quad D(24)=300,\quad D(32)=528.
\end{equation}
These are upper bounds on active choices, not changes in network size.
Increasing the cap also changes the data distribution, so the experiment does not isolate the effect of decision count alone.

\paragraph{Training comparisons.}
\methodname combines the hard Energy objective and structural penalties with a hashed rarity reward.
\methodnamez and STE evaluate the same numerical hard forward criterion.
STE uses a soft feature Jacobian during backpropagation, while \methodnamez still uses population rank shaping.
Soft optimizes the corresponding relaxed objective.
Thus, the comparison with STE matches the forward criterion, and the comparison between the two \methodname variants measures the effect of rarity.
Update counts are equal, but training work is not: an \methodname update evaluates 192 perturbed networks on 96 graphs each, whereas a gradient update uses 96 graphs.
The separate experiment in Section~\ref{app:zinc-time-budget} controls active training time.

\paragraph{Feature discrepancy and uncertainty.}
$\mathrm{MMD}_{\rm cell}$ is the arithmetic mean of unbiased RBF MMD$^2$ estimates for atom, valence, and spectral features on raw generated graphs, using test references and up to 1,500 samples per side.
It is compared only within a fixed dataset and cap, because reference distributions and kernel calibration can differ.
Connectivity is measured on all proposals, including chemically invalid graphs.
Paired differences use the five common seeds and unadjusted 95\% Student-$t$ intervals,
$\overline{\Delta}\pm t_{0.975,4}s_{\Delta}/\sqrt{5}$.
The reported intervals use the rounded critical value $t_{0.975,4}\approx2.776$.
These intervals are distinct from the sample standard deviations in the result tables and are not adjusted for multiple comparisons.

Figure~\ref{fig:ksweep-strict-vun} summarizes all six size settings.
Both \methodname variants have higher mean Strict VUN than STE and Soft in every setting.
The detailed tables below show the validity, connectivity, feature discrepancy, and paired uncertainty behind this pattern.

\begin{figure}[htbp]
 \centering
 \includegraphics[width=\linewidth]{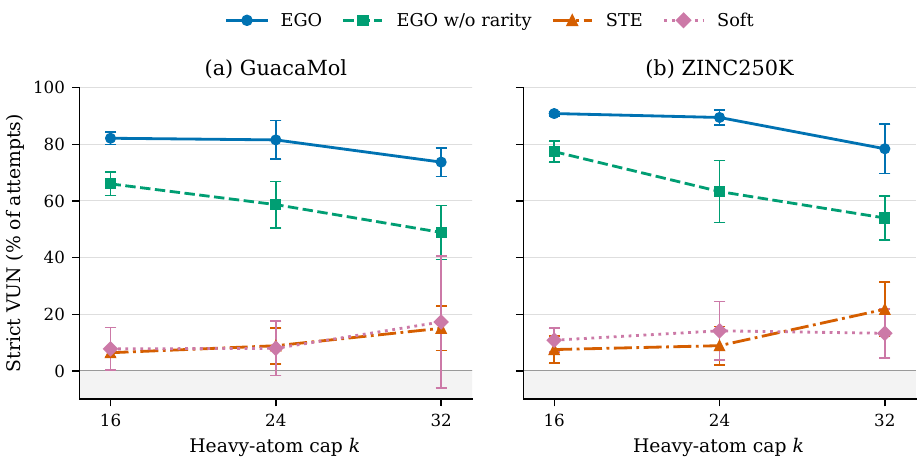}
 \caption{Strict VUN across cumulative molecular size caps. Points and bars show mean $\pm$ sample standard deviation across five seeds, with 30,000 updates and 10,000 final proposals per run. \methodname includes rarity; \methodnamez removes it. Increasing the cap changes the molecular size mixture while the output canvas remains fixed. Error bars are not truncated at zero.}
 \label{fig:ksweep-strict-vun}
\end{figure}

\FloatBarrier
\paragraph{Key takeaway.}
\methodname maintains a yield advantage across the tested size mixtures, including settings with up to 528 active choices.
Rarity further improves useful yield, with a small increase in feature MMD.
This is evidence across equal update budgets, while the separate time control tests the effect of unequal work per update.

\FloatBarrier
\subsection{QM9}
\label{app:external-qm9}

QM9 contains small organic molecules with up to nine heavy atoms drawn from carbon, nitrogen, oxygen, and fluorine, with hydrogen completing their valences~\citep{ramakrishnan2014qm9}.
The original dataset contains 133,885 molecules.
We use QM9 to study molecular quality and distinct yield on a small graph canvas.
Our prepared split has 107,109 source training rows, corresponding to 107,061 unique training molecules, and 13,374 test rows.
External studies can use different splits and evaluation rules.

\FloatBarrier
\subsubsection{Comparison with published generators}

Table~\ref{tab:qm9-generators} gives the complete QM9 comparison, including autoregressive models, iterative diffusion and flow models, SMILES generators, graph generators, and the REINFORCE training recipes.
Validity and uniqueness describe different parts of molecular quality: uniqueness is conditional on validity, while FCD measures distributional distance.
VU/s reports distinct valid outputs per generation second when compatible timing evidence is available.
External rows retain the identity, sampling, postprocessing, and timing rules of their sources, as detailed below.
Section~\ref{app:external-accounting} separately reports estimated dense inference efficiency.

\begin{table}[htbp]
  \centering
  \caption{QM9 generation: published baselines and training comparisons at 120,000 updates.
    \methodname and REINFORCE use the same $3\times768$ generator and three training seeds, 42, 43, and 44; their dispersion is the sample SD.
    Uniqueness is conditional on validity, with the identity and postprocessing rules detailed in the text.
    VU/s counts valid, unique outputs per generation second.
    External rows retain the protocols, hardware, and uncertainty conventions of their sources.
    Gray rows identify one-step generators; dashes denote unavailable compatible evidence.}
  \label{tab:qm9-generators}
  \setlength{\tabcolsep}{4pt}
  \begin{tabularx}{\linewidth}{@{}>{\raggedright\arraybackslash}Xcrrrr@{}}
    \toprule
    Method & Steps & \shortstack{Valid (\%)\\$\uparrow$} & \shortstack{Unique (\%)\\$\uparrow$}
    & FCD $\downarrow$ & VU/s $\uparrow$\\
    \midrule

    \multicolumn{6}{@{}l}{\textit{Autoregressive generation}}\\
    GraphARM~\citep{kong2023grapharm}
    & $n\leq9$ & 90.25 & 95.62 & 1.22 & ---\\
    G2PT-base~\citep{chen2025g2pt}
    & $L$ tokens & 99.0 & 96.8 & 0.06 & ---\\

    \addlinespace
    \multicolumn{6}{@{}l}{\textit{Iterative diffusion / flow}}\\
    GDSS$^{a}$~\citep{jo2022gdss}
    & 1,000 & 95.72 & 98.46 & 2.900 & $\approx86.4$\\
    DiGress~\citep{vignac2023digress}
    & 500 & 99.0 & 96.2 & --- & ---\\
    GruM$^{a}$~\citep{jo2024grum}
    & 1,000 & 99.69 & 96.90 & 0.108 & ---\\
    Cometh~\citep{siraudin2025cometh}
    & 500 & 99.57 & 96.76 & 0.25 & ---\\
    CatFlow$^{b}$~\citep{eijkelboom2024catflow}
    & 100 & 99.81 & 99.95 & 0.441 & ---\\
    DeFoG~\citep{qin2025defog}
    & 50 & 98.9 & 96.2 & 0.26 & ---\\
    DeFoG~\citep{qin2025defog}
    & 500 & 99.3 & 96.3 & 0.12 & ---\\

    \addlinespace
    MELD~\citep{seo2026meld}
    & --- & 100.00 & 96.49 & 0.09 & ---\\

    \addlinespace
    \multicolumn{6}{@{}l}{\textit{Few step SMILES generation}$^{c}$}\\
    UDLM + DCD
    & $1+\mathrm{tail}$ & 31.54 & $\approx99.35$ & --- & ---\\
    PairFlow + DCD
    & $1+\mathrm{tail}$ & 44.32 & $\approx99.56$ & --- & ---\\
    UDLM + ReDi
    & $1+\mathrm{tail}$ & 5.83 & $\approx100.00$ & --- & ---\\
    PairFlow + ReDi
    & $1+\mathrm{tail}$ & 35.25 & $\approx99.58$ & --- & ---\\

    \addlinespace
    \multicolumn{6}{@{}l}{\textit{One-step SMILES generation}$^{d}$}\\
    \rowcolor{black!6}
    Euclidean Laplace drifting
    & 1 & 22.0 & 40.0 & --- & ---\\
    \rowcolor{black!6}
    Kernel-gradient drifting
    & 1 & 38.9 & 44.1 & --- & ---\\

    \addlinespace
    \multicolumn{6}{@{}l}{\textit{One-step graph generation}}\\
    \rowcolor{black!6}
    CFM-CSD~\citep{roos2026categoricalflowmaps}
    & 1 & 90.0 & 91.8 & 2.14 & ---\\
    \rowcolor{black!6}
    CFM-ECLD~\citep{roos2026categoricalflowmaps}
    & 1 & 95.8 & 89.2 & 2.16 & ---\\
    \rowcolor{black!6}
    MolGAN~\citep{decao2018molgan}
    & 1 & 98.1 & 10.4 & --- & ---\\
    \rowcolor{black!6}
    MoFlow~\citep{zang2020moflow}
    & 1 & $100.00\pm0.00$ & $99.20\pm0.12$ & --- & ---\\

    \midrule
    \rowcolor{black!6}
    RLOO$^{e}$
    & 1 & $88.26\pm3.12$ & $50.68\pm9.21$
    & $3.907\pm0.162$ & $\approx3{,}927$\\
    \rowcolor{black!6}
    RLOO + rarity$^{e}$
    & 1 & $96.80\pm2.19$ & $51.32\pm3.14$
    & $2.710\pm0.535$ & $\approx4{,}425$\\
    \rowcolor{black!6}
    \textbf{\methodname ($3\times768$)}
    & \textbf{1} & $93.46\pm0.82$ & $83.69\pm1.49$
    & $3.534\pm0.638$ & $\approx6{,}746$\\

    \bottomrule
  \end{tabularx}

\end{table}

\paragraph{Published graph protocols.}
In Table~\ref{tab:qm9-generators}, $n$ denotes heavy atoms and $L$ the token sequence length; iterative steps need not equal network evaluations.
For rows marked $^{a}$, GDSS and GruM report validity before correction, whereas uniqueness and FCD use corrected outputs.
GDSS VU/s uses a TITAN RTX and 20 CPU cores, while \methodname uses an A100 and includes generation through SMILES/CSV, excluding later evaluation.
For CatFlow, marked $^{b}$, the step count follows \citet{roos2026categoricalflowmaps}.
DiGress, Cometh, and CFM use largest component identities.
DeFoG reports standard validity and uniqueness among relaxed valid molecules using largest component identities.
MELD reports 10,000 generated samples; its QM9 sampling step count and compatible VU/s are left unavailable here.
MoFlow retains the validity correction used in its original Table~2, with mean and reported SD over five generation runs of 10,000 molecules~\citep{zang2020moflow}.
No throughput from a different MoFlow quality evaluation is attached to that row.

\paragraph{Published SMILES protocols.}
The four rows marked $^{c}$ come from \citet{park2026pairflow} and use ten trials of 1,024 proposals.
Validity is $100\bar V/1{,}024$, and reconstructed uniqueness is $100\bar U/\bar V$, a ratio of mean counts.
Here, $U$ counts distinct valid SMILES strings without canonicalization.
One reported sampling step is followed by greedy tail denoising, and the released sampler adds a final network evaluation.
The rows marked $^{d}$ come from \citet{esteban2026kernelgradient}; kernel gradient drifting denotes the spherical Laplace gradient variant.
Both drifting rows use canonical largest fragment uniqueness among valid samples, and their evaluation budgets differ from 10,000 proposals.

\paragraph{\methodname and REINFORCE protocols.}
The rows marked $^{e}$ and \methodname use full molecule identities.
REINFORCE conditional uniqueness is aggregated over the same three training seeds.
These training recipes match update counts, but differ in reward scalarization, hard reward evaluation counts, and training time.
Section~\ref{app:qm9-reinforce} reports strict yields, checkpoint selection, and training costs.

\paragraph{\methodname protocol and generation time.}
\label{app:qm9-throughput}
\methodname uses three hidden layers of width 768, with 1,456,380 parameters and a canvas of nine heavy atoms.
Molecular sizes follow the training distribution.
Each seed completes 120,000 updates and uses a checkpoint selected by validation.
Count metrics use 10,000 raw proposals, without rejection sampling or graph repair.
FCD uses the first 10,000 full valid molecules from a separate bank of 30,000 proposals, retains duplicates, and compares them with 10,000 test reference molecules.
These sample counts are met in all three seeds.
The mean generation time through SMILES/CSV is $1.159\pm0.008$ seconds on an A100 40GB, corresponding to approximately 6,746 valid unique outputs per generation second.
This ratio uses the mean useful count and mean time, and excludes later deduplication, novelty checks, and distribution scoring.

\FloatBarrier
\paragraph{Key takeaway.}
\methodname combines $93.46\pm0.82$\% validity with $83.69\pm1.49$\% conditional uniqueness in a single network evaluation.
Several published generators achieve higher validity or lower FCD, while both REINFORCE recipes produce more repetitions in this comparison.
Table~\ref{tab:molecular-dense-audit} shows the associated cost tradeoff: compared with CFM-ECLD, \methodname has approximately $52.8\times$ higher estimated VU efficiency, with VU lower by 7.24 percentage points and FCD higher by 1.37.
Differences in source protocols and hardware prevent a direct runtime ranking from Table~\ref{tab:qm9-generators}.

\FloatBarrier
\subsubsection{Network Size and Training Budget}
\label{app:qm9-sweep}
\label{app:capacity-budget}

\paragraph{Motivation.}
The molecular size sweep changes the distribution of graph sizes while keeping the network fixed.
Here, we instead vary the depth and width of the \methodname generator on QM9 to test whether more network capacity improves molecular quality and distinct yield.
We also evaluate three training budgets to check whether a compact network can retain the quality of the main configuration and whether additional updates help each architecture.

\paragraph{Protocol.}
We evaluate five \methodname architectures at 30,000, 60,000, and 120,000 updates, with training seeds 42, 43, and 44 for every architecture.
The notation $L\times H$ gives the number of hidden layers and their width.
The main configuration remains $3\times768$, as in Table~\ref{tab:qm9-generators}; the other networks vary in depth or width.
The $2\times768$ network is a compact alternative, reported alongside every tested architecture.
Equal update counts do not imply equal computation across network sizes.
Differences between means over three seeds do not by themselves establish statistically significant superiority.

Each budget evaluates the checkpoint selected by validation within that budget, which can precede the final update.
Count metrics use 10,000 raw proposals per seed.
Full validity requires chemical sanitization of the complete output, and strict validity additionally requires connectedness.
Strict VUN divides the number of distinct valid connected molecules absent from training by all proposals; it is not uniqueness conditioned on validity.
Molecular sizes follow the distribution in the training data.
A separate bank of 30,000 raw proposals provides up to 10,000 full valid molecules for FCD, retaining duplicates and using the test reference.
FCD is a distance metric and does not share the denominator of the primary proposal bank.
For $3\times768$, every seed uses 10,000 generated molecules and 10,000 test reference molecules for FCD.
The prepared split has 107,109 source training rows, including 107,061 unique molecules, and 13,374 test rows; external studies can use other splits.

\paragraph{Architecture comparison.}
Table~\ref{tab:qm9-architecture-120k} reports all five \methodname architectures at 120,000 updates.

\begin{table}[htbp]
  \centering
  \caption{Complete \methodname network size sweep on QM9 at the 120,000 update budget.
    Metrics are mean $\pm$ sample standard deviation over seeds 42, 43, and 44.
    All columns except parameters and FCD are percentages.
    Parameter counts are rounded to millions; selected checkpoints can precede the budget milestone.}
  \label{tab:qm9-architecture-120k}
  \setlength{\tabcolsep}{4pt}
  \begin{tabular}{@{}lrrrrr@{}}
    \toprule
    Architecture & \shortstack{Params.\\(M)} & Full valid & Strict valid & Strict VUN & FCD $\downarrow$\\
    \midrule
    $2\times768$ & 0.866 & $94.05\pm1.36$ & $92.03\pm0.81$ & $68.70\pm0.98$ & $3.414\pm0.336$\\
    $3\times768$ & 1.456 & $93.46\pm0.82$ & $89.49\pm0.22$ & $68.31\pm1.03$ & $3.534\pm0.638$\\
    $4\times768$ & 2.047 & $93.50\pm0.31$ & $89.29\pm1.49$ & $67.10\pm1.94$ & $3.966\pm0.789$\\
    $3\times384$ & 0.433 & $94.30\pm1.13$ & $91.25\pm1.72$ & $67.06\pm2.16$ & $3.452\pm0.200$\\
    $3\times1024$ & 2.466 & $93.06\pm0.72$ & $89.35\pm1.33$ & $66.50\pm3.42$ & $3.679\pm0.293$\\
    \bottomrule
  \end{tabular}
\end{table}

The $2\times768$ \methodname network has about 41\% fewer parameters than the main $3\times768$ network.
Its mean Strict VUN is higher by 0.39 percentage points, corresponding to about 39 molecules per 10,000 proposals.
The ordering changes when connectedness is not required: full VUN is 71.99\% for $3\times768$ and 70.57\% for $2\times768$.
The main configuration has higher full VUN in every seed, with an aggregate of $71.99\pm0.85$\%; only the mean was supplied for the compact configuration.

The main \methodname configuration has 1,456,380 parameters.
Its connectedness is $95.93\pm0.65$\%, conditional uniqueness among full valid outputs is $83.69\pm1.49$\%, and novelty among those unique outputs is $92.07\pm2.01$\%.
Full VU is $78.21\pm0.84$\%, while strict unique yield is $74.53\pm1.17$\%; novelty among strict unique outputs is $91.68\pm2.11$\%.
These quantities use different denominators and should not be interchanged with Strict VUN.
FCD against the training reference is $3.501\pm0.606$, separate from the test reference FCD in Table~\ref{tab:qm9-architecture-120k}.
The descriptor KL score is $0.912\pm0.022$ and uses the training reference.

\paragraph{Training budget and checkpoint selection.}
Table~\ref{tab:qm9-budget-sweep} shows all five trajectories across the three budgets.
The milestones are repeated evaluations within each training run, rather than independent training repetitions.
Strict VUN increases at both budget transitions in all 15 runs, across the five architectures and three seeds.

\begin{table}[htbp]
  \centering
  \caption{\methodname Strict VUN (\%) on QM9 across training budgets, averaged over seeds 42, 43, and 44.
    Each result uses the checkpoint selected by validation within its budget.
    The dispersion at 120,000 updates is reported in Table~\ref{tab:qm9-architecture-120k}.}
  \label{tab:qm9-budget-sweep}
  \begin{tabular}{@{}lrrr@{}}
    \toprule
    Architecture & 30,000 updates & 60,000 updates & 120,000 updates\\
    \midrule
    $2\times768$ & 63.04 & 66.23 & 68.70\\
    $3\times768$ & 63.19 & 66.19 & 68.31\\
    $4\times768$ & 62.54 & 64.33 & 67.10\\
    $3\times384$ & 61.02 & 63.40 & 67.06\\
    $3\times1024$ & 62.16 & 64.66 & 66.50\\
    \bottomrule
  \end{tabular}
\end{table}

At the 120,000 update budget, the selected updates for $2\times768$ are 118,000, 120,000, and 116,000 for seeds 42, 43, and 44, respectively.
For $3\times768$, they are 82,000, 90,000, and 100,000.
These choices distinguish the completed training budget from the checkpoint used for evaluation.
A larger budget does not improve every metric: $4\times768$ has worse FCD at 120,000 than at 60,000 updates despite its higher Strict VUN.
The numerical FCD at 60,000 updates and the metric dispersions at the lower budgets were not supplied in this summary.

The earlier result of $64.39\pm1.85$\% Strict VUN for $3\times768$ at 60,000 updates belongs to a separate record.
It is not pooled with, or substituted for, the 66.19\% mean at that budget in the present sweep.

\paragraph{Generation cost.}
Table~\ref{tab:qm9-dense-efficiency} compares estimated dense work and measured generation time for the same five \methodname architectures at 120,000 updates.
The accounting follows Section~\ref{app:external-accounting}, and the timing scope follows Section~\ref{app:qm9-throughput}.
This comparison tests whether the gain from reducing dense work also appears in the time required to obtain molecular outputs.

\begin{table}[htbp]
  \centering
  \caption{\methodname inference cost across the QM9 network size sweep at 120,000 updates.
    Dense MFLOP is the estimated work per proposal, and $E_{\mathrm{VU}}$ is the number of distinct valid outputs per $10^{12}$ estimated dense operations.
    Times are seconds per 10,000 proposals through SMILES/CSV on an A100 40GB, with mean $\pm$ sample SD over seeds 42, 43, and 44.
    VU/s is the mean distinct valid count divided by the mean time.}
  \label{tab:qm9-dense-efficiency}
  \setlength{\tabcolsep}{3pt}
  \begin{tabular}{@{}lrrrr@{}}
    \toprule
    \methodname architecture & \shortstack{Dense MFLOP\\$\downarrow$} & $E_{\mathrm{VU}}$ $\uparrow$ & Time (s) $\downarrow$ & VU/s $\uparrow$\\
    \midrule
    $3\times768$ (primary) & $2.908$ & $268{,}980.3$ & $1.159\pm0.008$ & $\approx6{,}746$\\
    $2\times768$ & $1.728$ & $445{,}409.0$ & $1.144\pm0.008$ & $\approx6{,}725$\\
    $3\times384$ & $0.864$ & $880{,}169.8$ & $1.164\pm0.010$ & $\approx6{,}536$\\
    $4\times768$ & $4.087$ & $186{,}692.3$ & $1.161\pm0.009$ & $\approx6{,}574$\\
    $3\times1024$ & $4.925$ & $157{,}535.8$ & $1.167\pm0.013$ & $\approx6{,}647$\\
    \bottomrule
  \end{tabular}
\end{table}

Dense work includes all hidden linear maps and both output heads, with two operations per multiply-and-accumulate.
It excludes biases, activations, feature construction, sampling, and chemistry processing.
The reported $E_{\mathrm{VU}}$ is the mean across seeds and uses full VU, without a connectedness or novelty requirement.
VU/s is a ratio of means, not the mean of paired rates; no rate standard deviation is inferred from separate aggregate measurements.

The main $3\times768$ \methodname generator takes $1.159\pm0.008$ seconds through SMILES/CSV, including $0.03648\pm0.00042$ seconds for raw graph generation.
Its mean full valid unique count is 7,821, and its mean Strict VUN count is 6,831.
Dividing these counts by the unrounded mean generation time gives approximately 6,746 VU/s and 5,892 Strict VUN outputs per generation second.
The timer includes graph generation, host transfer, chemistry conversion, and SMILES/CSV writing, but excludes loading, compilation, warmup, and later screening.
Useful counts come from the later evaluation, so these rates do not measure the complete screening process.
The total time spent in optimizer updates over the 120,000 update budget is $1{,}079.56\pm4.43$ seconds, excluding preparation, compilation, validation, and checkpointing.
The separately reported evaluation timer is $114.14\pm0.72$ seconds and is not used as generation time.

Mean generation times range from 1.144 to 1.167 seconds per 10,000 proposals across the tested networks.
The $3\times384$ \methodname network has higher VU per estimated dense work than the primary configuration but produces approximately 6,536 VU/s, compared with 6,746 for $3\times768$.
Fewer dense operations therefore do not imply higher measured throughput through SMILES/CSV in these runs.

\FloatBarrier
\paragraph{Key takeaway.}
A larger \methodname network does not consistently improve molecular quality or Strict VUN on QM9.
The $2\times768$ network uses about 41\% fewer parameters than the main configuration and reaches a similar mean Strict VUN, while the main configuration retains higher full VUN.
Longer training improves Strict VUN at both budget transitions in every run, although FCD does not always improve.
Compact networks reduce estimated dense work, while the measured times through molecular conversion remain similar across architectures.

\FloatBarrier
\subsubsection{REINFORCE comparison at 120,000 updates}

\label{app:qm9-reinforce}

\paragraph{Motivation.}
REINFORCE can optimize rewards on hard molecular outputs without differentiating through sampling.
We first compare \methodname with REINFORCE under the same architecture and update count to test whether this direct reward access also preserves a large pool of distinct novel molecules.
This experiment compares the full training recipes; the following experiment then matches the reward and the number of generated graphs scored.

\paragraph{Why leave-one-out.}
We use REINFORCE with a leave-one-out baseline (RLOO) to control the variance of the gradient estimate.
Here, vanilla REINFORCE means REINFORCE without a baseline.
Our reward is defined over a batch and includes pairwise Energy terms.
Without a baseline, each sampled molecule is weighted by the full batch reward, including variation caused by the other molecules.
For each molecule, RLOO subtracts the reward recomputed after removing that molecule and renormalizing the remaining terms.
The name refers to this removal of one molecule at a time.
This measures a contribution relative to the remaining batch.
The baseline reuses the sampled graphs and distance matrices, so it requires no additional generated molecules or learned critic.
Throughout this appendix, ``REINFORCE'' refers to our implementation with a leave-one-out baseline (RLOO), unless explicitly described as vanilla REINFORCE.
We use this baseline with the aim of reducing gradient variance in the comparison with \methodname.
The estimator preserves the expected gradient under the conditions stated in Section~\ref{app:qm9-optimizer-control}.
We do not measure its variance or performance relative to vanilla REINFORCE in these experiments.

\paragraph{Protocol and training objective.}
We compare \methodname with RLOO, with and without rarity, using the same one-step QM9 generator: three hidden layers of width 768, 1,456,380 parameters, a canvas of nine heavy atoms, and a training-derived size prior. Each configuration uses seeds 42, 43, and 44 and completes 120,000 updates. REINFORCE uses 96 policy samples and 192 reference samples per update. Final count metrics use 10,000 raw proposals per seed, with no rejection sampling or graph repair; FCD uses a separate sampling bank. REINFORCE scores the hard Energy criterion and structural penalties directly, with feature scales fixed from training data. The rarity variant adds the hashed archive reward with weight 0.20, applied only to valid graphs.
\paragraph{Reward evaluation.}
The leave-one-out estimator is defined in Section~\ref{app:qm9-optimizer-control}.
For this experiment, feature scales and the rarity archive remain fixed while an update is scored, and archive entries are committed afterward.
Log probabilities include only active node and edge decisions.
The runs use no advantage whitening, entropy bonus, or gradient clipping.

\paragraph{Validity, distinct yield, and fidelity.}
Table~\ref{tab:qm9-reinforce-120k} shows a tradeoff: RLOO with rarity achieves higher full and strict validity and lower FCD, whereas \methodname produces more distinct valid molecules and more distinct novel connected molecules. \methodname reaches $68.31\%$ Strict VUN versus $42.07\%$ for REINFORCE with rarity, a difference of 26.24 percentage points and a ratio of approximately $1.62\times$. This corresponds to about 6,831 versus 4,207 distinct novel connected molecules per 10,000 attempts. Conditional uniqueness among full valid outputs is $83.69\%$ for \methodname and $51.32\%$ for REINFORCE with rarity; these percentages have a different denominator from strict unique yield. FCD evaluates the feature distribution of valid outputs while retaining duplicates, so lower FCD can coexist with lower distinct yield.

\begin{table}[htbp]
\centering
\caption{QM9 training comparison at 120,000 updates: mean $\pm$ sample SD across seeds 42, 43, and 44. All columns except raw FCD are percentages. Strict validity requires chemical validity and connectedness. Strict unique yield divides distinct valid connected molecules by all attempts; Strict VUN additionally requires absence from training. Bold values identify the best mean within this comparison.}
\label{tab:qm9-reinforce-120k}
\setlength{\tabcolsep}{4pt}
\begin{tabular}{@{}lrrrrr@{}}
\toprule
Method & Full valid $\uparrow$ & Strict valid $\uparrow$ & \shortstack{Strict unique\\yield $\uparrow$} & Strict VUN $\uparrow$ & FCD $\downarrow$\\
\midrule
\methodname & $93.46\pm0.82$ & $89.49\pm0.22$ & $\mathbf{74.53\pm1.17}$ & $\mathbf{68.31\pm1.03}$ & $3.534\pm0.638$\\
RLOO & $88.26\pm3.12$ & $85.89\pm2.43$ & $42.58\pm7.09$ & $39.39\pm7.61$ & $3.907\pm0.162$\\
RLOO + rarity & $\mathbf{96.80\pm2.19}$ & $\mathbf{95.77\pm2.10}$ & $48.86\pm4.15$ & $42.07\pm2.38$ & $\mathbf{2.710\pm0.535}$\\
\bottomrule
\end{tabular}
\end{table}
\paragraph{Effect of rarity.}
Adding rarity to REINFORCE increases strict validity by 9.88 percentage points and strict unique yield by 6.28 points, but Strict VUN increases by only 2.68 points. The gain in distinct valid connected outputs therefore does not translate into an equally large gain in distinct novel outputs. Rarity improves validity and distinct yield without closing the Strict VUN gap to \methodname under this training recipe; this result does not imply that history-dependent rewards are ineffective for all score function methods.

\paragraph{Budget milestones and checkpoint selection.}
Each budget evaluates the best checkpoint found so far under the validation criterion.
For REINFORCE, the selected parameters are already available at the 30,000 update milestone and remain unchanged at 60,000 and 120,000 updates. Without rarity, the selected update is 2,000 in all three seeds. With rarity, the selected updates are 28,000, 30,000, and 10,000 for seeds 42, 43, and 44. No subsequent checkpoint improves that criterion, and the fixed generation banks reproduce the same metrics for the same parameters. All RLOO runs nevertheless complete 120,000 updates and 11,520,000 policy graph draws. For \methodname, the selected updates are 82,000, 90,000, and 100,000 for seeds 42, 43, and 44.
The selected checkpoints should therefore be distinguished from the training budget consumed.
\paragraph{Training cost and scope of the comparison.}
REINFORCE update time is approximately 361 to 366 seconds across the two configurations, compared with approximately 1,080 seconds for \methodname. These timers exclude compilation/warmup, validation, checkpointing, and subsequent generation and evaluation. The comparison matches architecture and update count, but not training time or objective evaluation workload. \methodname uses dynamic Energy scales and centered rank fitness shaping, whereas REINFORCE uses fixed scales and a raw scalar reward. Consequently, this experiment compares complete training recipes and does not isolate the estimator alone. It demonstrates that direct access to hard molecular criteria can yield high validity and favorable FCD without preserving the distinct novel yield achieved by \methodname in this comparison. All variants retain one-step inference; no inference speedup of \methodname over REINFORCE is claimed.
\paragraph{Distribution metrics.}
Validity, yield, and raw FCD evaluations are complete. Only the optional descriptor KL metric is incomplete and is omitted from this comparison.

\FloatBarrier
\paragraph{Key takeaway.}
At 120,000 updates, \methodname produces about $1.62\times$ as many distinct novel connected molecules as RLOO with rarity, while RLOO with rarity has higher validity and lower FCD.
The difference concerns the complete training recipes, including their reward scaling, fitness shaping, and training work.
The next comparison checks whether the advantage in distinct yield remains under a shared reward and matched reward evaluation budgets.

\FloatBarrier
\subsubsection{REINFORCE comparison at matched reward budgets}

\label{app:qm9-optimizer-control}

\paragraph{Motivation.}
The comparison at 120,000 updates gives each method a different number of reward evaluations.
We therefore run a separate experiment to ask whether \methodname still produces more distinct novel molecules when both methods score the same number of generated graphs with the same reward.
We also measure validity and FCD to show the quality tradeoffs that can accompany a gain in distinct yield.

\paragraph{Protocol and shared training objective.}
We compare \methodname with RLOO in a dedicated
QM9 experiment with five seeds, separate from the runs used for the main
molecular generation benchmarks. Both methods use the same one-step generator: three hidden layers of width 768, 1,456,380 parameters, a
canvas of nine heavy atoms, and a training-derived size prior.
The prepared QM9 split and initialization are paired by seeds 42 to 46.
Both optimize the same hard Energy reward minus the same structural
penalties, with feature scales computed once from 8,192 deterministic
training pairs and then held fixed. Rarity and rank shaping are disabled
in both methods. The recorded definitions of the reward and feature scales agree
across all ten runs.

\paragraph{Leave-one-out estimator.}
For a batch reward $R_B$ evaluated on samples $X_{1:B}$, RLOO uses
\begin{equation}
 A_i=R_B(X_{1:B})-R_{B-1}(X_{-i}),\qquad
 \mathcal L_{\mathrm{LOO}}
 =-\sum_{i=1}^{B}\operatorname{stopgrad}(A_i)
 \log p_\theta(X_i\mid z_i,n_i).
\end{equation}
The delete-one reward removes both the row and column of sample $i$
from the Energy distances between generated samples and renormalizes the
U-statistic. Feature scales and positive references remain fixed while
scoring the update. Log probabilities include only active node and
edge decisions. RLOO uses 96 policy graphs and 192 positive training
references per update.

Conditional on the current parameters and training history, latent inputs, requested sizes, fixed feature scales, and reference panel, the generated graphs are sampled independently.
The baseline $b_i=R_{B-1}(X_{-i})$ therefore does not depend on the sampled graph $X_i$.
Its contribution to the expected gradient is zero:
\begin{equation}
 \mathbb{E}\!\left[b_i\nabla_\theta\log p_\theta(X_i\mid z_i,n_i)\right]=0.
\end{equation}
Subtracting this baseline leaves the expected gradient of the batch reward unchanged, following the baseline principle of REINFORCE~\citep{williams1992reinforce}.
Keeping feature scales and, when used, the rarity archive fixed during the update is necessary for this independence.
LOO is thus a variance control within REINFORCE, with the same hard reward and no learned value model.

\paragraph{Evolutionary updates and optimization settings.}
\methodname uses 96 antithetic rank-one perturbation pairs
(192 candidate networks), with 96 generated graphs per candidate.
Candidates share the positive references and generation random numbers
within an update. Updates use the numerical differences between paired
rewards, without a rank transformation. Matrix perturbation scales are
$\max(0.1\,\mathrm{RMS}(\theta_\ell),10^{-3})$, recomputed at each update;
vector parameters use scale 0.01. Candidates are evaluated in chunks
of 16.
Both methods use Adam with $\beta_1=0.9$, $\beta_2=0.999$, and
$\epsilon=10^{-8}$, with learning rates $10^{-2}$ for \methodname
and $3\times10^{-4}$ for RLOO. Neither method uses advantage whitening,
an entropy bonus, or gradient clipping. The comparison therefore
concerns the estimator and update recipes, including method-specific
optimization settings.

\paragraph{Matched reward budgets and checkpoint selection.}
An RLOO update scores 96 generated graphs, whereas an \methodname
update scores $192\times96=18{,}432$.
At the largest budget, 11,520,000 evaluations per method and seed, \methodname completes
625 updates, and RLOO completes 120,000.
Every reported result evaluates the current parameters at the declared
boundary, without substituting a validation-selected checkpoint.

Matching these evaluations controls the number of generated graphs
scored by the shared reward, not total FLOPs or elapsed training time.
Reference reuse and gradient computation differ. For example,
\methodname presents 120,000 positive reference rows at the largest
budget, versus 23,040,000 for RLOO, because its candidates reuse each
reference panel. Thus the reference sampling rule is shared, but the
total number of reference presentations is not matched.

\paragraph{Evaluation and uncertainty.}
Each checkpoint produces 10,000 raw proposals per seed, without
rejection sampling or graph repair.
Strict validity requires chemical validity and connectedness.
Strict unique yield divides the number of distinct valid connected
molecules by all attempts; Strict VUN additionally requires absence
from training.
FCD uses the first 10,000 chemically valid outputs from a separate
bank of 30,000 raw proposals, retaining duplicates, and a held-out
reference sample. All endpoint FCD evaluations meet this required sample count. Descriptor KL is omitted because its required number of
unique samples is not available for all five paired runs.

We report means and sample standard deviations across seeds 42 to 46.
Paired contrasts are \methodname minus RLOO.
Their two-sided 95\% Student-$t$ intervals use the five paired
differences and four degrees of freedom. These are descriptive,
pointwise intervals without correction across budgets or metrics.

\begin{table}[htbp]
\centering
\caption{QM9 comparison at 11,520,000 hard reward graph evaluations
per method and seed: mean $\pm$ sample SD over seeds 42 to 46.
The current checkpoints are at 625 updates for \methodname and
120,000 for RLOO. All columns except held-out raw FCD are percentages;
yield metrics use all 10,000 proposals as the denominator. Both
methods use the same reward and fixed scales, without rarity or rank
shaping.}
\label{tab:qm9-reinforce}
\label{tab:qm9-control-endpoints}
\setlength{\tabcolsep}{4pt}
\begin{tabular}{@{}lrrrrr@{}}
\toprule
Method & \shortstack{Full valid\\$\uparrow$} &
\shortstack{Strict valid\\$\uparrow$} &
\shortstack{Strict unique\\yield $\uparrow$} &
\shortstack{Strict VUN\\$\uparrow$} & FCD $\downarrow$\\
\midrule
\methodname & $76.71\pm3.67$ & $70.70\pm2.59$ &
$40.71\pm2.53$ & $38.39\pm2.89$ & $5.250\pm0.660$\\
RLOO & $98.64\pm0.86$ & $97.77\pm1.04$ &
$14.00\pm1.05$ & $12.05\pm0.53$ & $4.161\pm0.508$\\
\bottomrule
\end{tabular}
\end{table}

\paragraph{Validity, distinct yield, and fidelity.}
Table~\ref{tab:qm9-reinforce} reports the final comparison.
At the largest matched reward budget, \methodname reaches
$38.39\pm2.89$\% Strict VUN versus $12.05\pm0.53$\% for RLOO.
The paired difference is $+26.34$ percentage points, with a 95\%
interval of $[22.81,29.86]$, and is positive in all five seeds.
This corresponds to approximately 3,839 versus 1,205 distinct novel
connected molecules per 10,000 attempts.
The strict unique yield is also higher, $40.71\pm2.53$\% versus
$14.00\pm1.05$\%.

RLOO instead attains higher strict validity,
$97.77\pm1.04$\% versus $70.70\pm2.59$\%, and a lower mean FCD,
$4.161\pm0.508$ versus $5.250\pm0.660$.
The paired FCD difference is $+1.089$, with interval
$[-0.015,2.193]$, which includes zero.
Thus, the lower mean FCD is less conclusive than the paired differences
in Strict VUN and strict validity.
Because FCD retains repeated valid outputs, lower FCD can coexist
with a lower yield of distinct molecules.

\FloatBarrier
\paragraph{Key takeaway.}
The control supports a yield advantage for the update recipe of \methodname at the matched reward budgets even with identical reward
construction and counts of generated graphs scored.
Method-specific learning rates, perturbations, and reference reuse
remain part of the compared recipes, so the result does not isolate
the estimator independently of these choices.
It also does not establish uniform superiority across budgets,
a wall clock speedup, or an inference cost advantage over RLOO:
both methods retain the same one-step generator.

\FloatBarrier
\subsection{ZINC250K}
\label{app:zinc}

ZINC250K is a subset of the ZINC molecular collection~\citep{irwin2012zinc}.
It extends the evaluation from the small molecules in QM9 to graphs with up to 38 heavy atoms.
We use it to test whether \methodname can maintain a high yield of valid and distinct molecules on larger graphs.
The external comparison measures generation quality and estimated inference cost.
A separate experiment compares training methods under equal active training time.

\FloatBarrier
\subsubsection{Comparison with published generators}
\label{app:external-zinc}

Table~\ref{tab:zinc-generators} provides the full ZINC250K comparison with published generators.
It includes autoregressive generation, iterative diffusion and flow, and one-step generators.
The table reports chemical validity, uniqueness among valid molecules, FCD, and measured valid and unique outputs per second where available.
External rows retain the evaluation protocols of their sources.
Estimated inference cost is reported separately in Table~\ref{tab:molecular-dense-audit}; it is not a runtime measurement.

\begin{table}[htbp]
  \centering
  \caption{Full ZINC250K comparison with published generators.
    Validity and conditional uniqueness are percentages, FCD is a distance, and VU/s is a measured output rate.
    \methodname uses 120,000 training updates and reports mean $\pm$ sample SD where available.
    Its FCD is a single recovered evaluation.
    Gray rows identify one-step generators; dashes denote unavailable measurements.}
  \label{tab:zinc-generators}
  \setlength{\tabcolsep}{3pt}
  \begin{tabularx}{\linewidth}{@{}>{\raggedright\arraybackslash}Xcrrrr@{}}
    \toprule
    Method & Steps & Valid (\%) $\uparrow$ & Unique (\%) $\uparrow$ & FCD $\downarrow$ & VU/s $\uparrow$\\
    \midrule
    \multicolumn{6}{@{}l}{\textit{Autoregressive generation and autoregressive diffusion}}\\
    GraphARM~\citep{kong2023grapharm} & $n\leq38$ & 88.23 & 99.46 & 16.26 & ---\\
    LO-ARM-st-sep~\citep{wang2025loarm} & $D(n)$ & 96.26 & 100.00 & 3.229 & ---\\
    PARD~\citep{zhao2024pard} & blockwise & 95.23 & 99.99 & 1.98 & ---\\
    \addlinespace
    \multicolumn{6}{@{}l}{\textit{Iterative diffusion and flow}}\\
    GDSS$^{a}$~\citep{jo2022gdss} & 1,000 & 97.01 & 99.64 & 14.656 & $\approx4.93$\\
    DiGress$^{b}$~\citep{jo2024grum} & 1,000 & 94.99 & 99.97 & 3.482 & ---\\
    GruM~\citep{jo2024grum} & 1,000 & 98.65 & 99.97 & 2.257 & ---\\
    CatFlow~\citep{eijkelboom2024catflow} & 100 & 99.21 & --- & 13.211 & ---\\
    MELD~\citep{seo2026meld} & 500 & 100.00 & 99.99 & 1.51 & ---\\
    DeFoG~\citep{qin2025defog} & 50 & 96.65 & 99.99 & 2.123 & ---\\
    DeFoG~\citep{qin2025defog} & 500 & 99.22 & 99.99 & 1.425 & ---\\
    \addlinespace
    \multicolumn{6}{@{}l}{\textit{One-step generation}}\\
    \rowcolor{gray!10}
    MoFlow~\citep{zang2020moflow} & 1 & 100.00 & --- & --- & ---\\
    \rowcolor{gray!10}
    CFM-CSD$^{c}$~\citep{roos2026categoricalflowmaps} & 1 & 73.5 & 99.9 & 19.3 & ---\\
    \rowcolor{gray!10}
    CFM-ECLD$^{c}$~\citep{roos2026categoricalflowmaps} & 1 & 93.5 & 99.5 & 20.0 & ---\\
    \midrule
    \rowcolor{gray!10}
    \textbf{\methodname (ours)} & \textbf{1} & $94.02\pm3.94$ & $96.42\pm0.21$ & $30.36^{d}$ & $2{,}525\pm160$\\
    \bottomrule
  \end{tabularx}
\end{table}

\paragraph{Source protocols and timing.}
For LO-ARM, $D(n)=n(n+1)/2$ counts node and edge decisions, and the reported row uses 16,384 samples.
PARD generates blocks autoregressively and uses diffusion within each block.
The Steps column therefore does not provide a uniform count of neural evaluations.
$^{a}$GDSS reports validity before correction, while its uniqueness, FCD, and VU/s use corrected outputs on a TITAN RTX with 20 CPU cores.
$^{b}$The DiGress row is the reproduction with 1,000 steps reported by GruM.
$^{c}$CFM computes uniqueness from identities of the largest component.
MoFlow retains the molecular correction used in its source evaluation.
$^{d}$The \methodname FCD is 30.3573 before rounding, from seed 42 at a budget of 120,000 updates.

\methodname uniqueness is computed within each run among full valid molecules before aggregation.
Values recovered from rounded logs are approximate.
The \methodname rate of $2{,}525\pm160$ VU/s is measured through molecular conversion and SMILES/CSV export, then aggregated across runs.
It excludes later FCD and other distributional evaluations.
Hardware and timing boundaries differ across sources, so these rates do not establish a speedup under matched conditions.

\paragraph{\methodname configuration and evaluation.}
\methodname uses three hidden layers of width 1,024, a canvas of 38 nodes, and a budget of 120,000 updates.
Molecular sizes are sampled from the training distribution.
This architecture was selected after inspecting the capacity sweep in Section~\ref{app:zinc-sweep} because it had the highest observed mean full validity and full VU at this budget.
The choice was therefore not fixed before that sweep.
Training uses seven hard feature families, a population of 192 perturbed networks, fitness batches of 96 graphs, and a rarity weight of 0.2.

The historical preparation collects the atom state vocabulary and allowed valences before the split and applies no additional deduplication.
A fixed permutation with seed 0 reserves up to 10,000 accepted rows for testing.
The positive feature panel and validation references use separate portions of the remaining training pool.
Checkpoint selection uses a fixed validation noise bank and first prefers full validity of at least 85\% and connectivity of at least 97\%.
It then uses uniqueness of the largest component adjusted by feature MMD; before feasibility, it primarily rewards full validity.
This differs from the protocol of the controlled size sweep in Section~\ref{app:ksweep}.

Each final evaluation uses 10,000 raw proposals.
Full validity tests the complete molecule, and full VU counts distinct full valid molecular identities per attempt.
FCD retains all full valid outputs from this bank, including duplicates, and compares them with the valid held-out reference.
It does not use the separate replenishment bank used for QM9.
The FCD evaluation for this configuration is from seed 42.

\FloatBarrier
\paragraph{Key takeaway.}
\methodname reaches $94.02\pm3.94$\% full validity and $96.42\pm0.21$\% uniqueness among valid molecules with one generator evaluation.
Its full VU yield is $90.66\pm4.00$\%, with a measured rate of $2{,}525\pm160$ valid and unique molecules per second through export.
Table~\ref{tab:molecular-dense-audit} reports an estimated efficiency of about $208\times$ the value for CFM-ECLD, while the mean VU yield of \methodname is lower by 2.37 percentage points.
The available FCD value of 30.36 is higher than 20.0 for CFM-ECLD and the values of the iterative methods in Table~\ref{tab:zinc-generators}.
\methodname therefore combines high useful yield with low estimated inference cost, while distributional fidelity remains a limitation under the reported protocols.

\FloatBarrier
\subsubsection{Network Size and Training Budget}
\label{app:zinc-sweep}

\paragraph{Motivation.}
We vary the network depth and width to test whether more capacity improves the validity, useful yield, and distributional fidelity of \methodname.
We also compare three training budgets to examine the effect of longer optimization.
This study changes the generator network while keeping the molecular dataset fixed.
The controlled molecular size sweep in Section~\ref{app:ksweep} asks a different question by changing the molecular mixture while keeping the network fixed.

\paragraph{Architectures and training.}
We compare five \methodname architectures at 30,000, 60,000, and 120,000 updates.
Architecture names give the number and width of the hidden layers.
The campaign schedules seeds 42, 43, and 44 for each architecture, with intermediate milestones along continuous training trajectories.
The original reference architecture is $3\times768$.
We use $3\times1024$ in the main comparison because it has the highest observed mean full validity and full VU at 120,000 updates.
This choice was made after inspecting the sweep summary.
Equal update budgets do not imply equal computation across network sizes.

Training follows the historical ZINC preparation and \methodname protocol on a canvas of 38 nodes.
Molecular sizes are sampled from the training distribution.
The configuration retains seven hard feature families, a population of 192 perturbed networks, fitness batches of 96 graphs, and a rarity weight of 0.2.
The atom state vocabulary and allowed valences are collected before the data split, and the split has no additional deduplication.
A fixed permutation with seed 0 reserves up to 10,000 accepted rows for testing.
The positive feature panel and validation references use separate portions of the remaining training pool.

\paragraph{Checkpoint selection.}
A fixed validation noise bank is used to select checkpoints.
The rule first prefers full validity of at least 85\% and connectivity of at least 97\%, then uses uniqueness of the largest component adjusted by feature MMD.
Before these thresholds are met, it primarily rewards full validity.
Each budget evaluates the best available checkpoint under this rule, which can differ from the last iterate.
This selection rule differs from the rule used in the controlled molecular size sweep.

\paragraph{Metrics and available evidence.}
Each final evaluation uses 10,000 raw proposals.
Full validity tests the complete molecule, and strict validity additionally requires connectivity.
Connectivity is measured over all proposals, including invalid ones.
Full VU counts distinct full valid molecular identities per attempt.
Strict VUN counts distinct valid, connected molecules absent from the training reference per attempt.

Conditional uniqueness is computed within each seed as $100$ times the full VU percentage divided by the full validity percentage.
The mean and sample SD summarize these ratios rather than the ratio of the aggregate means.
Reconstruction from rounded log percentages introduces a small approximation.
FCD uses all full valid molecules from the original 10,000 proposals, including duplicates, against the valid held-out reference.
The generated sample count therefore varies with validity, without the separate replenishment bank used for QM9.

The mean and sample SD use the available runs for each metric.
The starred FCD of $3\times1024$ is a single available evaluation from seed 42.
The starred Strict VUN and FCD values of $4\times768$ come from seed 44.
These values have no sample SD, and missing measurements remain unavailable.
Table~\ref{tab:zinc-architecture-120k} reports all architecture results at 120,000 updates.

\begin{table}[!ht]
  \centering
  \caption{\methodname architecture comparison on ZINC250K at 120,000 updates.
    Values are the mean $\pm$ sample SD over the available runs per metric.
    All columns except FCD are percentages.
    Unique given valid is computed within each run among full valid molecules before aggregation; reconstruction from rounded logs is approximate.
    Stars mark single available evaluations, with no sample SD; dashes denote unavailable metrics.
    The highlighted architecture is the profile used in the main comparison.}
  \label{tab:zinc-architecture-120k}
  \setlength{\tabcolsep}{3pt}
  \begin{tabular}{@{}lrrrrrrr@{}}
    \toprule
    Architecture & \shortstack{Full\\valid} & \shortstack{Strict\\valid} & Connected & \shortstack{Unique\\given valid} & Full VU & \shortstack{Strict\\VUN} & FCD $\downarrow$\\
    \midrule
    $3\times384$ & \shortstack[r]{$84.42$\\$\pm4.62$} & \shortstack[r]{$79.53$\\$\pm1.86$} & \shortstack[r]{$95.08$\\$\pm2.76$} & \shortstack[r]{$93.16$\\$\pm1.03$} & \shortstack[r]{$78.62$\\$\pm3.88$} & \shortstack[r]{$74.79$\\$\pm2.01$} & \shortstack[r]{$29.58$\\$\pm1.49$}\\
    \addlinespace[2pt]
    $2\times768$ & \shortstack[r]{$90.72$\\$\pm5.31$} & \shortstack[r]{$86.71$\\$\pm7.76$} & \shortstack[r]{$95.93$\\$\pm2.50$} & \shortstack[r]{$97.06$\\$\pm0.68$} & \shortstack[r]{$88.08$\\$\pm5.66$} & --- & ---\\
    \addlinespace[2pt]
    $3\times768$ & \shortstack[r]{$88.35$\\$\pm2.75$} & \shortstack[r]{$85.59$\\$\pm4.20$} & \shortstack[r]{$97.20$\\$\pm2.88$} & \shortstack[r]{$96.69$\\$\pm0.24$} & \shortstack[r]{$85.43$\\$\pm2.86$} & --- & ---\\
    \addlinespace[2pt]
    $4\times768$ & \shortstack[r]{$91.23$\\$\pm0.29$} & \shortstack[r]{$86.39$\\$\pm1.83$} & \shortstack[r]{$95.16$\\$\pm1.59$} & \shortstack[r]{$94.89$\\$\pm1.00$} & \shortstack[r]{$86.56$\\$\pm0.73$} & $81.70^{*}$ & $28.89^{*}$\\
    \addlinespace[2pt]
    \rowcolor{gray!10}
    $\boldsymbol{3\times1024}$ & \shortstack[r]{$94.02$\\$\pm3.94$} & \shortstack[r]{$90.84$\\$\pm6.97$} & \shortstack[r]{$96.82$\\$\pm3.03$} & \shortstack[r]{$96.42$\\$\pm0.21$} & \shortstack[r]{$90.66$\\$\pm4.00$} & \shortstack[r]{$88.22$\\$\pm6.86$} & $30.36^{*}$\\
    \bottomrule
  \end{tabular}
\end{table}

\paragraph{Effect of network capacity.}
At 120,000 updates, $3\times1024$ has the highest reported mean full validity, strict validity, and full VU in Table~\ref{tab:zinc-architecture-120k}.
It also has the highest reported Strict VUN among configurations for which this metric is available.
The $2\times768$ model has the highest mean conditional uniqueness, $97.06\pm0.68$\%, compared with $96.42\pm0.21$\% for $3\times1024$.
The choice of the wider model therefore reflects its higher useful yield rather than conditional uniqueness alone.

Relative to $3\times768$, the wider network improves mean full validity, strict validity, and full VU by 5.67, 5.25, and 5.23 percentage points, respectively.
However, its SD of 6.97 points for strict validity and 6.86 points for Strict VUN shows substantial variation across the available runs.
Missing Strict VUN values for two architectures and incomplete FCD coverage further limit the comparison.
The aggregates do not establish a statistically resolved advantage of the wider network or a benefit for distributional fidelity.

\paragraph{Effect of training budget.}
Table~\ref{tab:zinc-budget-sweep} reports full and strict validity at each training budget for all five architectures.
It uses checkpoints selected on validation data within each budget, without pooling the milestones as independent runs.

\begin{table}[!ht]
  \centering
  \caption{\methodname validity on ZINC250K across training budgets.
    Each cell reports mean $\pm$ sample SD over the available runs.
    Each milestone evaluates a checkpoint selected on validation data within the corresponding budget.
    Improvements describe aggregate means, rather than verified improvements in every individual trajectory.}
  \label{tab:zinc-budget-sweep}
  \setlength{\tabcolsep}{5pt}
  \begin{tabular}{@{}lrrr@{}}
    \toprule
    Architecture & 30,000 updates & 60,000 updates & 120,000 updates\\
    \midrule
    \multicolumn{4}{@{}l}{\textit{Full validity (\%)}}\\
    $3\times384$ & $82.53\pm4.53$ & $83.62\pm4.27$ & $84.42\pm4.62$\\
    $2\times768$ & $85.62\pm6.98$ & $88.73\pm5.78$ & $90.72\pm5.31$\\
    $3\times768$ & $85.20\pm3.00$ & $86.23\pm2.36$ & $88.35\pm2.75$\\
    $4\times768$ & $87.82\pm2.05$ & $90.68\pm0.13$ & $91.23\pm0.29$\\
    \rowcolor{gray!10}
    $3\times1024$ & $91.13\pm1.17$ & $92.54\pm2.29$ & $94.02\pm3.94$\\
    \midrule
    \multicolumn{4}{@{}l}{\textit{Strict validity (\%)}}\\
    $3\times384$ & $77.20\pm1.87$ & $78.52\pm1.05$ & $79.53\pm1.86$\\
    $2\times768$ & $81.41\pm9.20$ & $84.66\pm8.17$ & $86.71\pm7.76$\\
    $3\times768$ & $82.11\pm4.58$ & $83.28\pm3.54$ & $85.59\pm4.20$\\
    $4\times768$ & $82.63\pm0.65$ & $85.71\pm1.72$ & $86.39\pm1.83$\\
    \rowcolor{gray!10}
    $3\times1024$ & $87.99\pm4.11$ & $89.32\pm5.37$ & $90.84\pm6.97$\\
    \bottomrule
  \end{tabular}
\end{table}

Mean full and strict validity increase at both budget transitions for every architecture in Table~\ref{tab:zinc-budget-sweep}.
From 30,000 to 120,000 updates, full validity improves by 1.89 to 5.10 percentage points and strict validity by 2.33 to 5.30 points.
For $3\times1024$, these gains are 2.89 and 2.85 points, respectively.
The aggregate trend supports a benefit from longer optimization for validity.
The available summaries do not establish improvement in every seed or a corresponding budget trend for Strict VUN or FCD.

\paragraph{Graph generation throughput.}
Table~\ref{tab:zinc-graph-throughput} reports the number of full valid and unique molecules per second of graph generation at 120,000 updates.
For each available run $s$, the rate is $q_s=r_s y_s/100$, where $r_s$ is the reported sampling rate and $y_s$ is the full VU percentage.
The table summarizes these rates across runs.
We use the reported sampling rate instead of inverting the recorded generation time, which is rounded to 0.01 seconds in the CSV.

Molecular evaluation determines the valid and unique count after graph generation.
The timing denominator excludes SMILES conversion, RDKit validity checks, canonicalization, and deduplication.
It measures graph generation cost, while the throughput in Table~\ref{tab:zinc-generators} includes molecular conversion and export.
The available summaries provide rates for every architecture except $3\times384$.
The $4\times768$ rate is based on two available runs.

\begin{table}[!ht]
  \centering
  \caption{Full VU per second of graph generation for the \methodname architectures on ZINC250K at 120,000 updates: mean $\pm$ sample SD of rates computed within each run.
    Molecular conversion and evaluation are excluded from the timing denominator.
    $^{a}$The $4\times768$ aggregate uses two available runs.
    Dashes denote unavailable rates; the highlighted profile is the configuration used in the main table.}
  \label{tab:zinc-graph-throughput}
  \setlength{\tabcolsep}{6pt}
  \begin{tabular}{@{}lr@{}}
    \toprule
    Architecture & VU/s (graph generation only) $\uparrow$\\
    \midrule
    $3\times384$ & ---\\
    $2\times768$ & $722{,}733\pm257{,}183$\\
    $3\times768$ & $841{,}619\pm28{,}323$\\
    $4\times768$ & $(841{,}365\pm25{,}413)^{a}$\\
    \rowcolor{gray!10}
    $\boldsymbol{3\times1024}$ & $862{,}142\pm46{,}109$\\
    \bottomrule
  \end{tabular}
\end{table}

\paragraph{Timing through molecular export.}
For $3\times1024$ at 120,000 updates, Table~\ref{tab:zinc-generation-boundaries} separates graph generation from generation through SMILES/CSV export.
Both rates use the full VU yield of $90.66\pm4.00$\%, with the rate computed within each run before aggregation.
The broader timing boundary includes RDKit conversion and CSV export and excludes subsequent FCD and other distributional evaluations.

The approximate times for 10,000 raw proposals are 0.011 seconds for graph generation and 3.59 seconds through export.
Graph generation accounts for about 0.3\% of the reported time through export, so molecular conversion and export dominate this measurement.
These rounded time summaries have no supplied SD.
The throughput means are not reconstructed by dividing aggregate counts by the rounded times.

\begin{table}[!ht]
  \centering
  \caption{Measured timing boundaries for \methodname $3\times1024$ on ZINC250K at 120,000 updates.
    Time refers to 10,000 raw proposals and is reported approximately.
    VU/s uses fully valid distinct outputs, and reports mean $\pm$ sample SD of rates computed within each run.
    The broader boundary includes molecular conversion and CSV export; subsequent distributional evaluation is excluded.}
  \label{tab:zinc-generation-boundaries}
  \setlength{\tabcolsep}{5pt}
  \begin{tabular}{@{}lrr@{}}
    \toprule
    Timing boundary & Time (s) & Full VU/s $\uparrow$\\
    \midrule
    Graph generation only & $\approx0.011$ & $862{,}142\pm46{,}109$\\
    \rowcolor{gray!10}
    Generation through SMILES/CSV & $\approx3.59$ & $2{,}525\pm160$\\
    \bottomrule
  \end{tabular}
\end{table}

\FloatBarrier
\paragraph{Key takeaway.}
The wider $3\times1024$ network has the highest reported mean full VU in this sweep, and more training updates improve mean validity across all five architectures.
The available FCD results provide no evidence that greater width improves distributional fidelity.
The timing measurements also show that molecular conversion and export dominate the cost of obtaining usable \methodname outputs.
The external comparison therefore uses the measured export rate in Table~\ref{tab:zinc-generation-boundaries}, while Table~\ref{tab:zinc-graph-throughput} describes graph generation alone.
Differences in hardware and timing boundaries prevent these rates from establishing a speedup against external methods.
The historical ZINC validity result of 93.20\% comes from a separate run and is not pooled into this campaign.

\FloatBarrier
\subsubsection{Controlled Molecular Size Sweep}
\label{app:zinc-size-sweep}

\paragraph{Why this experiment.}
We use the shared size sweep in Section~\ref{app:ksweep} to test whether \methodname retains useful yield as the ZINC250K task includes larger molecules.
The fixed architecture separates changes in the molecular mixture from changes in network capacity.
The comparison also tests whether \methodname needs the rarity reward to improve on the gradient baselines.

\paragraph{Results.}
Table~\ref{tab:ksweep-zinc-all} reports all twelve configurations.
Mean Strict VUN for \methodname decreases from $90.84\%$ at $k=16$ to $78.40\%$ at $k=32$, but remains above both gradient baselines at every cap.
At $k=32$, Soft reaches higher strict validity than \methodname, $87.35\%$ versus $82.48\%$, but much lower Strict VUN, $13.25\%$ versus $78.40\%$.
Thus, high validity alone does not guarantee a large set of distinct novel molecules.

\begin{table}[htbp]
  \centering
  \caption{Complete cumulative ZINC250K sweep.
    Values are mean $\pm$ sample standard deviation across five seeds, with 30,000 updates per run.
    Full, Strict, Strict VUN, and Conn. are percentages; all use the complete attempt budget as the denominator.
    Lower $\mathrm{MMD}_{\rm cell}$ is better, with comparisons restricted to a fixed cap.}
  \label{tab:ksweep-zinc-all}
  \setlength{\tabcolsep}{3pt}
  \begin{tabular}{@{}rlrrrrr@{}}
    \toprule
    $k$ & Method & Full $\uparrow$ & Strict $\uparrow$ & Strict VUN $\uparrow$ & Conn. $\uparrow$ & $\mathrm{MMD}_{\rm cell}\downarrow$ \\
    \midrule
    16 & \methodname & 97.33 $\pm$ 1.11 & 94.91 $\pm$ 1.57 & 90.84 $\pm$ 0.77 & 97.55 $\pm$ 0.94 & 0.0084 $\pm$ 0.0007 \\
    16 & \methodnamez & 86.60 $\pm$ 2.81 & 83.82 $\pm$ 4.18 & 77.39 $\pm$ 3.75 & 97.07 $\pm$ 1.85 & 0.0052 $\pm$ 0.0009 \\
    16 & \shortstack[l]{STE\\(feature-ST)} & 99.21 $\pm$ 0.91 & 96.52 $\pm$ 1.75 & 7.51 $\pm$ 4.83 & 97.31 $\pm$ 2.29 & 0.0538 $\pm$ 0.0351 \\
    16 & Soft & 95.01 $\pm$ 1.74 & 89.33 $\pm$ 7.19 & 10.81 $\pm$ 4.21 & 94.29 $\pm$ 6.82 & 0.0430 $\pm$ 0.0313 \\
    \midrule
    24 & \methodname & 95.48 $\pm$ 2.98 & 91.69 $\pm$ 3.19 & 89.45 $\pm$ 2.67 & 96.18 $\pm$ 2.12 & 0.0099 $\pm$ 0.0015 \\
    24 & \methodnamez & 75.58 $\pm$ 11.00 & 69.38 $\pm$ 14.01 & 63.30 $\pm$ 10.96 & 93.17 $\pm$ 4.79 & 0.0065 $\pm$ 0.0012 \\
    24 & \shortstack[l]{STE\\(feature-ST)} & 97.63 $\pm$ 1.99 & 81.40 $\pm$ 32.05 & 8.88 $\pm$ 6.82 & 83.12 $\pm$ 31.59 & 0.0593 $\pm$ 0.0189 \\
    24 & Soft & 90.44 $\pm$ 6.28 & 75.20 $\pm$ 13.89 & 14.13 $\pm$ 10.38 & 84.41 $\pm$ 11.25 & 0.0614 $\pm$ 0.0402 \\
    \midrule
    32 & \methodname & 86.31 $\pm$ 8.58 & 82.48 $\pm$ 9.33 & 78.40 $\pm$ 8.68 & 96.09 $\pm$ 1.53 & 0.0125 $\pm$ 0.0032 \\
    32 & \methodnamez & 65.98 $\pm$ 7.85 & 60.64 $\pm$ 9.18 & 53.96 $\pm$ 7.82 & 93.92 $\pm$ 5.52 & 0.0084 $\pm$ 0.0017 \\
    32 & \shortstack[l]{STE\\(feature-ST)} & 80.41 $\pm$ 27.59 & 76.10 $\pm$ 26.07 & 21.74 $\pm$ 9.61 & 94.30 $\pm$ 3.21 & 0.0555 $\pm$ 0.0122 \\
    32 & Soft & 92.47 $\pm$ 5.25 & 87.35 $\pm$ 7.88 & 13.25 $\pm$ 8.61 & 94.41 $\pm$ 4.91 & 0.0653 $\pm$ 0.0228 \\
    \bottomrule
  \end{tabular}
\end{table}

\paragraph{Paired comparisons.}
Table~\ref{tab:ksweep-contrasts-zinc} gives the paired differences and their uncertainty.
Without rarity, \methodname exceeds STE in mean Strict VUN by 69.88, 54.41, and 32.22 percentage points at caps 16, 24, and 32, respectively.
All three intervals lie above zero.
Adding rarity raises Strict VUN by 13.45, 26.15, and 24.44 points, again with positive intervals.
\methodnamez has the lowest mean feature MMD at each cap, while rarity raises MMD as it improves useful yield.
MMD comparisons apply within each cap, not across different molecular mixtures.

\begin{table}[htbp]
  \centering
  \caption{ZINC250K paired contrasts: mean difference with its unadjusted paired 95\% confidence interval in brackets, over five seeds. Each difference is the first method minus the second. Validity and VUN use the strict definition, with differences reported in percentage points; MMD differences are multiplied by $10^3$. Positive validity or yield differences and negative MMD differences favor the first method.
  }
  \label{tab:ksweep-contrasts-zinc}

  \setlength{\tabcolsep}{3pt}
  \renewcommand{\arraystretch}{1.15}

\begin{tabular}{@{}rlrrr@{}}
    \toprule
    $k$ & Contrast
    & $\Delta$Validity
    & $\Delta$VUN
    & $10^3\Delta\mathrm{MMD}_{\rm cell}$ \\
    \midrule

    16 & \methodname $-$ \methodnamez
    & $11.08\;[6.72,\,15.44]$
    & $13.45\;[8.62,\,18.28]$
    & $3.23\;[2.35,\,4.12]$ \\

    16 & \methodnamez $-$ Soft
    & $-5.51\;[-13.45,\,2.44]$
    & $66.58\;[59.67,\,73.48]$
    & $-37.80\;[-76.99,\,1.39]$ \\

    16 & \methodnamez $-$ STE
    & $-12.70\;[-19.28,\,-6.12]$
    & $69.88\;[61.27,\,78.49]$
    & $-48.64\;[-91.73,\,-5.55]$ \\

    16 & STE $-$ Soft
    & $7.19\;[-1.20,\,15.58]$
    & $-3.30\;[-13.92,\,7.32]$
    & $10.84\;[-66.83,\,88.50]$ \\

    \midrule

    24 & \methodname $-$ \methodnamez
    & $22.31\;[6.40,\,38.21]$
    & $26.15\;[13.23,\,39.08]$
    & $3.40\;[0.85,\,5.95]$ \\

    24 & \methodnamez $-$ Soft
    & $-5.82\;[-37.61,\,25.97]$
    & $49.17\;[40.16,\,58.17]$
    & $-54.89\;[-104.17,\,-5.60]$ \\

    24 & \methodnamez $-$ STE
    & $-12.02\;[-62.42,\,38.38]$
    & $54.41\;[36.32,\,72.50]$
    & $-52.85\;[-76.76,\,-28.94]$ \\

    24 & STE $-$ Soft
    & $6.20\;[-27.44,\,39.84]$
    & $-5.24\;[-19.20,\,8.71]$
    & $-2.04\;[-54.14,\,50.07]$ \\

    \midrule

    32 & \methodname $-$ \methodnamez
    & $21.84\;[3.92,\,39.77]$
    & $24.44\;[10.18,\,38.70]$
    & $4.09\;[0.54,\,7.64]$ \\

    32 & \methodnamez $-$ Soft
    & $-26.72\;[-36.70,\,-16.73]$
    & $40.71\;[21.60,\,59.82]$
    & $-56.90\;[-86.71,\,-27.08]$ \\

    32 & \methodnamez $-$ STE
    & $-15.47\;[-49.48,\,18.55]$
    & $32.22\;[15.17,\,49.27]$
    & $-47.15\;[-64.08,\,-30.23]$ \\

    32 & STE $-$ Soft
    & $-11.25\;[-49.87,\,27.37]$
    & $8.49\;[-10.49,\,27.47]$
    & $-9.75\;[-28.28,\,8.79]$ \\

    \bottomrule
  \end{tabular}
\end{table}

\FloatBarrier
\paragraph{Key takeaway.}
On the tested ZINC250K mixtures, \methodname produces more distinct novel connected molecules than STE or Soft, even without rarity.
Rarity further improves this yield but weakens the match to the selected feature distributions.
Equal update counts leave training work unmatched, which motivates the active time comparison below.

\FloatBarrier
\subsubsection{Matched active training budgets}
\label{app:zinc-time-budget}

\paragraph{Motivation.}
Equal update budgets can hide differences in the cost of a training update.
We therefore compare \methodname, \methodnamez, STE, and Soft under the same active training time on the same GPU.
This experiment asks whether the higher useful yield of \methodname remains when the gradient methods can complete more updates.
The comparison of \methodnamez with STE tests the update recipes under the same hard forward criterion. Comparing the two \methodname variants tests the additional rarity term.

\paragraph{Shared protocol.}
This experiment uses the full accepted ZINC250K data on a canvas of 38 nodes, including molecules with more than 32 heavy atoms.
All four methods use the same one-step generator with three hidden layers of width 768 and latent dimension 96.
They share the train/test split, molecular size prior, evaluation inputs, and initialization within each of the three seeds, 42, 43, and 44.
The methods run sequentially on one allocated GPU, with their order rotated across seeds.
The vocabulary is derived only from ZINC training molecules.
This differs from the joint ZINC/GuacaMol vocabulary used for the controlled size sweep, so the results are compared within this experiment.

\paragraph{Clock and checkpoint policy.}
We evaluate 10,000 proposals at 0, 1, 2, 5, 10, 20, and 30 minutes.
Each evaluation uses the last completed update within its budget, without retrospective checkpoint selection or a cap on update count.
Active time includes synchronized updates, batch preparation, device transfers, and archive updates.
It excludes initial setup, compilation and warmup, evaluation, and checkpoint I/O.
Strict VUN counts distinct valid, connected molecules absent from training, divided by all attempts, without repair by extracting the largest component.
Table~\ref{tab:zinc-time-budget-30min} reports the final measurements that support Figure~\ref{fig:molecular-training-evidence}.

\begin{table}[htbp]
  \centering
  \caption{Full ZINC250K after 30 minutes of active training.
    Values are mean $\pm$ sample SD across three seeds, with 10,000 proposals per seed.
    Quality metrics are percentages; updates are in thousands.}
  \label{tab:zinc-time-budget-30min}
  \setlength{\tabcolsep}{4pt}
  \begin{tabular}{@{}lrrrr@{}}
    \toprule
    Method & Full valid & Strict valid & Strict VUN & Updates ($10^3$)\\
    \midrule
    \methodname & $86.94\pm5.26$ & $79.69\pm8.47$ & $77.91\pm7.92$ & $25.84\pm0.29$\\
    \methodnamez & $71.53\pm4.95$ & $61.71\pm3.58$ & $53.73\pm3.70$ & $25.95\pm0.26$\\
    STE & $79.12\pm4.15$ & $57.26\pm11.27$ & $29.50\pm1.76$ & $145.73\pm2.39$\\
    Soft & $92.82\pm4.70$ & $66.81\pm25.46$ & $5.78\pm9.86$ & $156.58\pm3.30$\\
    \bottomrule
  \end{tabular}
\end{table}

\paragraph{Results.}
At 30 minutes, \methodname reaches $77.91\pm7.92$\% Strict VUN, and \methodnamez reaches $53.73\pm3.70$\%.
The corresponding values are $29.50\pm1.76$\% for STE and $5.78\pm9.86$\% for Soft in Table~\ref{tab:zinc-time-budget-30min}.
The ratios of mean yield relative to Soft are $13.47\times$ for \methodname and $9.29\times$ for \methodnamez.
The rarity term adds 24.19 percentage points to the mean yield.
Both \methodname variants have higher mean Strict VUN than both gradient baselines at all six positive budgets.
The aggregate data support this statement about means, without establishing the same ordering in every individual seed.

STE and Soft complete $5.64\times$ and $6.06\times$ as many updates as \methodname, respectively.
Soft has the highest mean full validity, $92.82\%$, but a low Strict VUN of $5.78\%$.
This gap shows why full validity alone is insufficient to describe useful output yield.
The supplied aggregate metrics do not separate the losses due to disconnection, duplication, and recovery of training molecules.

\FloatBarrier
\paragraph{Key takeaway.}
\methodname maintains a higher useful molecular yield under equal active training time, even though STE and Soft complete about six times as many updates.
The advantage remains when rarity is removed, and the rarity term increases it further.
This result concerns active training time on full ZINC250K.
It does not measure total job time, inference speed, or distributional fidelity.

\FloatBarrier
\subsection{GuacaMol}
\label{app:external-guacamol}

GuacaMol~\citep{brown2019guacamol} uses a curated subset of ChEMBL
to evaluate molecular generation. It includes a wider range of molecular
sizes and structures than QM9. Our full GuacaMol experiment supports up
to 88 heavy atoms and uses the supplied training, validation, and test
partitions. This setting tests whether \methodname can still produce a large
collection of distinct valid molecules as graph size increases.
Experiments on subsets defined by a smaller size limit are discussed
separately below.

\FloatBarrier
\subsubsection{Comparison with external generators}

Table~\ref{tab:guacamol-generators} reports the complete comparison with
published GuacaMol generators and the local CFM-ECLD pilot described in
Section~\ref{app:cfm-guacamol}. The \methodname experiment restricted to
$n\leq32$ is shown in a separate block because it uses a different
data distribution and a smaller graph canvas.
We report FCD as a distance
throughout this table, so lower values indicate a closer distributional
match. For published rows, we convert the rounded benchmark score
$s_{\mathrm{FCD}}$ using
$D_{\mathrm{FCD}}=-5\log s_{\mathrm{FCD}}$, following the definition
in~\citet{brown2019guacamol}. These converted distances are approximate.
For \methodname, we aggregate the recorded distance in each seed directly.

\begin{table}[htbp]
  \centering
  \caption{GuacaMol generation, including the local CFM-ECLD pilot
    after three epochs and 14,919 updates, using seed 42.
    The pilot is an early checkpoint, not a converged baseline, and is
    excluded from efficiency ratio and ranking claims.
    \methodname reports mean $\pm$ sample SD over three training seeds.
    VU is the percentage of raw attempts that yield distinct valid
    molecules, except for the proxies marked $^{a}$.
    FCD is a distance, with lower values preferred.
    Gray rows identify generators that use one network evaluation.
    Published results retain their source protocols; \methodname and the local
    CFM-ECLD pilot use our offline evaluation protocol.
    The final block reports the separate \methodname experiment on $n\leq32$.}
  \label{tab:guacamol-generators}
  \setlength{\tabcolsep}{3pt}
  \begin{tabularx}{\linewidth}{@{}>{\raggedright\arraybackslash}Xcrrrr@{}}
      \toprule
    Method & Steps & Valid (\%) $\uparrow$
    & VU (\%) $\uparrow$ & FCD $\downarrow$
    & VU/s $\uparrow$\\
    \midrule
    \multicolumn{6}{@{}l}{\textit{Autoregressive generation}}\\
    SMILES LSTM~\citep{brown2019guacamol}
    & $L$ tokens & 95.9 & $95.9^{a}$ & $\approx0.46$ & ---\\
    G2PT-base~\citep{chen2025g2pt}
    & $L$ tokens & 94.6 & $94.6^{a}$ & $\approx0.34$ & ---\\
    G2PT-large~\citep{chen2025g2pt}
    & $L$ tokens & 95.3 & $95.3^{a}$ & $\approx0.38$ & ---\\
    \addlinespace
    \multicolumn{6}{@{}l}{\textit{Iterative diffusion and flow}}\\
    DiGress~\citep{vignac2023digress}
    & 500 & 85.2 & $85.2^{a}$ & $\approx1.93$ & ---\\
    Cometh~\citep{siraudin2025cometh}
    & 500 & 94.4 & 94.4 & $\approx1.97$ & ---\\
    Cometh-PC~\citep{siraudin2025cometh}
    & 500 + PC & 98.9 & 98.9 & $\approx1.59$ & ---\\
    DeFoG~\citep{qin2025defog}
    & 50 & 91.7 & 91.7 & $\approx2.73$ & ---\\
    DeFoG~\citep{qin2025defog}
    & 500 & 99.0 & 99.0 & $\approx1.52$ & ---\\
    MELD~\citep{seo2026meld}
    & --- & 100.0 & $100.0^{a}$ & $\approx1.87$ & ---\\
    \addlinespace
    \multicolumn{6}{@{}l}{\textit{One-step generation}}\\
    \rowcolor{gray!10}
    \shortstack[l]{CFM-ECLD$^{\dagger}$\\(3 epochs)}
    & 1 & 19.67 & 19.67 & 33.36 & 28.27\\
    \rowcolor{gray!10}
    \textbf{\methodname (full)}
    & \textbf{1} & $80.80\pm10.36$ & $77.61\pm9.35$
    & $34.62\pm1.58$ & $686\pm76$\\
    \midrule
    \multicolumn{6}{@{}l}{\textit{Separate experiment on the $n\leq32$ subset}}\\
    \rowcolor{gray!10}
    \textbf{\methodname ($\boldsymbol{n\leq32}$)}
    & \textbf{1} & $87.71\pm3.54$ & $84.74\pm3.81$
    & $33.29\pm3.06$ & $2{,}442\pm63$\\
    \bottomrule
  \end{tabularx}
  \vspace{3pt}
  \begin{minipage}{\linewidth}
    $L$ denotes token sequence length; PC denotes predictor and corrector
    sampling, which adds computation. Iterative steps need not equal
    neural network evaluations.
    $^{a}$VU is approximated from reported validity and conditional
    uniqueness, using rounded source values. The two measurements may
    use separate sampling banks, so these proxies are not measured
    distinct yields from a shared proposal bank.
    $^{\dagger}$The CFM-ECLD pilot uses 10,000 raw proposals for validity
    and VU. Its FCD uses 10,000 valid molecules from a separate bank of
    60,000 attempts and the test reference. No across-seed SD is available
    for this single run. Full pilot results appear in Table~\ref{tab:cfm-guacamol-pilot}.
    \methodname uses three hidden layers of width 768, 240,000 training updates
    on full GuacaMol and 120,000 on $n\leq32$, with validation selection.
    Each seed generates 10,000 raw proposals
    for validity and VU. Its FCD uses 10,000 valid samples from a separate
    bank of 30,000 raw proposals and 10,000 test reference molecules;
    the $n\leq32$ reference is filtered to that subset.
    Full molecular identities allow disconnected
    valid outputs; no novelty requirement is applied to VU.
    \methodname VU/s is the mean and sample SD of per-seed rates on an A100 40GB,
    including generation, molecular conversion, and SMILES/CSV export,
    but excluding setup, warmup, and subsequent distribution evaluation.
    The CFM pilot rate includes conversion and CSV export on one A100;
    its sampler does not perform a separate warmup.
    These measurements do not establish a matched-hardware speedup.
    Estimated dense costs are reported separately in
    Table~\ref{tab:molecular-dense-audit}.
    A dash denotes an unavailable compatible measurement.
  \end{minipage}
\end{table}

\FloatBarrier
\paragraph{Key takeaway.}
\methodname produces $77.61\pm9.35\%$ distinct valid molecules from its proposal
budget with one network evaluation per attempt. Its measured full-data
rate is $686\pm76$ distinct valid molecules per second through CSV export.
The smaller $n\leq32$ experiment achieves higher yield and throughput,
but changes both the data distribution and the training budget.
The cost estimates in Table~\ref{tab:molecular-dense-audit} provide
additional computational context. The FCD of \methodname remains much
larger than the published reference values. These results show that
producing many valid candidates and matching the data distribution are
separate objectives. Differences in evaluation protocol prevent a
controlled ranking of distributional quality across all rows.

\FloatBarrier
\subsubsection{Controlled Molecular Size Sweep}
\label{app:guacamol-size-sweep}

\paragraph{Why this experiment.}
We repeat the size sweep on GuacaMol to test whether the result extends beyond the ZINC250K mixtures.
These cumulative subsets use the shared protocol in Section~\ref{app:ksweep} and remain separate from the full GuacaMol comparison.
The experiment asks whether useful yield persists as the task includes larger molecules and whether rarity gives a consistent benefit.

\paragraph{Results.}
Table~\ref{tab:ksweep-guacamol-all} reports all twelve configurations.
\methodname reaches mean Strict VUN of $82.12\%$, $81.53\%$, and $73.69\%$ at caps 16, 24, and 32.
Both \methodname variants remain above STE and Soft in mean Strict VUN at every cap.
At $k=32$, STE has higher strict validity than \methodname, $83.64\%$ versus $78.41\%$, but much lower Strict VUN, $15.02\%$ versus $73.69\%$.
The difference between these metrics combines repetitions and training-set overlap; it does not separate their individual contributions.

\begin{table}[htbp]
  \centering
  \caption{Complete cumulative GuacaMol sweep.
    Values are mean $\pm$ sample standard deviation across five seeds, with 30,000 updates per run.
    Full, Strict, Strict VUN, and Conn. are percentages; all use the complete attempt budget as the denominator.
    Lower $\mathrm{MMD}_{\rm cell}$ is better, with comparisons restricted to a fixed cap.}
  \label{tab:ksweep-guacamol-all}
  \setlength{\tabcolsep}{3pt}
  \begin{tabular}{@{}rlrrrrr@{}}
    \toprule
    $k$ & Method & Full $\uparrow$ & Strict $\uparrow$ & Strict VUN $\uparrow$ & Conn. $\uparrow$ & $\mathrm{MMD}_{\rm cell}\downarrow$ \\
    \midrule
    16 & \methodname & 94.57 $\pm$ 2.66 & 90.27 $\pm$ 2.34 & 82.12 $\pm$ 2.19 & 95.68 $\pm$ 1.43 & 0.0091 $\pm$ 0.0024 \\
    16 & \methodnamez & 86.71 $\pm$ 5.49 & 82.60 $\pm$ 6.82 & 66.01 $\pm$ 4.08 & 95.79 $\pm$ 1.59 & 0.0064 $\pm$ 0.0016 \\
    16 & \shortstack[l]{STE\\(feature-ST)} & 99.72 $\pm$ 0.38 & 97.88 $\pm$ 1.95 & 6.39 $\pm$ 1.38 & 98.16 $\pm$ 1.63 & 0.0364 $\pm$ 0.0189 \\
    16 & Soft & 97.75 $\pm$ 2.09 & 92.52 $\pm$ 4.41 & 7.74 $\pm$ 7.53 & 94.77 $\pm$ 2.62 & 0.0387 $\pm$ 0.0162 \\
    \midrule
    24 & \methodname & 92.80 $\pm$ 5.42 & 86.27 $\pm$ 6.47 & 81.53 $\pm$ 6.81 & 93.44 $\pm$ 1.54 & 0.0109 $\pm$ 0.0037 \\
    24 & \methodnamez & 74.42 $\pm$ 7.91 & 68.65 $\pm$ 7.93 & 58.68 $\pm$ 8.22 & 93.99 $\pm$ 3.17 & 0.0088 $\pm$ 0.0024 \\
    24 & \shortstack[l]{STE\\(feature-ST)} & 66.52 $\pm$ 32.18 & 62.37 $\pm$ 31.21 & 8.86 $\pm$ 6.35 & 91.06 $\pm$ 6.50 & 0.0620 $\pm$ 0.0281 \\
    24 & Soft & 98.97 $\pm$ 1.65 & 90.73 $\pm$ 5.57 & 7.95 $\pm$ 9.58 & 91.75 $\pm$ 4.05 & 0.0984 $\pm$ 0.0491 \\
    \midrule
    32 & \methodname & 84.28 $\pm$ 6.05 & 78.41 $\pm$ 5.79 & 73.69 $\pm$ 5.10 & 94.07 $\pm$ 1.57 & 0.0106 $\pm$ 0.0018 \\
    32 & \methodnamez & 64.98 $\pm$ 4.97 & 58.22 $\pm$ 7.81 & 48.87 $\pm$ 9.51 & 92.90 $\pm$ 4.81 & 0.0100 $\pm$ 0.0020 \\
    32 & \shortstack[l]{STE\\(feature-ST)} & 93.23 $\pm$ 3.12 & 83.64 $\pm$ 5.53 & 15.02 $\pm$ 7.85 & 90.21 $\pm$ 7.69 & 0.0736 $\pm$ 0.0037 \\
    32 & Soft & 86.95 $\pm$ 5.83 & 49.04 $\pm$ 34.28 & 17.25 $\pm$ 23.28 & 59.26 $\pm$ 32.59 & 0.0752 $\pm$ 0.0255 \\
    \bottomrule
  \end{tabular}
\end{table}

\paragraph{Paired comparisons.}
Table~\ref{tab:ksweep-contrasts} reports the complete paired contrasts.
\methodnamez exceeds STE in Strict VUN by 59.62, 49.82, and 33.86 percentage points at caps 16, 24, and 32, with all three intervals above zero.
The rarity gains are 16.12, 22.84, and 24.81 points, also with positive intervals.
The comparison of \methodnamez with Soft at $k=32$ has a positive mean difference but an interval that includes zero.
We therefore distinguish a consistent ordering of means from evidence for each individual contrast.

\begin{table}[htbp]
  \centering
  \caption{GuacaMol paired contrasts: mean difference with its unadjusted paired 95\% confidence interval in brackets, over five seeds. Each difference is the first method minus the second. Validity and VUN use the strict definition, with differences reported in percentage points; MMD differences are multiplied by $10^3$. Positive validity or yield differences and negative MMD differences favor the first method.
  }
  \label{tab:ksweep-contrasts}

  \setlength{\tabcolsep}{3pt}
  \renewcommand{\arraystretch}{1.15}

\begin{tabular}{@{}rlrrr@{}}
    \toprule
    $k$ & Contrast
    & $\Delta$Validity
    & $\Delta$VUN
    & $10^3\Delta\mathrm{MMD}_{\rm cell}$ \\
    \midrule

    16 & \methodname $-$ \methodnamez
    & \shortstack[r]{$7.68$\\$[-0.68,\,16.03]$}
    & \shortstack[r]{$16.12$\\$[10.91,\,21.32]$}
    & \shortstack[r]{$2.63$\\$[0.49,\,4.76]$} \\

    16 & \methodnamez $-$ Soft
    & \shortstack[r]{$-9.92$\\$[-15.07,\,-4.78]$}
    & \shortstack[r]{$58.27$\\$[50.18,\,66.35]$}
    & \shortstack[r]{$-32.20$\\$[-52.40,\,-12.01]$} \\

    16 & \methodnamez $-$ STE
    & \shortstack[r]{$-15.28$\\$[-22.69,\,-7.88]$}
    & \shortstack[r]{$59.62$\\$[55.76,\,63.48]$}
    & \shortstack[r]{$-29.99$\\$[-52.55,\,-7.43]$} \\

    16 & STE $-$ Soft
    & \shortstack[r]{$5.36$\\$[0.12,\,10.61]$}
    & \shortstack[r]{$-1.35$\\$[-9.21,\,6.50]$}
    & \shortstack[r]{$-2.21$\\$[-25.74,\,21.31]$} \\

    \midrule

    24 & \methodname $-$ \methodnamez
    & \shortstack[r]{$17.62$\\$[5.49,\,29.75]$}
    & \shortstack[r]{$22.84$\\$[12.98,\,32.70]$}
    & \shortstack[r]{$2.03$\\$[-2.51,\,6.58]$} \\

    24 & \methodnamez $-$ Soft
    & \shortstack[r]{$-22.08$\\$[-35.91,\,-8.26]$}
    & \shortstack[r]{$50.73$\\$[33.29,\,68.17]$}
    & \shortstack[r]{$-89.57$\\$[-152.35,\,-26.79]$} \\

    24 & \methodnamez $-$ STE
    & \shortstack[r]{$6.28$\\$[-39.97,\,52.53]$}
    & \shortstack[r]{$49.82$\\$[33.59,\,66.06]$}
    & \shortstack[r]{$-53.20$\\$[-86.49,\,-19.91]$} \\

    24 & STE $-$ Soft
    & \shortstack[r]{$-28.37$\\$[-70.36,\,13.63]$}
    & \shortstack[r]{$0.91$\\$[-11.69,\,13.50]$}
    & \shortstack[r]{$-36.37$\\$[-102.75,\,30.01]$} \\

    \midrule

    32 & \methodname $-$ \methodnamez
    & \shortstack[r]{$20.19$\\$[13.34,\,27.04]$}
    & \shortstack[r]{$24.81$\\$[14.86,\,34.76]$}
    & \shortstack[r]{$0.63$\\$[-3.49,\,4.75]$} \\

    32 & \methodnamez $-$ Soft
    & \shortstack[r]{$9.17$\\$[-33.97,\,52.31]$}
    & \shortstack[r]{$31.62$\\$[-2.70,\,65.94]$}
    & \shortstack[r]{$-65.22$\\$[-94.99,\,-35.45]$} \\

    32 & \methodnamez $-$ STE
    & \shortstack[r]{$-25.42$\\$[-37.09,\,-13.75]$}
    & \shortstack[r]{$33.86$\\$[28.24,\,39.47]$}
    & \shortstack[r]{$-63.64$\\$[-69.10,\,-58.19]$} \\

    32 & STE $-$ Soft
    & \shortstack[r]{$34.59$\\$[-8.87,\,78.05]$}
    & \shortstack[r]{$-2.24$\\$[-36.98,\,32.51]$}
    & \shortstack[r]{$-1.58$\\$[-36.53,\,33.37]$} \\

    \bottomrule
  \end{tabular}
\end{table}

\FloatBarrier
\paragraph{Key takeaway.}
The GuacaMol sweep supports the same pattern as ZINC250K: \methodname has higher distinct novel yield across the tested mixtures, and rarity provides an additional gain.
The gain does not imply better performance on every metric, since \methodnamez has the lowest mean feature MMD at each cap and some validity comparisons favor the gradient methods.

\FloatBarrier
\subsubsection{A CFM pilot on full GuacaMol}
\label{app:cfm-guacamol}

\paragraph{Motivation.}
We selected CFM-ECLD~\citep{roos2026categoricalflowmaps}
for its strong one-step graph generation results on both QM9 and
ZINC250K. Kernel-gradient drifting~\citep{esteban2026kernelgradient}
reports molecular experiments only on QM9, using SMILES.
We therefore adapted CFM-ECLD to full GuacaMol, absent from its
published evaluation, and ran a limited pilot to measure training
cost and sample quality early in training on larger graphs.

\paragraph{Setup.}
The adaptation uses the public CFM model and loss implementation, with
12 transformer layers, node, edge, and global hidden dimensions of
256, 64, and 64, respectively, and eight attention heads. The global
feedforward width is explicitly set to 128. This explicitly forwards a configuration argument that the public constructor call omitted, leaving its default at 2048. The adaptation therefore differs from that code path and is not a reproduction of an authenticated published checkpoint. The data use 58 atom states and a fixed canvas
of 88 nodes. Fixed padding preserves the global feature statistics of
the public model. Training uses AdamW with learning rate $10^{-4}$,
microbatches of 8, an effective batch size of 256, and EMA decay 0.999.

We train one seed, 42, for three epochs, corresponding to 14,919
optimizer updates. The cosine learning rate schedule remains configured
for 84 epochs. Evaluation uses the EMA parameters at the final pilot
checkpoint, rather than a checkpoint selected on validation. Sampling
uses one network evaluation per attempt and batches of 16. We apply
the same offline molecular evaluation protocol as for \methodname, with no
repair, replacement of invalid samples, or selection of the largest
fragment.

\paragraph{Results.}
Table~\ref{tab:cfm-guacamol-pilot} reports the completed evaluation.
The primary bank contains 10,000 raw attempts. A separate fixed bank
of 60,000 attempts produces 11,694 valid molecules, which supplies the
10,000 generated samples required for FCD and KL evaluation. FCD uses
the test reference, and KL uses the training reference. The table uses
the completed evaluation with the larger bank, with no additional training.

\begin{table}[htbp]
  \centering
  \caption{CFM-ECLD pilot on full GuacaMol, seed 42, after three epochs
    and 14,919 updates. This is a single pilot checkpoint without an
    estimate of variation across training seeds. Full validity allows
    disconnected molecules; strict validity also requires connectivity.}
  \label{tab:cfm-guacamol-pilot}
  \setlength{\tabcolsep}{6pt}
  \begin{tabular}{@{}lr@{}}
    \toprule
    Metric & CFM-ECLD pilot\\
    \midrule
    Full validity (\%) & 19.67\\
    Strict validity (\%) & 4.64\\
    Uniqueness among full valid molecules (\%) & 100.00\\
    Full VU (\%) & 19.67\\
    Strict VU (\%) & 4.64\\
    Full VUN (\%) & 19.66\\
    Strict VUN (\%) & 4.63\\
    FCD to test reference $\downarrow$ & 33.3561\\
    GuacaMol KL score $\uparrow$ & 0.6814\\
    \midrule
    Graph generation, 10,000 attempts (s) & 64.67\\
    Generation through SMILES/CSV, 10,000 attempts (s) & 69.59\\
    Full VU/s through SMILES/CSV & 28.27\\
    Estimated dense MFLOP per proposal & 13,312.918\\
    $E_{\mathrm{VU}}$ & 14.8\\
    \bottomrule
  \end{tabular}
  \vspace{3pt}
  \begin{minipage}{\linewidth}
    Validity, VU, and VUN use the primary bank. VUN additionally requires
    absence from the training reference. FCD and KL use the separate
    distribution bank. Timings come from the first evaluation of the
    same checkpoint and primary bank on one A100, include synchronized
    generation and transfer to CPU, and separate chemistry conversion
    and CSV export. They exclude subsequent FCD and KL computation;
    the sampler does not run a separate warmup. Dense cost is reconstructed
    from the local 58-state architecture at 88 nodes and one network
    evaluation, using the partial dense operation convention of this
    appendix. It is not a measured hardware cost.
  \end{minipage}
\end{table}

\paragraph{Training cost.}
The logs record approximately 19 to 24 seconds per optimizer update
on one A100-SXM-64GB. With 4,973 updates per epoch, extrapolating to
the configured 84 epochs gives
\[
T_{84}\approx
\frac{84\times4{,}973\times[19,24]}{3{,}600}
= [2{,}205,2{,}785]\ \text{hours}.
\]
This is the source of the rounded 2,200 to 2,800-hour projection in
Table~\ref{tab:molecular-training-time}. It assumes unchanged update
throughput and excludes setup, validation, checkpointing, final
evaluation, restarts, and queue time. The range describes observed
throughput scenarios, not a confidence interval. The exact cumulative
duration of the pilot across restarts was not recovered. The projection
does not establish that 84 epochs would be needed or sufficient for
convergence.

\FloatBarrier
\paragraph{Key takeaway.}
The pilot makes the missing GuacaMol comparison concrete, but it remains
an early checkpoint. All 1,967 valid molecules in the primary bank are
distinct, yet only 464 are connected. Thus, approximately 76.4\% of
valid outputs contain multiple fragments, and strict VUN is only
4.63\%. The run exposes a substantial training cost and a connectivity
limitation at this budget. It does not establish the quality of a fully
trained CFM model and is excluded from efficiency ratio and ranking
claims against \methodname.

\FloatBarrier
\subsection{Inference Cost and Training Time}
\label{app:external-accounting}

\FloatBarrier
\subsubsection{Estimated Dense Inference Work}

\paragraph{Why this analysis.}
The main comparison asks how many distinct valid candidates a generator produces relative to its inference work.
A single sampling step does not imply the same cost across architectures.
We therefore count the dense projections executed during generation and report measured training time separately.
The resulting estimates support Figure~\ref{fig:molecular-validity-efficiency} in the main text.

\paragraph{Definition.}
Let $N$ be the number of raw proposals, $U_N$ the number of distinct valid outputs, and $W_N$ the estimated dense work for that bank.
With $Y_{\mathrm{VU}}=U_N/N$ and average work $C_{\mathrm{dense}}=W_N/N$ per proposal,
\begin{equation}
 E_{\mathrm{VU}}=10^{12}\frac{U_N}{W_N}
 =10^{12}\frac{Y_{\mathrm{VU}}}{C_{\mathrm{dense}}}.
 \label{eq:molecular-dense-efficiency}
\end{equation}
Yield is a fraction in this equation.
Invalid and duplicate proposals consume work and remain in the denominator.
Because distinct yield depends on the number of attempts, this normalization does not predict how many distinct molecules a new run would find under a common compute budget.

\paragraph{Counting convention.}
We count two operations per multiply-accumulate in each executed learned weight projection.
This includes linear layers, recurrent projections, and convolutional projections where applicable, summed over all network evaluations.
For a linear map from width $a$ to width $b$ applied at $m$ positions, the count is $2mab$.
Convolutions use the corresponding kernel dimensions and output positions.
Bias additions, activations, normalization, attention products and reductions, graph aggregation, feature construction, sampling, molecular conversion, deduplication, and distribution evaluation are excluded.
For MoFlow, the count includes inverse $1\times1$ projections, but excludes weight-inverse setup and elementwise flow operations.
These are partial inference operation counts, not complete FLOPs or runtime measurements.

\paragraph{\methodname architectures.}
For an \methodname multilayer perceptron with input width $d$, $L$ hidden layers of width $H$, and output width $o$,
\begin{equation}
 C_{\mathrm{dense}}=2\left[dH+(L-1)H^2+Ho\right].
\end{equation}
The output width includes both atom and bond heads.
For QM9, $d=105$, $L=3$, $H=768$, and $o=252$, which gives $2{,}907{,}648$ dense operations per proposal.
Using the mean VU of $0.7821$ gives $E_{\mathrm{VU}}=268{,}980$.
The corresponding ZINC250K and full GuacaMol counts are $12{,}630{,}016$ and $33{,}991{,}680$ operations.
The corresponding $(d,L,H,o)$ values are $(129,3,1024,3990)$ and $(178,3,768,20416)$.
Their output canvases have 38 and 88 nodes, respectively; the size encoding widths are 33 and 82.
Table~\ref{tab:molecular-dense-audit} preserves the available efficiency estimates from the main molecular comparison and gives the counts supporting the \methodname versus CFM ratios.

\begin{table}[htbp]
 \centering
 \caption{Dense inference work and efficiency for the molecular comparison. VU is a percentage of raw proposals, and MFLOP denotes $10^6$ counted operations per proposal, including all network evaluations. \methodname counts use local architectures; published CFM counts use reconstructed architectures. The GuacaMol CFM count uses the local three-epoch pilot. Estimates marked $^{u}$ retain the values used in the main figure, but lack a complete recovered layer count. Dashes indicate unavailable counts. The pilot is excluded from ranking and ratio claims.}
 \label{tab:molecular-dense-audit}
 \setlength{\tabcolsep}{5pt}
 \begin{tabular}{@{}llrrr@{}}
 \toprule
 Dataset & Method & VU (\%) & Dense MFLOP & $E_{\mathrm{VU}}$ \\
 \midrule
 QM9 & GraphARM & $\approx86.30$ & --- & --- \\
 & G2PT-base & $\approx95.83$ & --- & $67.1^{u}$ \\
 & CatFlow (100) & $\approx99.76$ & --- & $169^{u}$ \\
 & DeFoG (500) & $\approx95.72$ & 86,234.240 & 11.1 \\
 & MELD & $\approx96.49$ & --- & $1.72^{u}$ \\
 & MolGAN & $\approx10.20$ & 0.714 & 142,945 \\
 & MoFlow & 99.20 & 47.024 & 21,096 \\
 & CFM-CSD & $\approx82.62$ & 167.689 & 4,927 \\
 & CFM-ECLD & $\approx85.45$ & 167.689 & 5,096 \\
 & \textbf{\methodname} & $78.21\pm0.84$ & 2.908 & \textbf{268,980} \\
 \midrule
 ZINC250K & GraphARM & $\approx87.75$ & --- & --- \\
 & CatFlow (100) & $\approx99.21$ & --- & $14.3^{u}$ \\
 & DeFoG (500) & $\approx99.21$ & 1,359,621.888 & 0.730 \\
 & MELD & $\approx99.99$ & --- & $0.265^{u}$ \\
 & MoFlow & 99.99 & --- & $666^{u}$ \\
 & CFM-CSD & $\approx73.43$ & 2,694.691 & 272 \\
 & CFM-ECLD & $\approx93.03$ & 2,694.691 & 345 \\
 & \textbf{\methodname} & $90.66\pm4.00$ & 12.630 & \textbf{71,781} \\
 \midrule
 GuacaMol & SMILES LSTM & $\approx95.9$ & --- & $190^{u}$ \\
 & G2PT-base & $\approx94.6$ & --- & $9.07^{u}$ \\
 & DeFoG (500) & 99.0 & 6,685,054.720 & 0.148 \\
 & MELD & $\approx100.0$ & --- & --- \\
 & CFM-ECLD, 3 epochs & 19.67 & 13,312.918 & 14.8 \\
 & \textbf{\methodname} & $77.61\pm9.35$ & 33.992 & \textbf{22,831} \\
 \bottomrule
 \end{tabular}
\end{table}

\paragraph{External architectures and sampling.}
Graph Transformer estimates use reference canvases of 9, 38, and 88 nodes for QM9, ZINC250K, and GuacaMol.
These describe a maximum-canvas workload, not an observed average over the sampled molecular sizes or padded batches.
Autoregressive costs must include the full generated sequence, and iterative costs must include all network calls, including any final denoiser.
We do not replace these counts with parameter count multiplied by nominal sampling steps.
The DeFoG reconstruction gives $86{,}234.240$, $1{,}359{,}621.888$, and $6{,}685{,}054.720$ MFLOP for 500 evaluations on the three datasets.
The audited QM9 counts for MolGAN and MoFlow are $0.714$ and $47.024$ MFLOP, respectively.
External configurations are associated with published quality, but the precise checkpoints that produced those values were not authenticated.
Estimates marked $^{u}$ in Table~\ref{tab:molecular-dense-audit} retain the values used in the main figure; a complete layer-level reconstruction was not recovered for those entries.
They do not support the \methodname versus CFM ratios.
The approximate QM9 VU for DeFoG uses the published relaxed validity of 99.4\% and uniqueness of 96.3\%; Table~\ref{tab:qm9-generators} retains the standard validity of 99.3\%.
We therefore do not multiply columns based on different validity rules.

\paragraph{Sensitivity of the CFM reconstruction.}
The QM9 reconstruction uses a global feed-forward width of 128 based on the paper description.
Using the constructor default of 2048 changes the parameter count from 5.925 to 8.154 million and dense work from 167.689 to 172.112 MFLOP.
Holding published quality fixed, the corresponding CFM-ECLD efficiencies are approximately 5,096 and 4,965.
The \methodname efficiency ratio is then 52.8 or 54.2, respectively.
Thus, the reported advantage is stable across these two assumptions, although neither scenario authenticates the original checkpoint.
The ZINC250K CFM reconstruction gives the approximately $208\times$ ratio reported in the main text.
The local GuacaMol configuration and its limited training budget are described in Section~\ref{app:cfm-guacamol}.

\FloatBarrier
\paragraph{Key takeaway.}
The compact \methodname architecture requires much less dense projection work per proposal than the reconstructed CFM architectures.
This yields the reported efficiency advantage despite lower VU and worse FCD on QM9 and ZINC250K.
The ratios describe the stated partial operation count and source protocols, not measured generation speedups.

\FloatBarrier
\subsubsection{Training Time Audit}

\paragraph{Why this comparison.}
Low inference cost can require substantial training work.
Table~\ref{tab:molecular-training-time} therefore reports the training durations used in the resource panel of the main text.
It provides context across \methodname, DeFoG, Cometh, and CFM while preserving differences in timing scope and hardware.

\begin{table}[htbp]
 \centering
 \caption{Training time in hours for the methods discussed in the main text. \methodname measures synchronized optimizer updates; external values use the reported elapsed training durations. The GuacaMol CFM entry is a constant-throughput projection to 84 epochs from the local three-epoch pilot. It is not a measured full training duration. A dash denotes unavailable evidence.}
 \label{tab:molecular-training-time}
 \begin{tabular}{@{}lrrr@{}}
 \toprule
 Method & QM9 & ZINC250K & GuacaMol \\
 \midrule
 \methodname & $0.2999\pm0.0012$ & 2.0983 & $14.8468\pm0.0697$ \\
 DeFoG~\citep{qin2025defog} & 6.5 & 14 & 141 \\
 Cometh~\citep{siraudin2025cometh} & 6 & --- & $\approx96$ \\
 CFM~\citep{roos2026categoricalflowmaps} & $\approx48$ & $\approx144$ & $\approx2{,}200$ to $2{,}800$ \\
 \bottomrule
 \end{tabular}
\end{table}

\paragraph{\methodname measurements.}
QM9 uses the $3\times768$ generator trained for 120,000 updates, and full GuacaMol uses the same depth and width for 240,000 updates.
Both report the mean and sample standard deviation across seeds 42, 43, and 44.
The ZINC250K duration is the available seed-42 measurement for the $3\times1024$ model at 120,000 updates, so no across-seed deviation is reported.
These timers cover the completed training budget even when validation selects an earlier checkpoint.
They exclude setup, compilation, warmup, validation, checkpointing, and final evaluation.
The historical ZINC timer also excludes input preparation and perturbation-scale calculation.
The recovered timing records do not certify the GPU model for every \methodname campaign, so we do not assign a common GPU specification to these durations.

\paragraph{Published measurements.}
DeFoG reports 6.5 hours for QM9 without explicit hydrogen nodes, 14 hours for ZINC250K, and 141 hours for GuacaMol on one A100-SXM4 80GB.
Cometh reports 6 hours on QM9 using a V100 or A10, without resolving which device produced that duration, and about four days on GuacaMol using two A100 GPUs.
The latter corresponds to about 192 allocated GPU-hours.
CFM reports about two days per QM9 model and six days per ZINC250K model with RTX A6000 hardware, with the GPU count unspecified.
Its reported schedules use 5,000 and 500 epochs, respectively.
These measurements follow the training-time sections of the original papers.
The local GuacaMol projection follows the calculation in Section~\ref{app:cfm-guacamol}; it is not a published CFM result.

\FloatBarrier
\paragraph{Key takeaway.}
\methodname has the lowest recorded training durations among these entries, but the table does not establish a speedup under matched hardware, timing scope, or final quality.
The controlled ZINC250K experiment in Section~\ref{app:zinc-time-budget} supplies the direct comparison under a common active training budget.

\section{RQ2: Scaffold Completion}
\label{app:scaffold}

Scaffold completion tests whether a compact generator can rapidly supply
molecular candidates that preserve a given core and satisfy screening
criteria. We train MoLeR~\citep{maziarz2022moler} from scratch on the same
ZINC scaffold split as \methodname. Both methods receive the same held-out
scaffold queries and are evaluated with the same screening rules.
Architectures, training objectives, and training budgets differ.

We evaluate two resource constraints. The primary experiment matches the
time available for generation and screening on each query. A complementary
experiment matches the number of proposals and measures the resulting
distinct candidate libraries. Generation-to-SMILES timing is recorded
separately for that fixed proposal bank.

\paragraph{Summary.}
At matched time, \methodname screens $7.39\times$ as many proposals on
average and returns $10.05\times$ as many filter-passing proposals. MoLeR covers more queries and more often
returns the higher-QED passing candidate. We report outcomes over all
100 queries in each replicate, including failures of either or both methods.
At a fixed budget of 100,000 proposals, \methodname returns
$20{,}245\pm755$ distinct filter-passing completions versus
$16{,}684\pm79$ for MoLeR, with an observed $44.3\times$ reduction in
generation-to-SMILES time under that separate protocol.
\Cref{tab:scaffold-matched-time-results,tab:scaffold-zinc-summary}
report the two evaluations.

\subsection{Shared Training and Scaffold-Only Inference}
\label{app:scaffold-protocol}

The benchmark contains 100 scaffold queries from the held-out ZINC partition.
Molecules are grouped by canonical, nonisomeric Bemis-Murcko scaffold identity before splitting.
The frozen split contains 68,296 training, 4,017 validation, and 8,035 test scaffold identities.
The query scaffold identities are therefore absent from the training partition of both methods.
The decoders receive only the scaffold, without the corresponding held-out molecule or its target size.

\methodname uses the scaffold-conditioned rooted forest decoder described in Appendix~\ref{app:architectures}.
Training combines the Drift distribution matching objective, structural constraints, and a rarity reward, as specified in Appendices~\ref{app:objectives} and~\ref{app:rarity}.
During training, the condition includes the number of atoms added to the scaffold in each example.
At inference, this number is sampled from a prior estimated only from the training partition.
The selected checkpoints are reused without further training.

The three \methodname checkpoints use training seeds 42, 43, and 44.
Seeds 43 and 44 complete 30,000 updates and use checkpoints selected at updates 21,000 and 20,000, respectively.
Seed 42 uses the archived checkpoint from the original experiment.
The original validation procedure supplies a target size and first prioritizes feasibility.
It then combines completion success, distinct completion yield, and scaffold coverage, with penalties for feature discrepancy and constraint violations.
The selected checkpoints for seeds 43 and 44 do not meet the original validation feasibility thresholds.
They are retained under the same selection rule and are not reselected using the screening results reported here.

The size prior records the number of heavy atoms added to each training scaffold.
It is conditioned on scaffold size when that size occurs in training, with the pooled distribution of training additions as a fallback.
Only positive additions that yield a supported total size of at most 38 heavy atoms are retained, and the distribution is renormalized.
The procedure uses neither the held-out target size nor clipping based on that target.
Core atom states and their ordering are reconstructed from the public scaffold SMILES.
Randomness is assigned deterministically from the generation seed, query identity, and attempt index.
Unsupported cores and failed conversions remain counted attempts.
Failed proposals are not replaced, and conversion uses the complete graph without chemical repair or rescue of the largest component.

MoLeR is trained from scratch, without external pretrained weights.
The training runner exports the original ZINC training and validation partitions using the frozen scaffold assignments.
It does not create a new random split.
The benchmark test partition is excluded from preprocessing and checkpoint selection.
The test input required by the MoLeR pipeline is populated with a copy of the validation partition, rather than the benchmark test molecules.
Training uses the MoLeR VAE procedure with validation-based checkpoint selection.

We evaluate one fixed MoLeR checkpoint with sampling seeds 41, 42, and 43 and beam size 1.
The decoder receives the scaffold and a latent sampled from the prior, without a requested final size.
The initial diagnostic run with sampling seed 123 is excluded from the reported means.
Thus, the \methodname results describe three trained checkpoints, while the MoLeR results describe sampling variability for one trained checkpoint.
The shared scaffold queries provide the pairing for comparisons within each query.
The repeats do not constitute paired training experiments.

\subsection{Shared Usability and Screening Criteria}
\label{app:scaffold-screening}

A usable completion must pass chemical sanitization of the complete molecule, be connected, and contain the supplied scaffold as a substructure.
An induced exact scaffold match is recorded separately and is not required by the primary criterion.
The completion does not need to match the heavy atom count of the held-out reference molecule.
For distinct-library metrics, usable outputs are deduplicated by canonical, nonisomeric SMILES within each query.
The same molecular identity can contribute to different queries, but repeated identities within a query count only once.

Screening requires QED $\geq0.7$, SA $\leq4$, and no additional PAINS alert types relative to the scaffold.
QED measures drug likeness \citep{bickerton2012qed}, and SA estimates synthetic accessibility \citep{ertl2009sa}.
We use the RDKit implementation of QED, the RDKit Contrib implementation of SA, and the RDKit PAINS filter catalog.
If $A(x)$ and $A(c)$ are the sets of alert types for completion $x$ and scaffold $c$, the PAINS rule is
\[
    A(x)\setminus A(c)=\emptyset.
\]
Alert types already present in the core are permitted.
The rule therefore does not require every completion to be free of PAINS alerts.
These screening criteria are applied during evaluation and are not included in the training objective of \methodname or used to select the reported checkpoints.

Novelty relative to training molecules is not required by the screening criterion.

\subsection{Generation and Screening at Matched Time}
\label{app:scaffold-matched-time}

This primary application experiment tests whether faster generation allows more candidates to be evaluated within a short time budget; its aggregate results appear in the main text.
The experiment uses the same 100 scaffold queries and the same selected checkpoints, with no further training.
The three replicates combine \methodname training seeds 42, 43, and 44 with MoLeR sampling seeds 41, 42, and 43, respectively.
Both methods run on the same NVIDIA A100-SXM4-40GB GPU, with models loaded and warmed up before timing.

For each scaffold and replicate, MoLeR generates one batch of 256 proposals with beam size 1 and six decoding workers.
The measured time $T$ includes query preparation, latent sampling, decoding to SMILES, communication with the scoring process, screening, and selection of the eligible candidate with the highest QED.
This measured interval defines the time budget for \methodname on the same query.
The budget therefore varies across queries and replicates.

\methodname generates batches of 256 graphs and converts them directly into RDKit molecules for screening.
It does not serialize every candidate to SMILES inside the timed loop.
Only candidates fully scored and selected by time $T$ are eligible.
Both methods apply the same usability checks and QED, SA, and scaffold relative PAINS filters.
A query with no eligible candidate is recorded as a failure.
The first eligible candidate attaining the maximum QED is retained in a tie, and comparisons between methods use a tolerance of $10^{-12}$.
Model loading, warmup, result I/O, and export of the selected \methodname molecule to SMILES are outside the timed interval.
The largest recorded overrun of a final \methodname operation is 0.005853 seconds, and results completed after $T$ are excluded.
This protocol measures candidate selection within a common time budget.
Its output representation and timing boundaries differ from the generation-to-SMILES measurement in Section~\ref{app:scaffold-timing}.

\Cref{tab:scaffold-matched-time-results} reports all three replicates over
the same 100-query panel. The mean of the three replicate-level mean
deadlines is $6.085\pm0.103$ seconds per query. The corresponding numbers
of proposals fully screened per query are $1{,}892.7\pm63.7$ for
\methodname and $256.0\pm0.0$ for MoLeR. The ratio of these means is
$7.39$.

The saved per-query records also contain the number of distinct proposals that
pass all screening criteria before the deadline. These counts are
$547.3\pm15.9$ per query for \methodname and $54.5\pm0.4$ for MoLeR.
Their ratio of means is $10.05$.

Let $P^{(m)}_{sr}$ and $B^{(m)}_{sr}$ denote passing and screened proposal
counts for method $m$, scaffold $s$, and replicate $r$. Each replicate
contains $S=100$ queries. Its mean passing count per query and passing
fraction are
\[
 \overline P^{(m)}_r=\frac{1}{S}\sum_{s=1}^{S}P^{(m)}_{sr},
 \qquad
 \pi^{(m)}_r=\frac{\sum_{s=1}^{S}P^{(m)}_{sr}}
 {\sum_{s=1}^{S}B^{(m)}_{sr}}.
\]
We report means and sample SDs of these replicate summaries.
Passing fractions are $28.95\pm1.58$\% for \methodname and
$21.27\pm0.14$\% for MoLeR. These are averages of the three replicate
fractions, rather than an average of per-query percentages.

All aggregate uncertainties in the main matched-time table are sample
standard deviations of the three replicate summaries, not standard
deviations across individual queries or standard errors. The means and
SDs are computed from the full per-query records before rounding. The same
scaffolds are reused across replicates, and the two methods' seeds
represent different sources of variation.

To account for every query, we distinguish outcomes where both methods return passing candidates, and one has a higher maximum QED, where only one method returns a passing candidate, and where neither succeeds. Equal-QED outcomes are recorded separately; none occur here. These
categories partition the 100 queries in each replicate. Coverage is a separate summary that counts all queries where the specified method
returns a passing candidate.

\begin{table}[tbp]
  \centering
  \caption{Generation and screening at matched time. Each column reports one replicate on the same 100 scaffold queries; headers identify the \methodname training seed and MoLeR sampling seed. Screened counts are means per query; passing counts are totals over the panel. Passing fractions use all screened proposals. The outcome block partitions all 100 queries. Higher-QED counts require passing candidates from both methods; exclusive successes and failures of both methods remain explicit.}
  \label{tab:scaffold-matched-time-results}
  \setlength{\tabcolsep}{4pt}
  \renewcommand{\arraystretch}{1.12}
  \begin{tabular*}{\linewidth}{@{\extracolsep{\fill}}lrrr@{}}
    \toprule
    Metric & 42 / 41 & 43 / 42 & 44 / 43 \\
    \midrule
    Mean deadline per query (s) & 6.176 & 6.105 & 5.974 \\
    Mean proposals screened, \methodname & 1,822.94 & 1,947.70 & 1,907.35 \\
    Mean proposals screened, MoLeR & 256.0 & 256.0 & 256.0 \\
    Filter-passing proposals, \methodname & 55,152 & 52,975 & 56,064 \\
    Filter-passing proposals, MoLeR & 5,434 & 5,485 & 5,417 \\
    Passing fraction (\%), \methodname & 30.25 & 27.20 & 29.39 \\
    Passing fraction (\%), MoLeR & 21.23 & 21.43 & 21.16 \\
    Queries with a hit, \methodname & 79 & 79 & 79 \\
    Queries with a hit, MoLeR & 85 & 85 & 85 \\
    \midrule
    \multicolumn{4}{@{}l}{\textit{Query outcomes (counts per 100 queries)}} \\
    Both pass: higher QED, \methodname & 23 & 34 & 33 \\
    Both pass: higher QED, MoLeR & 55 & 44 & 46 \\
    Equal QED with hits from both & 0 & 0 & 0 \\
    Hit from \methodname only & 1 & 1 & 0 \\
    Hit from MoLeR only & 7 & 7 & 6 \\
    No hit from either method & 14 & 14 & 15 \\
    \bottomrule
  \end{tabular*}
\end{table}

For the comparison of passing counts across scaffolds, we first average
the three replicate counts separately for each query and method.
\methodname has a higher mean passing count on 73 queries, MoLeR on 14,
and 13 queries tie. These are counts over the same 100-query panel.
Within individual replicates, the numbers of \methodname wins, MoLeR wins, and
ties are respectively $(71,14,15)$, $(72,14,14)$, and $(71,14,15)$.

Averaging the outcome counts across replicates, both methods pass, and
\methodname has a higher maximum QED on $30.0\pm6.1$ queries per 100;
MoLeR has higher maximum QED on $48.3\pm5.9$. Only \methodname passes
on $0.7\pm0.6$ queries, only MoLeR passes on $6.7\pm0.6$, and neither
passes on $14.3\pm0.6$. These five observed outcomes account for the
entire query panel, without assigning an artificial QED to failures or
combining feasibility and quality into a new score.

MoLeR covers 85 queries in every replicate, compared with 79 for
\methodname. Its higher coverage and more frequent higher-QED outcomes
coexist with the greater proposal throughput of \methodname. These are
descriptive comparisons on repeated evaluations of the same 100
scaffolds, not significance tests on 300 independent queries.

\Cref{fig:scaffold-matched-time} shows the selected candidates for scaffold 23 in all three replicates.
The scaffold was specified before examining the timed results.
The third replicate corresponds to the 5.13-second example in the main text.
Within this interval, \methodname generates 3,072 graphs and completes evaluation of 2,830 proposals, of which 469 pass screening.
MoLeR generates and evaluates 256 proposals, of which 102 pass screening.
The highest QED is 0.919 for \methodname and 0.946 for MoLeR, with corresponding SA scores of 3.77 and 3.32.
Generated graphs that remain unevaluated at the deadline do not enter the selection.

\begin{figure*}[t]
  \centering
  \includegraphics[width=0.9\textwidth]{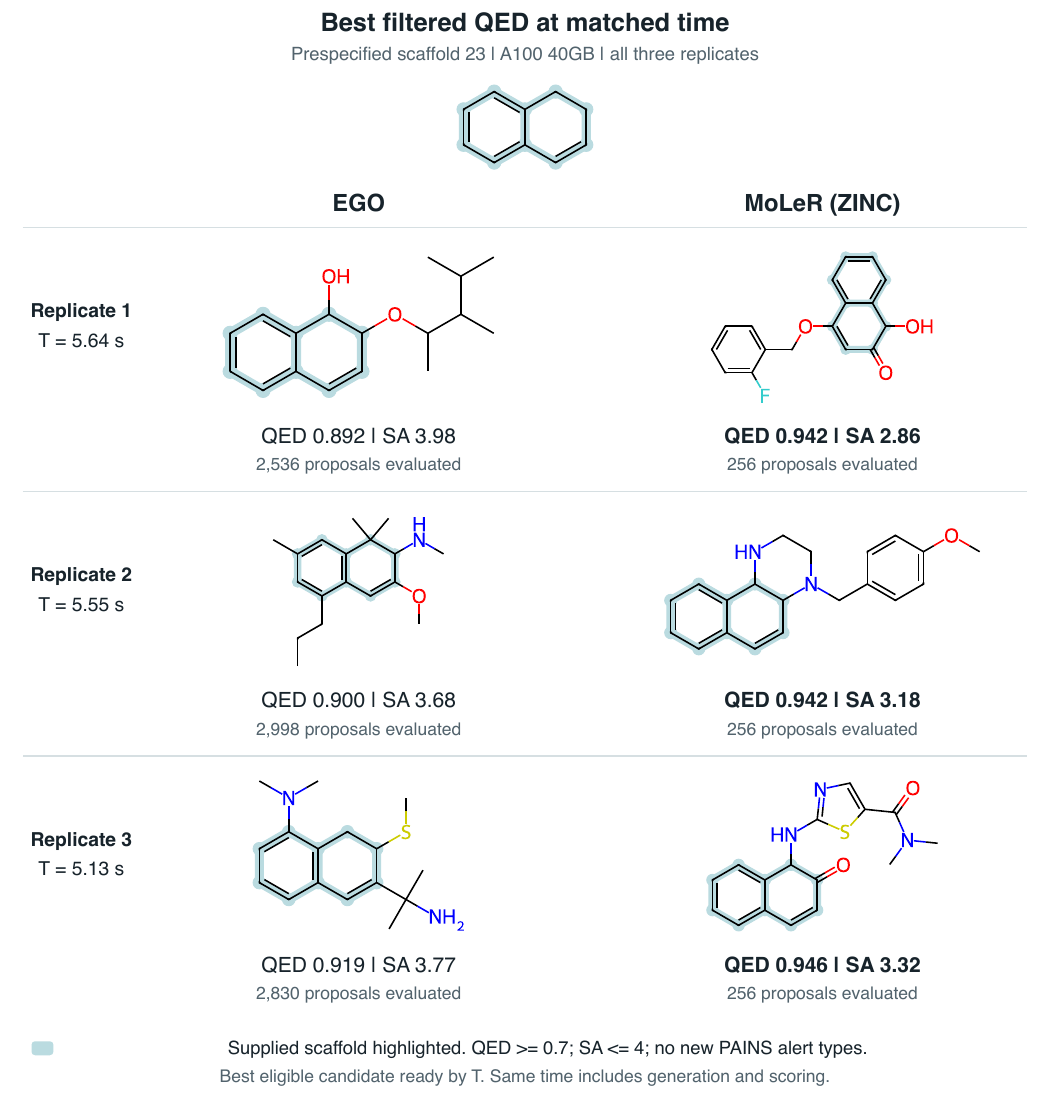}
  \caption{Highest QED completions that pass screening for the prespecified scaffold 23 at matched time. Each row shows one replicate, using \methodname training seeds 42, 43, and 44 and sampling seeds 41, 42, and 43 from the same MoLeR checkpoint trained on ZINC. The common budgets are 5.64, 5.55, and 5.13 seconds. The supplied scaffold is highlighted. \Cref{tab:scaffold-matched-time-results} reports results over all 100 queries.}
  \label{fig:scaffold-matched-time}
\end{figure*}

The matched-time experiment establishes higher screening throughput and
more filter-passing proposals for \methodname under the stated protocol. The next experiment
assesses distinct candidate libraries at a fixed proposal budget.
The screening filters are heuristic and do not establish experimental
usefulness.

\subsection{Distinct Candidate Libraries at a Fixed Proposal Budget}
\label{app:scaffold-results}

This complementary evaluation uses the same selected checkpoints and
100 held-out scaffold queries. Each method makes 1,000 proposals per
query, giving 100,000 proposals per run. There is no shared time deadline
in this experiment. Screening uses the criteria in
Section~\ref{app:scaffold-screening}, and distinct counts remove repeated
identities within each query.

Let $U_{sr}$ be the number of distinct usable completions for scaffold $s$ in run $r$, and let $F_{sr}$ be the number that also pass screening.
The distinct usable yield is $100\sum_s U_{sr}/100{,}000$.
The distinct filtered count is $\sum_s F_{sr}$, and the conditional fraction that passes screening is $100\sum_s F_{sr}/\sum_s U_{sr}$.
Coverage counts queries with $F_{sr}>0$.
The usable completion rate includes repeated proposals, while distinct usable yield removes them.
Invalid and failed attempts remain in all proposal-based denominators.

\Cref{tab:scaffold-zinc-summary} reports the aggregate results.
MoLeR produces more distinct usable completions, with a yield of $85.80\%$ compared with $66.78\%$ for \methodname.
After screening, the ordering reverses.
The fraction of distinct usable completions that passes the filters is approximately $30.3\%$ for \methodname and $19.4\%$ for MoLeR.
Consequently, \methodname returns $20{,}245$ distinct filtered completions on average, compared with $16{,}684$ for MoLeR.
The ratio of these mean counts is $1.21$.

\begin{table}[tbp]
  \centering
  \caption{Scaffold completion at the same proposal budget. Each run uses 100 queries and 1,000 proposals per query. Values are the mean $\pm$ sample SD over three \methodname training seeds and three sampling seeds from one MoLeR checkpoint trained from scratch on ZINC. Distinct-library counts are deduplicated within each query. The filter passing fraction is conditional on distinct usable completions. Generation time excludes screening.}
  \label{tab:scaffold-zinc-summary}
  \setlength{\tabcolsep}{4pt}
  \renewcommand{\arraystretch}{1.12}
  \begin{tabular*}{\linewidth}{@{\extracolsep{\fill}}lrr@{}}
    \toprule
    Metric & \methodname & MoLeR \\
    \midrule
    Usable completions (\%) & $77.67\pm2.57$ & $97.58\pm0.09$ \\
    Distinct usable yield (\%) & $66.78\pm1.15$ & $85.80\pm0.12$ \\
    Filter passing fraction (\%) & $30.3\pm1.1$ & $19.4\pm0.1$ \\
    Distinct filtered completions & $20{,}245\pm755$ & $16{,}684\pm79$ \\
    Scaffold coverage (/100) & $77.7\pm1.5$ & $86.3\pm0.6$ \\
    \midrule
    Generation to SMILES (s) & $41.82\pm8.39$ & $1{,}851.82\pm8.05$ \\
    \bottomrule
  \end{tabular*}
\end{table}

\Cref{tab:scaffold-zinc-seeds} gives the exact filtered counts for the individual runs.
Means are unweighted across runs, and sample SD uses denominator $n-1$.
Conditional fractions are calculated within each run before averaging.
Outputs from different runs are not pooled into a larger library.

\begin{table}[tbp]
  \centering
  \caption{Distinct completions that pass screening in each run. Each count is the sum of distinct counts within 100 queries at 1,000 proposals per query. MoLeR sampling seeds share one checkpoint trained with seed 42.}
  \label{tab:scaffold-zinc-seeds}
  \setlength{\tabcolsep}{5pt}
  \renewcommand{\arraystretch}{1.12}
  \begin{tabular*}{\linewidth}{@{\extracolsep{\fill}}llrr@{}}
    \toprule
    Method & Seed type & Seed & Distinct filtered count \\
    \midrule
    \methodname & Training & 42 & $19{,}527$ \\
    \methodname & Training & 43 & $20{,}175$ \\
    \methodname & Training & 44 & $21{,}033$ \\
    \midrule
    MoLeR & Sampling & 41 & $16{,}763$ \\
    MoLeR & Sampling & 42 & $16{,}683$ \\
    MoLeR & Sampling & 43 & $16{,}606$ \\
    \bottomrule
  \end{tabular*}
\end{table}

For the query comparison, we first average the distinct filtered counts across the three runs separately for each scaffold and method.
\methodname has a higher mean count on 51 scaffolds, with 13 ties and 36 MoLeR wins.
Its larger total library therefore coexists with lower coverage, at $77.7\pm1.5$ queries compared with $86.3\pm0.6$ for MoLeR.
These counts describe the observed query comparisons and are not significance tests.
The reported SD values also describe different sources of variation for the two methods.
They do not estimate uncertainty over independently trained MoLeR models.

\Cref{fig:scaffold_completions} shows examples of distinct completions around the same supplied core.
\Cref{fig:scaffold_large_extensions} illustrates completions with larger additions to the core.
Both figures retain qualitative examples from the earlier comparison with externally pretrained MoLeR.
The quantitative results in \Cref{tab:scaffold-zinc-summary} use the new MoLeR checkpoint trained from scratch on ZINC.

\begin{figure*}[t]
  \centering
  \includegraphics[width=\textwidth]{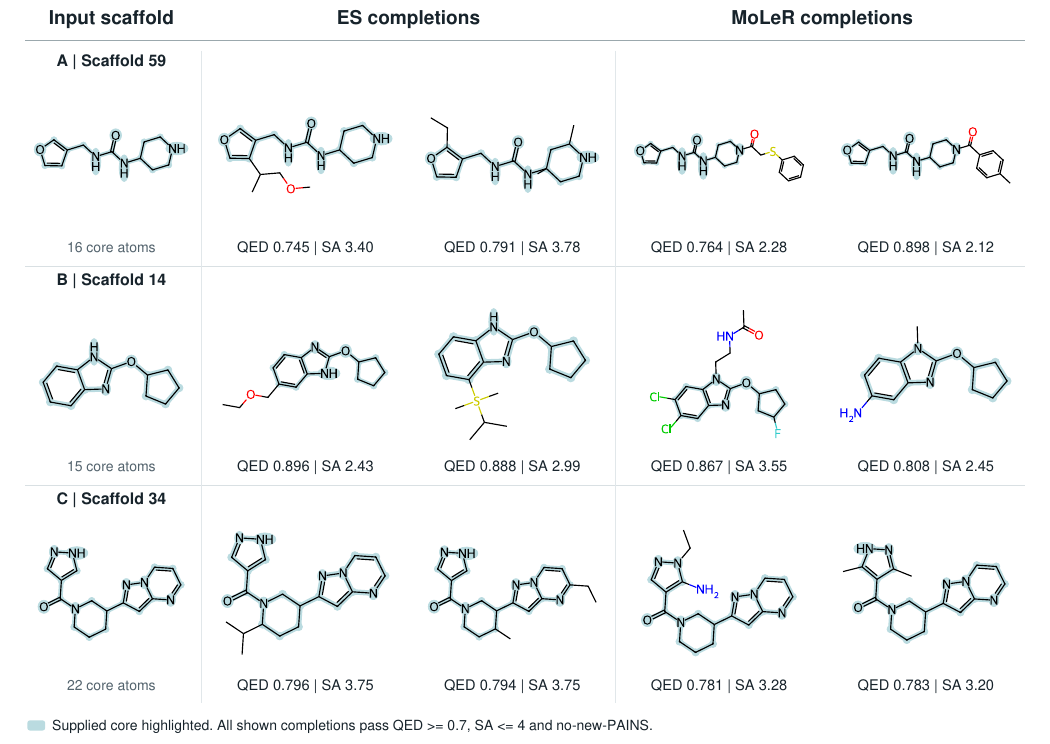}
  \caption{Examples of scaffold completions from \methodname with training seed 43 and externally pretrained MoLeR with sampling seed 42. The supplied core is highlighted. All displayed molecules pass the screening criteria. Three scaffolds were selected by a fixed hash ordering among queries with at least two distinct passing completions from both methods. The first two distinct passing completions in proposal order are shown for each method and scaffold.}
  \label{fig:scaffold_completions}
\end{figure*}

\begin{figure*}[t]
  \centering
  \includegraphics[width=\textwidth]{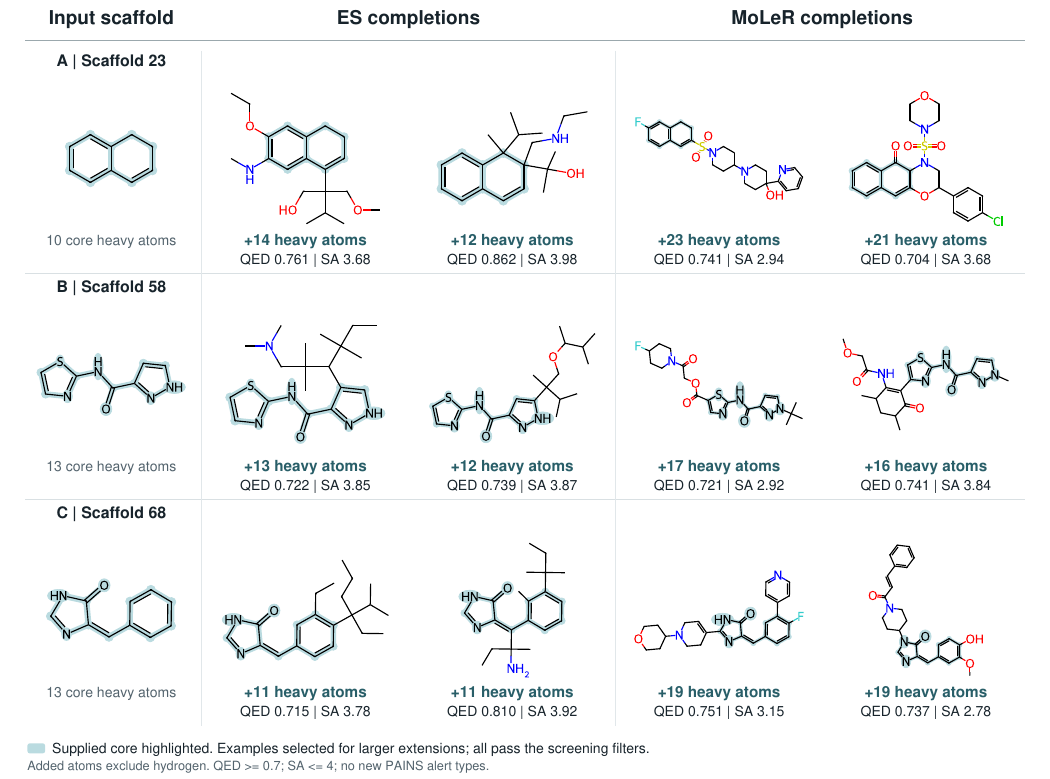}
  \caption{Selected examples of larger scaffold extensions from \methodname with training seed 44 and externally pretrained MoLeR with sampling seed 42. The supplied core is highlighted, and all displayed molecules pass screening. Scaffolds were ranked by the second-largest addition of heavy atoms from \methodname among queries with at least two distinct passing completions from both methods. The three selected scaffolds show the two largest passing extensions from each method. These examples illustrate larger extensions rather than typical outputs.}
  \label{fig:scaffold_large_extensions}
\end{figure*}

\subsection{Generation to SMILES Time}
\label{app:scaffold-timing}

The timing row in \Cref{tab:scaffold-zinc-summary} measures generation of the full proposal bank and conversion to SMILES.
For \methodname, the interval includes size sampling, synchronized graph generation, and conversion of the complete graphs.
It excludes setup, compilation, warmup, file I/O, and subsequent screening.
For MoLeR, the interval includes latent sampling and decoding to SMILES, including the first decoding call.
It excludes model loading, file I/O, and subsequent screening.

The recorded \methodname times are 32.140, 46.343, and 46.973 seconds for training seeds 42, 43, and 44.
Their mean is $41.82\pm8.39$ seconds, compared with $1{,}851.82\pm8.05$ seconds for the three MoLeR sampling runs.
The ratio of mean times is
\[
    \frac{1{,}851.82}{41.8187}\approx44.3.
\]
This ratio describes the recorded generation runs under the stated timing boundaries.
It is not the mean of ratios between seeds and does not measure the complete screening pipeline.
The experiment in Section~\ref{app:scaffold-matched-time} separately matches the available time for generation and screening on the same GPU.

\section{RQ3: Beyond Molecular Generation}
\label{app:rq3}

We test the same one-step generator framework on neural architectures and
binary images. NAS measures the yield of distinct high-performing candidates;
MNIST tests distribution matching when graph feasibility is no longer a
constraint. These experiments help identify where the training approach is
useful and where gradient-based alternatives perform better.

\subsection{NAS-Bench-101}
\label{app:nas}

\noindent\textbf{Summary.}
NAS-Bench-101~\citep{ying2019nasbench101} tests whether the candidate-library
result extends beyond chemistry. We train on high-performing architectures
and measure how many distinct held-out elite architectures each method
returns within 10,000 proposals. Validity, mean accuracy, and feature
distance reveal the trade-offs involved in producing a larger library.

\noindent\textbf{Takeaway.}
\methodname returns more distinct held-out elites than STE or Soft
(Table~\ref{tab:nas}). Soft achieves higher validity and mean accuracy,
and closer feature-distribution matching. The result supports \methodname when the
aim is to obtain more qualifying alternatives for subsequent comparison.

\paragraph{Shared protocol and fitness.}
We compare \methodname, the straight-through estimator (STE), and
Soft using seeds 42--44 and 15,000 updates per run. The generator,
initialization, sampling streams, and checkpoint-selection rule are shared
within each seed. Each generator maps 128-dimensional noise through three
512-unit hidden layers to 21 binary edge decisions and five operation
choices on a seven-vertex graph. Pruning removes vertices outside
input-to-output paths.

The fitness in Eq.~\eqref{eq:app-nas-fitness} combines negative Energy
with a structural penalty. Energy compares edge and operation choices,
degrees, and path structure with those of training elites. The penalty,
with fixed weight 0.5, discourages missing input-to-output paths and more
than nine edges after pruning. Its purpose is to enforce the benchmark's
feasibility rules while learning the structure of high-performing designs.
There is no rarity reward or explicit duplicate penalty; Energy itself
rewards spread among generated samples.

\methodname evaluates this fitness on hard graphs. STE uses the same hard forward
values with derivatives from relaxed features; Soft evaluates the relaxed
objective. Each objective evaluation uses 64 generated graphs and 192
training-elite references; \methodname evaluates a population of 192 candidates.
Appendix~\ref{app:method-nas} defines the features, normalization,
fitness terms, optimizer settings, and common selector.

\paragraph{Data split and checkpoint selection.}
We use the benchmark's 108-epoch records and split architectures 80/10/10
with seed 2026. An architecture is \emph{elite} when its benchmark
validation accuracy is at or above the 90th percentile in the architecture
training partition. This threshold is fixed for all three partitions.
Training elites supply the target distribution; validation-partition
elites and a fixed noise bank support checkpoint selection. Initialization
is not eligible for selection.

The architecture partition and benchmark accuracy split serve different
purposes. A \emph{held-out test elite} is an architecture in the test
partition that qualifies using its benchmark \emph{validation accuracy}.
Its benchmark \emph{test accuracy} is a separate final metric.
Test-partition architectures are excluded from the training and selection
reference sets, and their benchmark labels are used only for final
evaluation. Benchmark test accuracy is never used for training or
checkpoint selection. Feature normalization and kernel bandwidth are
fitted only on training elites.

\paragraph{Final metrics and interpretation.}
Each selected model produces 10,000 proposals. Held-out elite yield
divides the number of distinct generated test elites by all
10,000 attempts, including invalid graphs and duplicates. Valid-and-unique
yield in Table~\ref{tab:nas} likewise divides distinct valid architectures
by all attempts. Architectures are identified by the benchmark's canonical
hash after pruning, so equivalent graph encodings do not count as different
designs. All valid final proposals are found in the benchmark lookup.
Mean test accuracy averages over these valid proposals, retaining
duplicates.

Feature distance is the squared maximum mean discrepancy (MMD$^2$)
between raw concatenated graph features from up to 1,500 valid proposals
and 1,500 test-partition elites. Its radial basis function kernel has a
bandwidth fixed from training elites. It measures agreement with elite
structure, separately from accuracy and distinct candidate yield.

The held-out elite yields in Table~\ref{tab:nas} correspond to
$221.7\pm7.5$ distinct test elites for \methodname, compared with $138.0\pm28.9$
for STE and $143.7\pm18.6$ for Soft.
Among valid proposals, the fractions qualifying as elite across \emph{all}
architecture partitions are $42.81\pm2.97$\%, $60.27\pm2.93$\%, and
$64.35\pm1.96$\%, respectively. These fractions include repeated
architectures. Together, the metrics show why a higher concentration of
elites need not produce a larger collection of distinct held-out elites.
All $\pm$ values are sample standard deviations across the three runs.

\paragraph{Training work and timing.}
The runs match update counts, not objective evaluations or training time.
Over 15,000 updates, \methodname evaluates 184.32 million candidate graphs plus
0.96 million unperturbed diagnostic graphs. STE and Soft each evaluate
0.96 million candidate graphs. Each method draws 2.88 million reference
examples. Table~\ref{tab:nas-cost} separates optimizer-update time,
compilation, and checkpoint validation; initialization and final
evaluation are excluded.

\begin{table}[htbp]
  \centering
  \caption{NAS training time in seconds, mean $\pm$ sample standard
  deviation across seeds 42--44. Columns measure separate stages of the
  15,000-update runs; none alone is total runtime.}
  \label{tab:nas-cost}
  \begin{tabular}{@{}lrrr@{}}
    \toprule
    Method & Update time & Compilation & Validation \\
    \midrule
    \methodname & $46.12\pm0.59$ & $48.41\pm0.75$ & $39.58\pm1.91$ \\
    STE & $34.06\pm0.38$ & $54.18\pm0.88$ & $30.84\pm1.01$ \\
    Soft & $34.37\pm0.65$ & $54.51\pm1.18$ & $16.36\pm1.28$ \\
    \bottomrule
  \end{tabular}
\end{table}

\subsection{Binarized MNIST Control}
\label{app:mnist}

\noindent\textbf{Summary.}
Binarized MNIST tests whether the NAS result extends to image generation
when every sampled output is already a valid binary array. We compare
image-distribution distances and predicted digit coverage to assess what
the methods learn beyond producing distinct outputs. The generator learns
from image features without digit labels or classifier feedback.

\noindent\textbf{Takeaway.}
Soft achieves better distribution matching and broader predicted digit
coverage than \methodname in every paired seed
(Tables~\ref{tab:mnist-summary-energy} and~\ref{tab:mnist-extra}).
\methodname produces almost no exact duplicates, but small pixel differences do
not establish semantic diversity. Its NAS candidate-yield advantage
therefore does not imply better generation on every discrete task.

\paragraph{Shared protocol and fitness.}
Images are $28\times28$ and binarized at intensity 128 on the 0--255 scale.
With split seed 2026, we reserve 500 images per digit for validation,
leaving 55,000 training and 5,000 validation images. The official test
set is used only after checkpoint selection.

All methods use the same three-layer, 768-unit generator, with
96-dimensional noise and 784 pixel logits. One forward pass and parallel
Bernoulli sampling produce each binary image. Within each seed, the
initial parameters, latent noise, pixel-sampling noise, and reference
samples are shared. Each run completes 30,000 updates with 96 generated
images and 192 references per objective evaluation; \methodname uses 192 candidates.

The fitness is the negative average Energy over six feature families:
pixels, two pooling scales, row and column projections, neighboring-pixel
differences, and global image statistics. These features compare fine
detail and larger shapes without requiring a trained classifier.
Normalization is fixed from training images, and each family's Energy is
divided by the square root of its dimension so that a larger feature vector
does not dominate solely through its size. There are no extra penalties
or rewards. Table~\ref{tab:method-mnist-features} lists the features, and
Eq.~\eqref{eq:app-mnist-fitness} defines the complete fitness.

\methodname uses hard images and mean-centered fitness without rank shaping.
STE preserves the hard forward Energy while using relaxed derivatives;
Soft evaluates relaxed images. Both gradient methods use temperature
0.75, and all final evaluations use binary samples.
Appendix~\ref{app:method-mnist} gives the remaining training settings.

\paragraph{Checkpoint selection.}
Every 500 updates, the same selector scores 2,000 proposals from a fixed
noise bank. It combines exact image uniqueness with agreement between
generated and validation feature means and variances, as defined in
Eq.~\eqref{eq:app-mnist-selection-distance} and the accompanying selection
rule. This rewards variety and distribution matching without classifier
feedback. Initialization is excluded. Digit labels are used only to
stratify the split and train the separate evaluator.

\paragraph{Independent evaluation.}
One frozen LeNet-style classifier is shared across all nine generator runs.
It is trained on the 55,000 model-training images and labels using seed
2026, 2,500 updates, batch size 256, and learning rate $10^{-3}$.
Its accuracy on real test images is 98.98\%; this is not an accuracy
measurement for generated images.

Distribution distances use 2,000 generated and 2,000 real test images.
LeNet Fr\'echet distance compares means and covariances of raw
128-dimensional classifier embeddings. Kernel inception distance (KID)
is computed here on \emph{LeNet} embeddings, using the polynomial kernel
$k(a,b)=(a^\top b/128+1)^3$, where $a,b\in\mathbb R^{128}$ are
standardized embeddings. Thus, these scores do not use Inception features.
Table~\ref{tab:mnist-summary-energy} reports Fr\'echet distance alongside
predicted digit coverage.

\begin{table}[t]
  \centering
  \caption{Controlled Energy comparison on binarized MNIST, mean $\pm$
  sample standard deviation over seeds 42--44. Distances use 2,000
  generated and 2,000 test images. Coverage uses all 10,000 generated
  images, with at least 50 predictions required to count a digit class.}
  \label{tab:mnist-summary-energy}
  \begin{tabular}{@{}lrr@{}}
    \toprule
    Method & LeNet Fr\'echet $\downarrow$
    & Predicted classes / 10 $\uparrow$ \\
    \midrule
    \methodname   & $599.8\pm47.1$ & $6.33\pm0.58$ \\
    STE  & $522.4\pm95.6$ & $7.67\pm0.58$ \\
    Soft & $\mathbf{420.6\pm89.1}$ & $\mathbf{8.33\pm0.58}$ \\
    \bottomrule
  \end{tabular}
\end{table}

We also compute MMD$^2$ using a radial basis function (RBF) kernel, both
on standardized LeNet embeddings and on the 49 raw means from
nonoverlapping $4\times4$ image blocks. The latter provides an image
distance independent of the classifier. For feature vectors $a$ and $b$
from either representation, the kernel is
$k(a,b)=\exp[-\|a-b\|_2^2/(2\beta)]$.

For each representation, $\beta$ is the median positive squared pairwise
distance on a fixed subset of 2,000 training images (seed 2028), floored
at $10^{-8}$. Embedding means and standard deviations are fitted on the
same subset, with a standard-deviation floor of 0.05. KID and RBF-MMD$^2$
use unbiased estimators. All metric settings are fixed before test
evaluation.

\paragraph{Distribution matching and digit coverage.}
Table~\ref{tab:mnist-extra} reports complementary distances and class-balance
diagnostics. Class metrics use the most likely LeNet label for each of
10,000 generated images, without confidence filtering. A digit is covered
when at least 50 predictions belong to it. Total variation (TV) is half
the sum of absolute differences between predicted class frequencies and
real test-label frequencies. Lower TV indicates closer agreement.
Class entropy is Shannon entropy divided by $\log(10)$, so one denotes
uniform predicted frequencies across all ten digits.

\begin{table}[htbp]
  \centering
  \caption{Additional MNIST diagnostics, mean $\pm$ sample standard
  deviation across seeds 42--44. LeNet KID uses classifier embeddings;
  pooled-image MMD$^2$ uses raw block means. Class metrics use predicted
  labels on 10,000 images.}
  \label{tab:mnist-extra}
  \begin{tabular}{@{}lrrr@{}}
    \toprule
    Metric & \methodname & STE & Soft \\
    \midrule
    LeNet KID $\downarrow$ & $0.4517\pm0.0490$ & $0.3868\pm0.0853$ & $0.3069\pm0.0718$ \\
    Pooled-image MMD$^2$ $\downarrow$ & $0.0414\pm0.0008$ & $0.0235\pm0.0094$ & $0.0145\pm0.0062$ \\
    Class-balance TV $\downarrow$ & $0.545\pm0.020$ & $0.463\pm0.045$ & $0.401\pm0.099$ \\
    Normalized class entropy $\uparrow$ & $0.628\pm0.035$ & $0.738\pm0.044$ & $0.754\pm0.108$ \\
    \bottomrule
  \end{tabular}
\end{table}

Soft improves on \methodname in Fr\'echet distance, KID, both RBF-MMD$^2$ measures,
coverage, entropy, and class-balance TV in each paired seed. STE also
improves mean Fr\'echet distance and coverage over \methodname
(Table~\ref{tab:mnist-summary-energy}). Coverage remains incomplete for
all methods (Table~\ref{tab:mnist-selection}).

\methodname's exact image uniqueness is $99.98\pm0.03$\%, despite its lower
predicted digit coverage. No generated image exactly matches a training
image in any run, but this pixel-level test cannot rule out copies with
small changes. Figure~\ref{fig:mnist-samples} presents sample grids for
the three methods; the quantitative comparison uses the full evaluation
sets described above.

\paragraph{Selected checkpoints and cost.}
Table~\ref{tab:mnist-selection} records the selected update for each seed.
Some STE and Soft checkpoints are selected early, but every run still
completes 30,000 updates. The selected iteration therefore differs from
the completed training budget.

\begin{table}[htbp]
  \centering
  \caption{MNIST checkpoint selection and predicted digit coverage.
  Entries follow seed order 42, 43, 44. Every run completes 30,000 updates;
  coverage is measured on 10,000 final images.}
  \label{tab:mnist-selection}
  \begin{tabular}{@{}lrr@{}}
    \toprule
    Method & Selected update & Predicted classes / 10 \\
    \midrule
    \methodname & 29,000 / 30,000 / 30,000 & 6 / 6 / 7 \\
    STE & 500 / 30,000 / 500 & 8 / 8 / 7 \\
    Soft & 1,500 / 28,000 / 1,000 & 8 / 9 / 8 \\
    \bottomrule
  \end{tabular}
\end{table}

\methodname spends $118.60\pm3.50$ seconds on optimizer updates, compared with
$67.27\pm0.23$ for STE and $67.70\pm0.15$ for Soft. These times exclude
compilation, checkpoint validation, evaluator training, and final
evaluation. \methodname evaluates 552.96 million candidate images plus 2.88 million
unperturbed diagnostic images; each gradient method evaluates 2.88 million
candidate images. All methods draw 5.76 million reference images.

Generating 10,000 binary images takes about 29--30 milliseconds for each
method. This includes sampling and transfer to host arrays, but excludes
warmup, classifier evaluation, metrics, and file writing. All three methods
retain the same one-step inference procedure; their differences concern
training cost and the distribution learned.

\begin{figure}[t]
  \centering
  \includegraphics[width=0.5\linewidth]{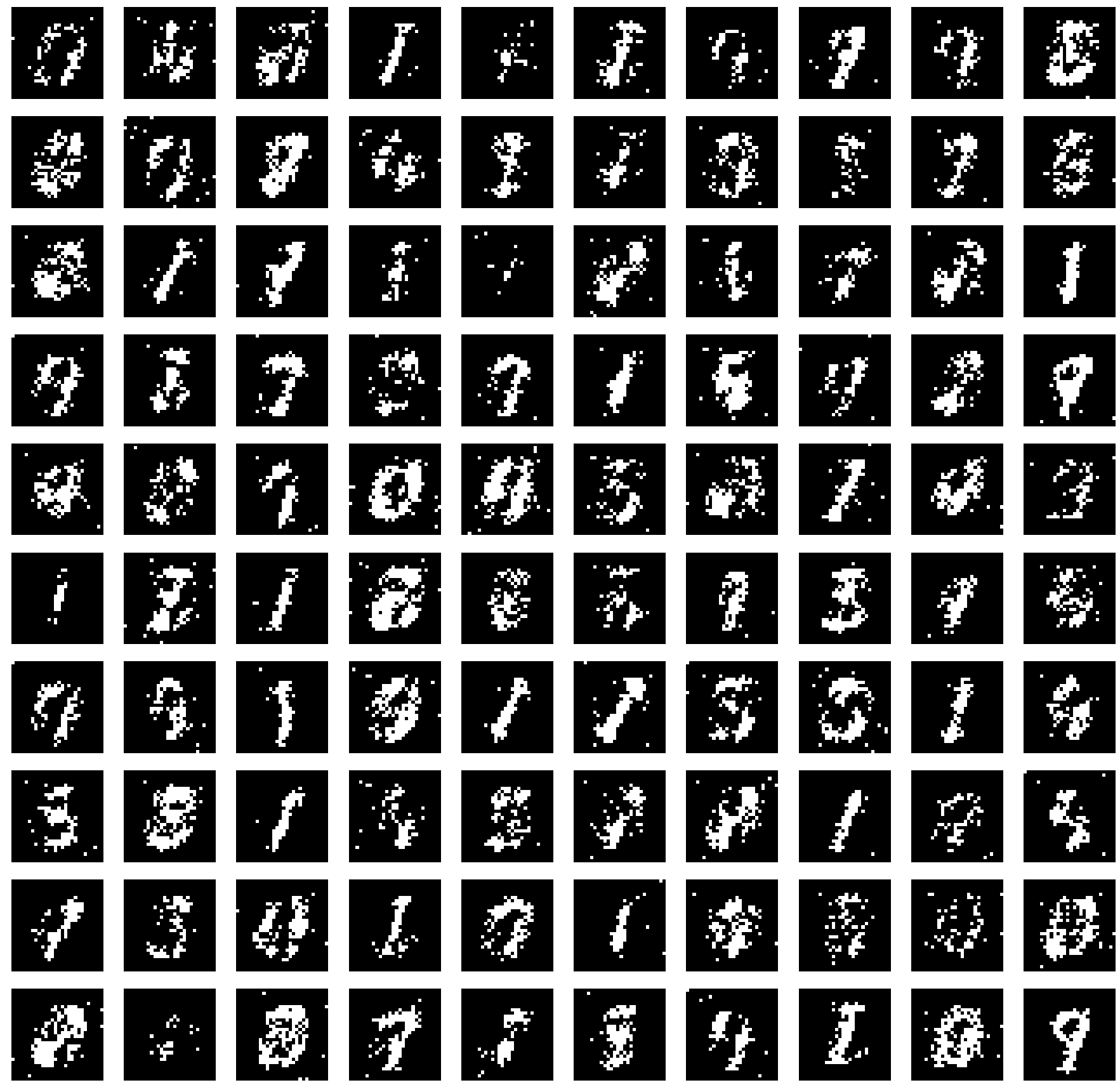}
  \par{\small \methodname}\medskip

  \includegraphics[width=0.5\linewidth]{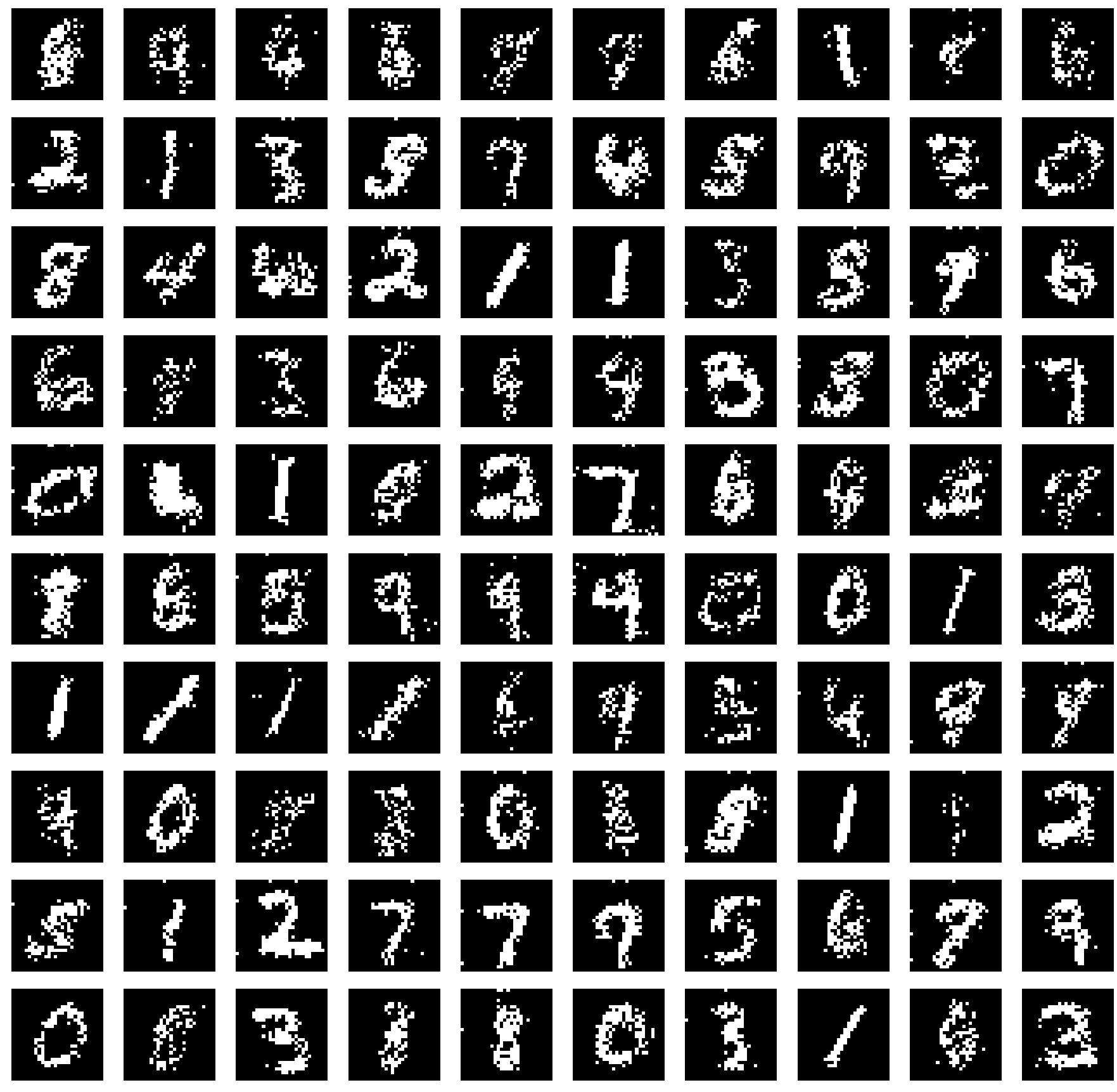}
  \par{\small STE}\medskip

  \includegraphics[width=0.5\linewidth]{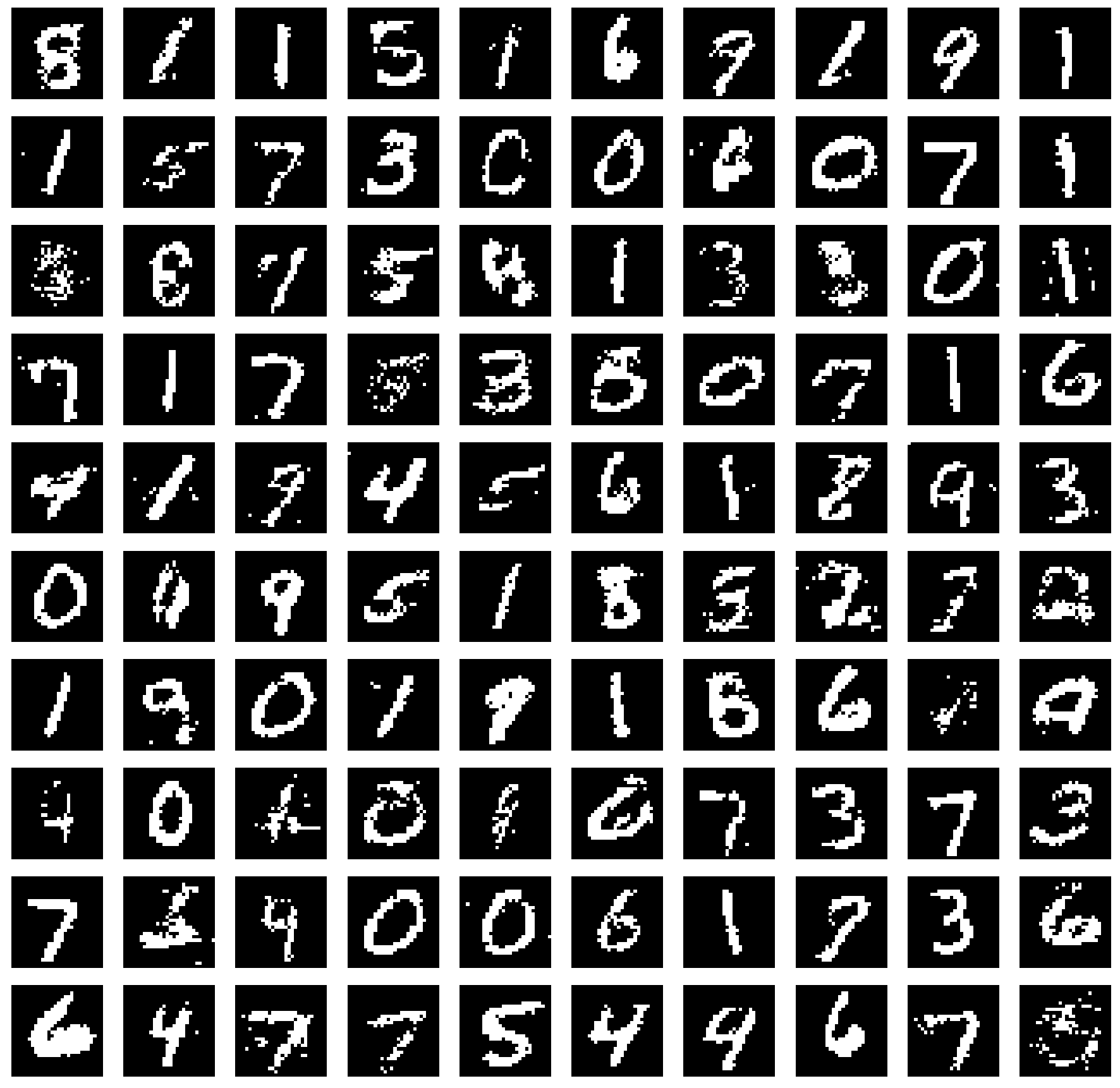}
  \par{\small Soft}

  \caption{Binarized MNIST samples generated by \methodname, STE, and
  Soft, from top to bottom.}
  \label{fig:mnist-samples}
\end{figure}

\end{document}